\documentclass[10pt]{article} 
\usepackage[table]{xcolor}
\usepackage[preprint]{tmlr}

\usepackage{amsmath,amsfonts,bm}

\def\eqref#1{equation~\ref{#1}}

\def\1{\bm{1}}

\DeclareMathAlphabet{\mathsfit}{\encodingdefault}{\sfdefault}{m}{sl}
\SetMathAlphabet{\mathsfit}{bold}{\encodingdefault}{\sfdefault}{bx}{n}

\DeclareMathOperator*{\argmin}{arg\,min}

\usepackage{hyperref}
\usepackage{url}
\usepackage{etoc}

\newcommand{\yes}{\textcolor{green!55!black}{\ding{51}}}      
\newcommand{\soft}{\textcolor{orange!90!black}{$\boldsymbol{\sim}$}} 
\newcommand{\no}{\textcolor{red!70!black}{\ding{55}}}          
\newcommand{\na}{\textcolor{gray!70}{--}}   
\newcommand{\noyes}{\no\,{\scriptsize$\rightarrow$}\,\yes}    
\newcommand{\softyes}{\soft\,{\scriptsize$\rightarrow$}\,\yes} 

\title{{\huge\textit{The Standard Interpretable Model}}\\
{\large Design interpretable machine learning methods from elementary symmetries}}

\author{\name Pietro Barbiero\thanks{Primary author.} \email pietro.barbiero@ibm.com \\
      \addr IBM Research (CH)
      \AND
      \name Giovanni De Felice\thanks{Contributed to initial conceptualisation, technical
        discussions, writing process, and experiments.} \email giovanni.defelice@usi.ch \\
      \addr Universit\`a della Svizzera Italiana (CH)
      \AND
      \name Mateo Espinosa Zarlenga\thanks{Contributed to technical discussions and writing
        process.} \email mateo.espinosa@cs.ox.ac.uk \\
      \addr University of Oxford (UK)
      \AND
      \name Francesco Giannini\footnotemark[3] \email francesco.giannini@unipi.it \\
      \addr Universit\`a di Pisa (IT)
      \AND
      \name Filippo Bonchi\footnotemark[3] \email filippo.bonchi@unipi.it \\
      \addr Universit\`a di Pisa (IT)
      \AND
      \name Mateja Jamnik\footnotemark[3] \email mateja.jamnik@cl.cam.ac.uk \\
      \addr University of Cambridge (UK)
      \AND
      \name Giuseppe Marra\thanks{Provided theoretical, experimental, and writing oversight
        from conception to submission.} \email giuseppe.marra@kuleuven.be \\
      \addr KU Leuven (BE)
      \AND
      \name Ruggero Noris\footnotemark[4] \email ruggero.noris@fzu.cz \\
      \addr Institute of Physics of the Czech Academy of Sciences (CZ)}

\def\month{MM}  
\def\year{YYYY} 
\def\openreview{\url{https://openreview.net/forum?id=XXXX}} 

\usepackage{array, tabularx}
\usepackage{multicol}
\usepackage{makecell}
\usepackage{booktabs}
\usepackage{xltabular}

\usepackage{amsmath}
\usepackage{amssymb}
\usepackage{mathtools}
\usepackage{amsthm}
\usepackage{cancel}
\usepackage[title]{appendix}
\usepackage{titletoc}
\usepackage[normalem]{ulem}

\newcommand{\inlineimg}[1]{%
  \smash{\raisebox{-.2\height}{\includegraphics[height=.9em]{figs/#1.png}}}%
}

\newcommand{\inlineimgbig}[1]{%
  \smash{\raisebox{-.2\height}{\includegraphics[height=1.1em]{figs/#1.png}}}%
}
\newcommand{\inlineimgrotate}[1]{%
  \smash{\raisebox{-.2\height}{\includegraphics[height=.9em, origin=c, angle=180]{figs/#1.png}}}%
}

\usepackage{listings}

\definecolor{symmetriesYellow}{HTML}{FCDDB0}
\definecolor{architecturesRed}{HTML}{A8E6D4}
\definecolor{learningBlue}{HTML}{CDBFF5}
\newcommand{\symmetriesText}[1]{\textcolor[HTML]{B85C00}{#1}}

\newcommand{\architecturesText}[1]{\textcolor[HTML]{007A52}{#1}}
\newcommand{\learningText}[1]{\textcolor[HTML]{5B30C8}{#1}}

\theoremstyle{plain}
\newtheorem{theorem}{Theorem}[section]

\newtheorem{lemma}[theorem]{Lemma}

\theoremstyle{definition}

\theoremstyle{remark}

\newtheorem*{problem}{\textbf{Problem Statement}}

\usepackage{mdframed}

\newtheorem*{lemma*}{Lemma}
\newtheorem*{proposition*}{Proposition}
\newtheorem*{corollary*}{Corollary}
\newtheorem*{conjecture*}{Conjecture}
\newtheorem*{example*}{Example}

\newmdtheoremenv[
  backgroundcolor=blue!10,   
  linecolor=blue!50,         
  innertopmargin=8pt,
  innerbottommargin=2pt,
  skipabove=2pt,
  skipbelow=2pt
]{definition2}{Definition}

\newcounter{symmetry}

\newmdtheoremenv[
  backgroundcolor=symmetriesYellow,   
  innertopmargin=8pt,
  innerbottommargin=2pt,
  skipabove=2pt,
  skipbelow=2pt,
  nobreak=true
]{definition3}[symmetry]{Symmetry}

\surroundwithmdframed[
  backgroundcolor=gray!30,
  innertopmargin=8pt,
  innerbottommargin=2pt,
  skipabove=2pt,
  skipbelow=2pt,
  nobreak=true
]{definition9}

\surroundwithmdframed[
  backgroundcolor=gray!30,
  innertopmargin=8pt,
  innerbottommargin=2pt,
  skipabove=2pt,
  skipbelow=2pt,
  nobreak=true
]{definition10}

\surroundwithmdframed[
  backgroundcolor=gray!30,
  innertopmargin=8pt,
  innerbottommargin=2pt,
  skipabove=2pt,
  skipbelow=2pt,
  nobreak=true
]{definition11}

\newcounter{constraint}

\newmdtheoremenv[
  backgroundcolor=symmetriesYellow,   
  innertopmargin=8pt,
  innerbottommargin=2pt,
  skipabove=2pt,
  skipbelow=2pt,
  nobreak=true
]{definition4}[constraint]{Constraint}

\newcounter{architecture}

\newmdtheoremenv[
  backgroundcolor=architecturesRed,   
  innertopmargin=8pt,
  innerbottommargin=2pt,
  skipabove=2pt,
  skipbelow=2pt
]{definition5}[architecture]{Architecture}

\surroundwithmdframed[
  backgroundcolor=symmetriesYellow,   
  innertopmargin=8pt,
  innerbottommargin=2pt,
  skipabove=2pt,
  skipbelow=2pt
]{definition6}

\newtheorem*{definition7}{Interpretable optimisation problem}
\surroundwithmdframed[
  backgroundcolor=learningBlue,   
  innertopmargin=8pt,
  innerbottommargin=2pt,
  skipabove=2pt,
  skipbelow=2pt
]{definition7}

\surroundwithmdframed[
  backgroundcolor=learningBlue,
  innertopmargin=8pt,
  innerbottommargin=2pt,
  skipabove=2pt,
  skipbelow=2pt
]{definition8}

\newcommand{\myset}[1]{\myset{#1}}

\usepackage[capitalize,noabbrev]{cleveref}

\newif\ifcomments
\commentsfalse  

\usepackage{pifont}   
\definecolor{DeepGreen}{rgb}{0, 0.5, 0}

\usepackage{tabularray}
\usepackage{graphicx}

\begin{document}

\maketitle

\begin{abstract}
As Artificial Intelligence models grow in complexity, interpretability has become an indispensable tool for understanding, debugging, and controlling their computations. 
However, \emph{interpretability lacks general theories to deductively design interpretable methods}. This gap between theories and methods results in a fragmented literature and inconsistent evaluation protocols.
To fill this gap, we introduce the \emph{Standard Interpretable Model} (SIM), 
a general theory that enables the deductive design of interpretable methods from first principles.
Specifically, the SIM summarises, in a set of \symmetriesText{premises}, what interpretability is for a target entity. From these premises, the SIM systematically derives interpretability \symmetriesText{symmetries} and corresponding \symmetriesText{constraints}. One can then build interpretable models by either \learningText{introducing constraints in the optimisation problem} or by
\architecturesText{compiling constraints into an interpretable architecture}.
The most important finding, both theoretical and empirical, is that full expressivity and interpretability symmetries can coexist in the same model, showing that the trade-off between accuracy and interpretability is an artefact of model design. We demonstrate this on language models (LMs), where SIM-based LMs match the performance of equivalent opaque LMs without making any compromise on interpretability.
Beyond this, the SIM shows how to improve the interpretability of existing pre-trained models, highlights underexplored research directions, and unifies a discipline that has long been fragmented, offering pedagogical grounding for interpretability curricula.
\end{abstract}

\etocdepthtag.toc{main}
\etocsettagdepth{main}{toc}        
\etocsettagdepth{appendix}{none}   
\begingroup
\footnotesize 
\tableofcontents
\endgroup
\clearpage

\section{Introduction}
As Artificial Intelligence (AI) models grow in complexity, it is necessary to understand their computations to diagnose errors, steer predictions, ensure fairness, and attain legal compliance \citep{lee2021development,meng2022interpretability,richmond2024explainable}. To this end, interpretability research has produced a new generation of models \citep{senn,protopnet,cem,oikarinenlabel,steerling2026paper} whose decisions are easy for humans to understand and whose predictive performance is comparable to that of powerful opaque architectures such as Deep Neural Networks (DNNs).

While significant, we argue that interpretable AI currently lacks a general, systematic method for developing operational theories of interpretability.
We argue this, despite key foundational efforts by \citet{kim2016examples}, \citet{biran2017explanation} \citet{doshi2017towards}, \citet{lipton2018mythos}, \citet{miller2019explanation}, \citet{watson2021explanation}, \citet{facchini2021first}, \citet{giannini2024categorical},  \citet{tull2024towards}, and others (see Section~\ref{sec:related-works}), as previous works have failed to deductively translate interpretability principles into interpretable methods, architectures, loss functions, constraints, or metrics. This persistent gap between the development of interpretability theories and the design of interpretable methods results in a fragmented literature, inconsistent evaluation protocols, and methods that are difficult to systematically compare, reproduce, or build upon.

\paragraph{How modern theories analyse complex phenomena.}
To build a complete and operational theory of interpretability, that is, a theory that formally describes when models can be understood by humans, we gain inspiration from how previous scientific theories have successfully made complex phenomena intelligible. 
\begin{figure}[!b]
    \centering
    \resizebox{1\columnwidth}{!}{\includegraphics{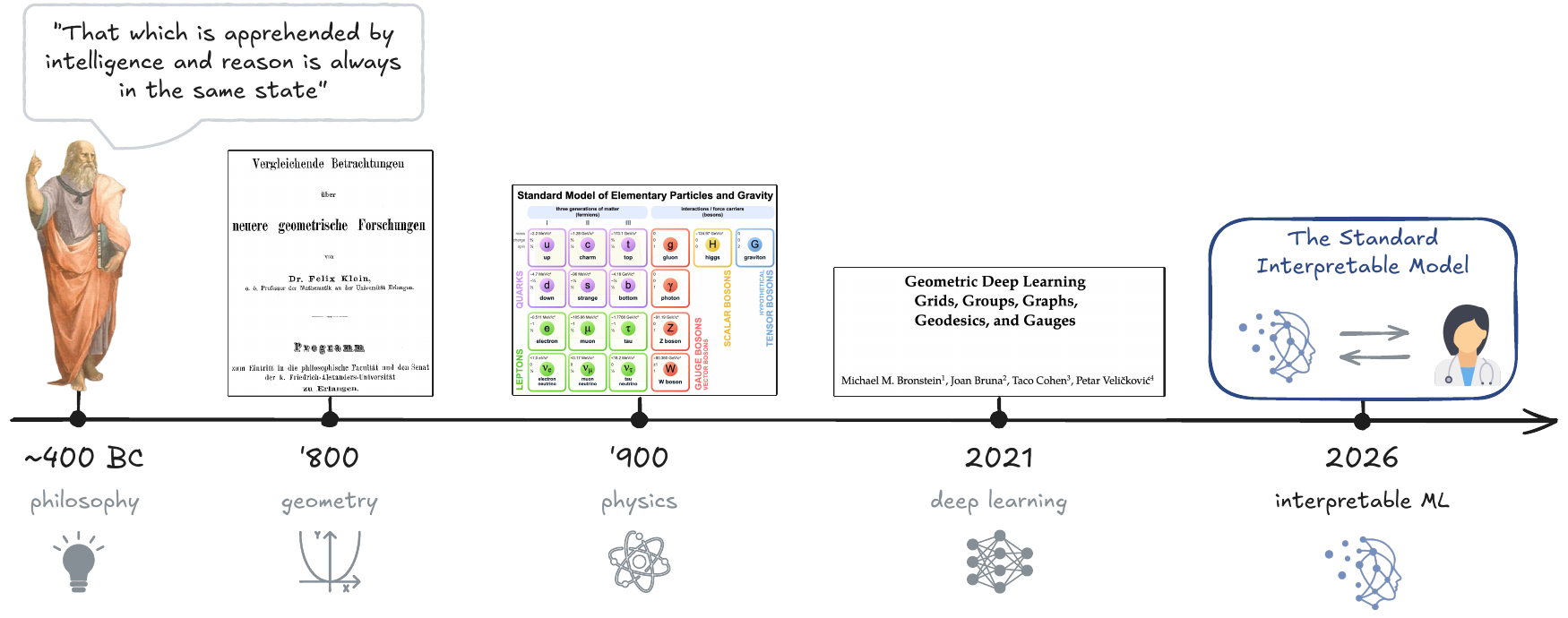}}
    \caption{\textbf{Scientific theories based on ``immutables''.} At their core, modern theories identify properties that remain invariant under transformations, i.e., \emph{symmetries}. Inspired by such theories, the standard interpretable model identifies the key symmetries related to interpretable machine learning.}
    \label{fig:history}
\end{figure}
Across mathematics and physics, theories explaining a phenomenon of interest are often proposed by first identifying which \emph{properties remain invariant under transformations} (Figure~\ref{fig:history}). For instance, the gravitational field produced by a spherical body is independent of the angle of measurement/observation, revealing a fundamental property of gravitational fields.
The idea of studying phenomena by isolating their structural properties from contingent features traces back to~\citet{platotimaeus}'s theory of \emph{immutables}, where he argued that: ``That which is apprehended by intelligence and reason is always in the same state''. This idea had a significant influence on the history of mathematics and the sciences. For example, ~\citet{klein1893comparative} proposed rethinking the whole field of geometry as a theory of immutables, known there as \emph{symmetries}. 
A few decades later,~\citet{einstein1916grundlagen} used symmetries to formulate special and general relativity, and~\citet{noether1918invariante} formalised the connection between symmetries and conserved quantities, opening the way to what we know today as the ``standard model'', a unified theory of the fundamental interactions in our universe based on symmetries \citep{weyl1929electron,weinberg1967model}. 
Recently, this idea has found its way into AI, where symmetries have been used as an organising principle for ``geometric'' Deep Learning (DL) \citep{bronstein2021geometric}.

\paragraph{Contributions.}
Inspired by these theories, we propose (\emph{contribution I}) the \emph{Standard Interpretable Model} (SIM), a general interpretability theory based on ``immutables'' that deductively derives interpretable methods from first principles (Section~\ref{sec:sim}).
As an example, we use the SIM to develop (\emph{contribution II}) a theory of interpretability relative to a user endowed with a formal language and bounded computation. Within this theory, we derive the symmetries under which full expressivity and interpretability symmetries coexist in the same model, showing that \textbf{the trade-off between accuracy and interpretability is an artefact of model design} (Section~\ref{sec:expressivity}). We then empirically validate this theory (\emph{contribution III}), showing that SIM-based language models (LMs) match the performance of equivalent opaque LMs without compromising on interpretability (Section~\ref{sec:experiments}). Finally, we use the SIM to improve the interpretability of pre-trained models, compare existing interpretability methods, identify often overlooked limitations, and highlight underexplored research directions (Section~\ref{sec:discussion}).

\section{Standard Interpretable Model}
\label{sec:sim}

A key problem in machine learning is determining the parameters $\theta$ of a function $f: Z \to Y$ that minimise the error $\mathcal{L}$ when predicting a target $y \in Y$, i.e., $\argmin_\theta \mathcal{L}(f(z;\theta), y)$. Yet this formulation says little about the nature of the end function or how easy it is for a target entity to understand, inspect, and interpret. Therefore, a natural practical extension of this formulation is to impose constraints on $f$ that require it to be ``interpretable'' for some target entity~$h$~\citep{rudin2019stop}:
\begin{equation} \label{eq:interpretable-problem}
    \arg\min_\theta \mathcal{L}(f(z;\theta), y) \quad \text{such that $f$ is ``interpretable'' for $h$}
\end{equation}
Nevertheless, defining what \emph{formal} conditions make a model and its parameters ``interpretable'' remains an open question. This implies that, although informal definitions of interpretability may sometimes be partly captured by proxy constraints, the interpretability literature lacks a \emph{standardised theory} for expressing general interpretability conditions operationally in a form that can be translated into interpretable methods. The Standard Interpretable Model aims to fill this gap by addressing the following problem:
\begin{problem}
    \emph{How can we formalise the conditions that make an ML model $f$ interpretable for a target entity~$h$? Crucially, can we introduce a formalisation that provides a deductive method for measuring a model's interpretability or for constructing an interpretable model from first principles?}
\end{problem}

\subsection{Deductive design of interpretable methods}

The \emph{Standard Interpretable Model} (SIM) operates in three phases (Figure~\ref{fig:overview}). In \symmetriesText{Phase~I}, the SIM formalises what counts as interpretable for a target entity in terms of \textit{symmetries}, that is, properties that do not change under specific geometric transformations. Then, in \learningText{Phase~II} and \architecturesText{Phase~III}, the SIM operationalises symmetries by providing two (non-mutually exclusive) deductive methods for building and training interpretable ML models. 
\begin{figure}[!t]
    \centering
    \resizebox{\columnwidth}{!}{\includegraphics{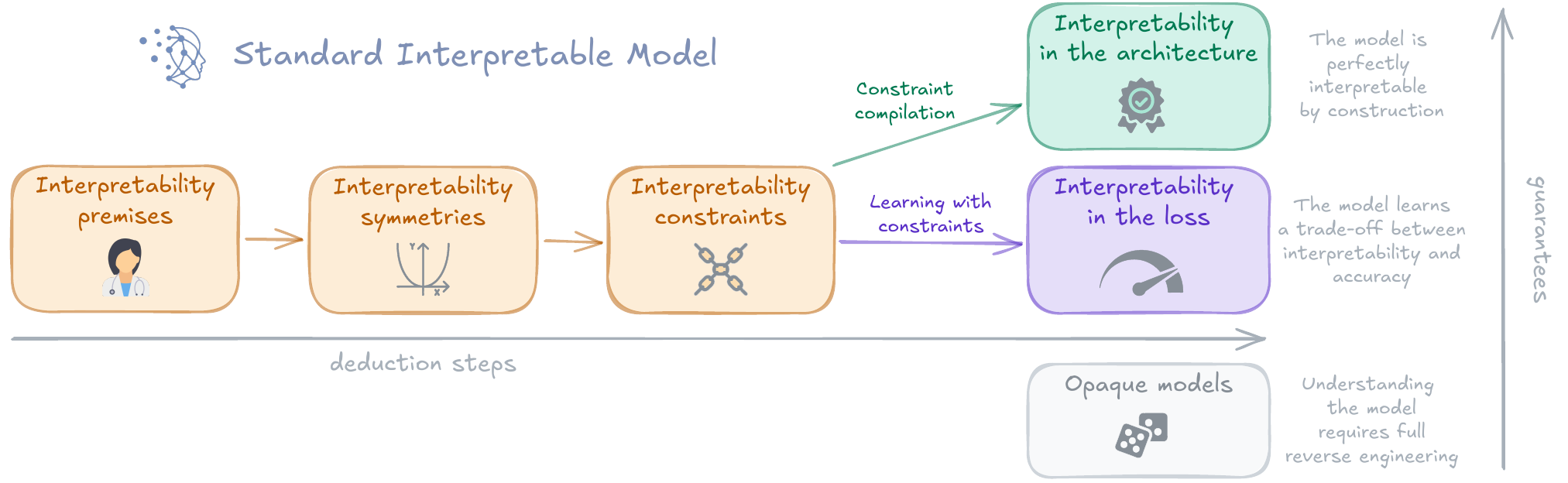}}
    \caption{\textbf{The Standard Interpretable Model (SIM).} The SIM formalises what counts as interpretable for a target entity $h$ in terms of symmetries. Next, the SIM translates interpretability symmetries into constraints. The SIM uses interpretability constraints to build interpretable models either by minimising constraint violation in the training process or by compiling constraints into the architecture.}
    \label{fig:overview}
\end{figure}
These phases can be further broken down into the following individual steps\footnote{For an immediate intuition of how the SIM works in practice, we provide a minimal example in Appendix~\ref{app:dummy}.}:
\begin{enumerate}
    \item \symmetriesText{\textbf{(Phase I: step 1)} \emph{Interpretability premises}} (Section~\ref{sec:premises}): First, one must describe what counts as ``interpretable'' for a \textit{target entity} $h$, the entity relative to which interpretability is defined and assessed (e.g., an expert, a group of humans, an idealised user, other AIs). Since interpretability is inherently subjective \citep{doshi2017towards, miller2019explanation}, \emph{interpretability theories derived from the SIM are therefore user-aware} and parameterised as a function of the expected target entity $h$.
    \item \symmetriesText{\textbf{(Phase I: step 2)} \emph{Interpretability symmetries}} (Section~\ref{sec:symmetries}): Next, we formalise interpretability premises in the form of properties that should remain invariant under transformations~$\mathfrak{g}$ (i.e., \textit{symmetries})
    to guarantee user understanding of a model $f$. These symmetries yield a \emph{formal definition of interpretability} relative to the target entity $h$.
    \item \symmetriesText{\textbf{(Phase I: step 3)} \emph{Interpretability constraints}} (Section~\ref{sec:constraints}): We then translate interpretability symmetries into an \emph{operational form}. Specifically, each interpretability constraint $q(z) \leq 0$ serves as both a metric \emph{quantifying how interpretable a model is} (that is, how much it violates the symmetry) and a building block for constructing architectures that satisfy the symmetry by design.
    \item \learningText{\textbf{(Phase II)} \emph{Interpretability in the loss}} (Section~\ref{sec:learning}): The first approach to build an interpretable function~$f$ is through \emph{optimisation}. That is, during training we minimise how much~$f$ violates interpretability constraints (i.e., instantiating Eq.~\ref{eq:interpretable-problem}). This approach applies to any function~$f$, but it only locally encourages interpretability symmetries rather than guaranteeing them. 
    \item \architecturesText{\textbf{(Phase III)} \emph{Interpretability in the architecture}} (Section~\ref{sec:architecture}): To ensure that a model $f$ satisfies interpretability constraints, we can \textit{compile} \citep{selman1991knowledge,selman1996knowledge} constraints into the model's architecture. Rather than minimising constraint violation, this process yields a function~$f$ whose structure guarantees $q(z) \leq 0$ for all $x \in \text{Dom}(f)$, even out of the training distribution.
    Since the architecture satisfies all interpretability constraints by construction, this approach requires no interpretability loss.
\end{enumerate}
This five-step process brings two key advantages over existing frameworks of interpretability:
\begin{itemize}
    \item \textbf{Field unification.} By formalising interpretability in terms of symmetries, the SIM yields a general geometric approach to interpretability that reconciles the fragmentation in the interpretability literature subsuming formalisms, informal statements, and methods. For instance, the SIM enables interpretability to be analysed both \emph{locally}, by studying how interpretability varies under infinitesimal changes in the inputs, and \textit{globally}, when local interpretability holds everywhere in the domain.
    \item \textbf{User-aware theory.} Since what counts as interpretable depends on the target entity, interpretability premises are not unique. The SIM captures this property of interpretability by providing a systematic procedure to derive symmetries, constraints, architectures, learning objectives, and metrics from any set of user-derived premises. In this sense, the SIM can be viewed as a template for constructing interpretability theories. Changing the premises yields different, but internally consistent, theories of interpretability, much as changing Euclid's postulates yields different geometries.
\end{itemize}


Having motivated and illustrated an overview of the SIM, we now show how to use the SIM to derive a concrete user-aware theory of interpretability in machine learning.

\subsection{Premises of interpretability for bounded and formal user}
\label{sec:premises}
Throughout this section, we consider an entity $h$ that wants to interact with the model $f$ so they can, for example, distil knowledge from the model, better control the model, simulate the model's output, verify the model's decisions, or debug a failure mode. At a minimum, such interactions require a communication channel to enable information exchange between the user and the model. Moreover, this communication channel should be designed to avoid ambiguity, suggesting it should be carried through a formal language.
Hence, we consider target entities endowed with a \emph{formal} language, that is, a fixed vocabulary of symbols (i.e., \textit{syntax}) and a mechanism assigning a meaning to each symbol (i.e., \textit{semantics})~\citep{tarski1953undecidable,tarsky1956concept}. For example, given a vocabulary of three symbols $\{ \texttt{small}, \texttt{round}, \land \}$, a target entity may judge the object \inlineimg{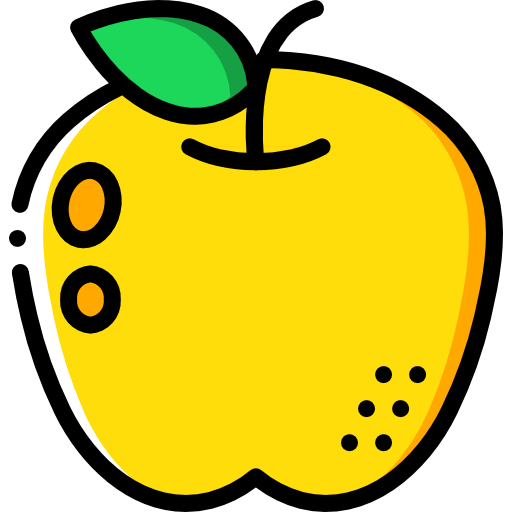} as \texttt{round} to a degree of $0.9$, \texttt{small} to a degree of $0.3$, and evaluate \texttt{small}~$\land$~\texttt{round} as $0.3 \cdot 0.9 = 0.27$, using a fuzzy interpretation for $\land$. While in principle the syntax allows unbounded composition (e.g.,  \texttt{small}~$\land$~\texttt{round}~$\land$~\texttt{red}~$\land \dots$), some classes of entities (e.g., humans) have limited computational resources, which impose concrete \emph{bounds} on the complexity of their interactions with the model~\citep{simonm1947administrative,simon1956rational,miller1956magical}. Hence, not every well-formed expression is valid for every target entity.
\begin{figure}[!t]
    \centering
    \resizebox{\columnwidth}{!}{\includegraphics{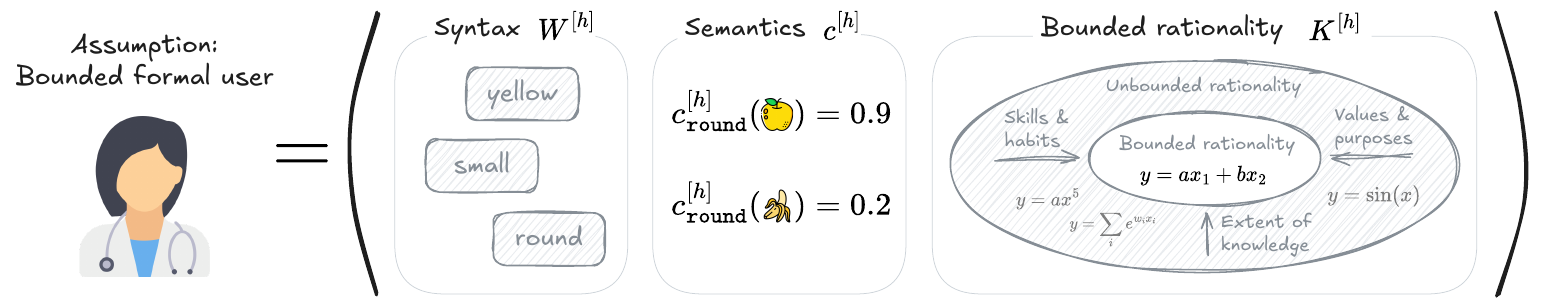}}
    \caption{\textbf{Representation of a ``bounded formal user''.} We assume that the recipients of interpretable ML models are bounded and formal users, that is, users endowed with a vocabulary of concepts, a mechanism to assign values to such concepts, and time or computational limits.}
    \label{fig:target-entity}
\end{figure}
We refer to such entities as \emph{bounded formal users} (Figure~\ref{fig:target-entity}), and posit the following three premises characterise when a model is interpretable for these entities (Figure~\ref{fig:premises}):

\paragraph{Premise I: Concept alignment (Section~\ref{sec:symmetryI}).} To prevent ambiguity, the model must agree with the user on the meaning of each concept (string, operator, function, or relation) in the vocabulary. For example, if the user thinks the object \inlineimg{Apple} is \texttt{round}, then the model must also consider the object \texttt{round}. Concept alignment enables reliable atomic communication between the user and the model.

\paragraph{Premise II: Concept-prediction dependency (Section~\ref{sec:symmetryII}).}
Since aligned concepts are the main model components that the entity formally understands, fixing the concepts should fix the model output. This guarantees that, by interacting with concepts, the entity interacts with everything that drives the model's behaviour. For example, if the model predicts that \inlineimg{Apple} is an \texttt{apple} because it is \texttt{small}~$\land$~\texttt{round}, then changing anything outside these concepts, such as the image background, should not change the prediction.


\paragraph{Premise III: Hypothesis space alignment (Section~\ref{sec:symmetryIII}).} To respect the user's computational constraints, the model must use only concept compositions within the hypothesis class that the user can tractably reason about~\citep{rudin2019stop,rudin2022interpretable}. For example, a rule with a thousand conjuncts could be syntactically and semantically well defined, yet no human can meaningfully process it. 

In the next section we formalise these premises as symmetries\footnote{We discuss further details on interpretability symmetries, their consequences, and subsumed special cases in Appendix~\ref{app:symmetries}.}.

\begin{figure}[!t]
    \centering
    \resizebox{\columnwidth}{!}{\includegraphics{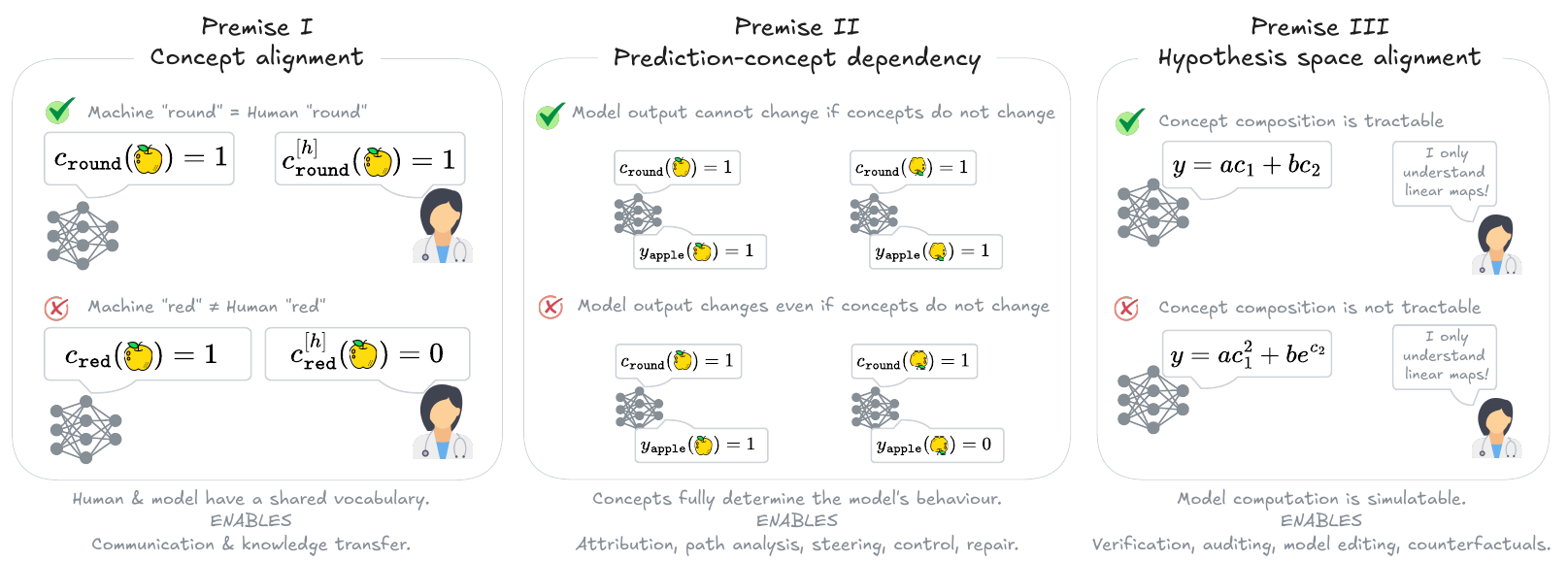}}
    \caption{\symmetriesText{\textbf{Interpretability premises.}} Bounded formal users require ML models to satisfy three premises to be interpretable: concept alignment, formal interaction, and bounded reasoning.}
    \label{fig:premises}
\end{figure}







\subsection{Interpretability symmetries}
\label{sec:symmetries}


\paragraph{Setup.} 
Let $f: Z \to Y$ be a model mapping $n$-dimensional object representations $z \in Z \subseteq \mathbb{R}^n, n \in \mathbb{N}$ (e.g., embeddings, pixels, tokens) to outputs $y \in Y \subseteq \mathbb{R}^v, v \in \mathbb{N}$ (e.g., labels, pixels, token embeddings). Moreover, let $\mathcal{L}(f(z;\theta_f),y)$ be a function measuring $f$'s error when predicting a target label $y$ (e.g., cross-entropy loss), and let $\mathcal{D} = \{(z_i, y_i)\}_{i=1}^d$ be a dataset of $d\in \mathbb{N}$ examples we can use to estimate the expected value of $\mathcal{L}$. Here, a \emph{symmetry} of~$f$ is a group of transformations $\mathfrak{g} \in \mathfrak{G}$ acting on $Z$ via the group action $\mathfrak{g} . z$. We say that~$f$ is \emph{invariant} under $\mathfrak{G}$ if $f(\mathfrak{g}.z) = f(z)$ for all $\mathfrak{g} \in \mathfrak{G}$ and $z \in Z$, that is, if the output of~$f$ is unchanged by the group action. For instance, the transformation $\mathfrak{g}: z \mapsto -z$ is an invariance of $f(z) = |z|$, since $f(\mathfrak{g}.z) = f(-z) = |-z| = |z| = f(z)$. Finally, without loss of generality, we refer to the entity~$h$ as a \emph{human}.

\subsubsection{Symmetry~I: Concept invariance under monotonic maps}
\label{sec:symmetryI}

A key requirement for a model $f$ to be ``interpretable'' is that it preserves the semantics $h$ assigns to all symbols, a condition we capture in Premise I. This \emph{concept alignment} implies that, if the model uses a symbol from the human's vocabulary, then its use is interpretable only if it preserves the semantics the human $h$ assigns to that symbol. To formalise this idea, let $W^{[h]}$ denote $h$'s vocabulary (i.e., the set of symbols available to $h$ for describing objects, properties, or relations). For example, in a vision task, this vocabulary may contain symbols such as \texttt{round}, \texttt{small}, or \texttt{striped}; in a medical task, it may contain symbols such as \texttt{lesion} or \texttt{age}. For each symbol $w \in W^{[h]}$, let $c_w^{[h]} : Z \to \mathbb{R}$ be a \emph{human concept map}.  This map assigns to each object's representation $z \in Z$ a score measuring the degree to which~$w$ applies to $z$ according to $h$.

Notice how the score $c_w^{[h]}(z)$ on a single object $z$ is not itself meaningful: $c_\texttt{size}^{[h]}(\inlineimg{Apple})=1.3$ does not mean anything. However, the way $c_w^{[h]}$ changes as we \textit{transform} an object may tell us what $w$ means for $h$ (Figure~\ref{fig:symmetryI}). For instance, consider a transformation that enlarges the apple, $\mathfrak{g}_w.\inlineimg{Apple} = \inlineimgbig{Apple}$, and suppose $c_\texttt{size}^{[h]}(\inlineimg{Apple}) < c_\texttt{size}^{[h]}(\mathfrak{g}_w.\inlineimg{Apple})$. This suggests that \texttt{size} may potentially mean, for $h$, a quantity proportional to an object's physical size (e.g., its area). Now consider a model~$f$ that uses the symbol \texttt{size}, with concept map $c_\texttt{size} : Z \to \mathbb{R}$. If the model also scores $c_\texttt{size}(\mathfrak{g}_w.\inlineimg{Apple}) = \rho_{\mathfrak{g}}\big(c_\texttt{size}(\inlineimg{Apple})\big) > c_\texttt{size}(\inlineimg{Apple})$, i.e., enlarging the apple acts on the model's concept score through a map $\rho_{\mathfrak{g}}$ that raises the concept score, then the model gives the same response as $h$.\footnote{How much the score is raised is left free. While the magnitude of the change may be useful for calibration, it is an arbitrary convention even for humans (e.g., consider different units) and therefore carries no semantic meaning per se.} If instead the model scores $c_\texttt{size}(\inlineimg{Apple}) > c_\texttt{size}(\mathfrak{g}_w.\inlineimg{Apple})$, then for the model the enlarged apple is smaller, and the model disagrees with $h$.


\begin{figure}[!h]
    \centering
    \resizebox{1.\columnwidth}{!}{\includegraphics{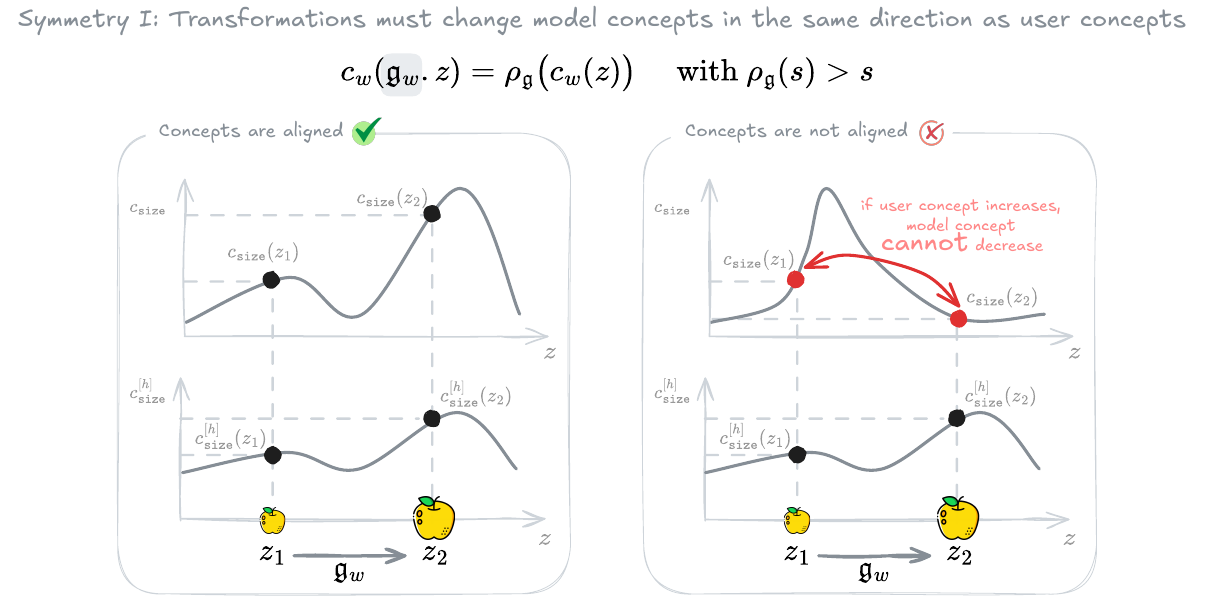}}
    \caption{\symmetriesText{\textbf{Symmetry I.}} A model's concept is aligned with the semantics assigned to that concept by a human~$h$ so long as it changes in the same direction as its corresponding human concept.}
    \label{fig:symmetryI}
\end{figure}

\noindent
Formally, we define the semantics of $w$ for a human~$h$ through the set of transformations that increase $w$ according to $h$:
\begin{equation}
    \mathfrak{G}^{[h]}_w
    =
    \{\mathfrak{g}_w : Z \to Z \mid c_w^{[h]}(z) < c_w^{[h]}(\mathfrak{g}_w.z),\forall z\in Z\}.
\end{equation}

\medskip
\begin{definition3}
    \label{invariance:symm_1}
    Let $c_w^{[h]}: Z \to \mathbb{R}$ be the concept map associated with symbol $w$ by a target human $h$ and $c_w: Z \to \mathbb{R}$ be the model's corresponding concept map. The model preserves $h$'s semantics of $w$ if, for any $\mathfrak{g}_w \in \mathfrak{G}^{[h]}_w$,
    \[
        c_w(\mathfrak{g}_w.z) = \rho_{\mathfrak{g}}\big(c_w(z)\big)
        \quad \text{for some } \rho_{\mathfrak{g}} : \mathbb{R} \to \mathbb{R} \text{ with } \rho_{\mathfrak{g}}(s) > s.
    \]
    \vspace{-1mm}
\end{definition3}

\subsubsection{Symmetry~II: Output invariance under concept-preserving transformations}
\label{sec:symmetryII}
Our second premise requires that concepts, on their own, drive model behaviour. In other words, this premise requires that a model $f$'s predictions depend only on symbols whose semantics are shared with the target human $h$. Intuitively, this means that whenever the model's prediction changes, this can be explained by a change in the shared concepts. Conversely, if we change the input representation in a way that does not affect any shared concept, then the model's prediction should not change.

For instance, consider an apple $z_1 = \inlineimg{Apple}$ and the same apple rotated upside down, $z_2 = \inlineimgrotate{Apple}$ (Figure~\ref{fig:symmetryII}). 
The rotation does not make the apple any more or less round, $c_\texttt{round}(\inlineimg{Apple}) = c_\texttt{round}(\inlineimgrotate{Apple})$, that is, from the perspective of the concept round, $z_1$ and $z_2$ are indistinguishable. 
If the model's behaviour depends only on shared concepts, then it has no grounds to treat the two apples differently, and therefore $f_\texttt{apple}(\inlineimg{Apple}) = f_\texttt{apple}(\inlineimgrotate{Apple})$.

To formalise this idea, let $c : Z \to \mathbb{R}^l$ be the vector of concept maps, collecting one map $c_w : Z \to \mathbb{R}$ for each of the $l$ symbols $w$ in the vocabulary shared between $h$ and $f$. Around a point $z \in Z$, the model's behaviour should be determined only by the concepts $c$. To express this as a symmetry, let $\mathfrak{G}_c(z)$ be the set of all transformations near $z$ that do not have any impact on concepts
\begin{equation}
    \mathfrak{G}_c(z) = \{\mathfrak{g}_c: Z \to Z \ \mid \ c(\mathfrak{g}_c.z) = c(z)\}
\end{equation}
Then if the model behaviour is determined only by concepts, the model output should not be affected by any transformation $\mathfrak{g}_c \in \mathfrak{G}_c(z)$ (Figure~\ref{fig:symmetryII}).

\begin{figure}[!t]
    \centering
    \resizebox{.9\columnwidth}{!}{\includegraphics{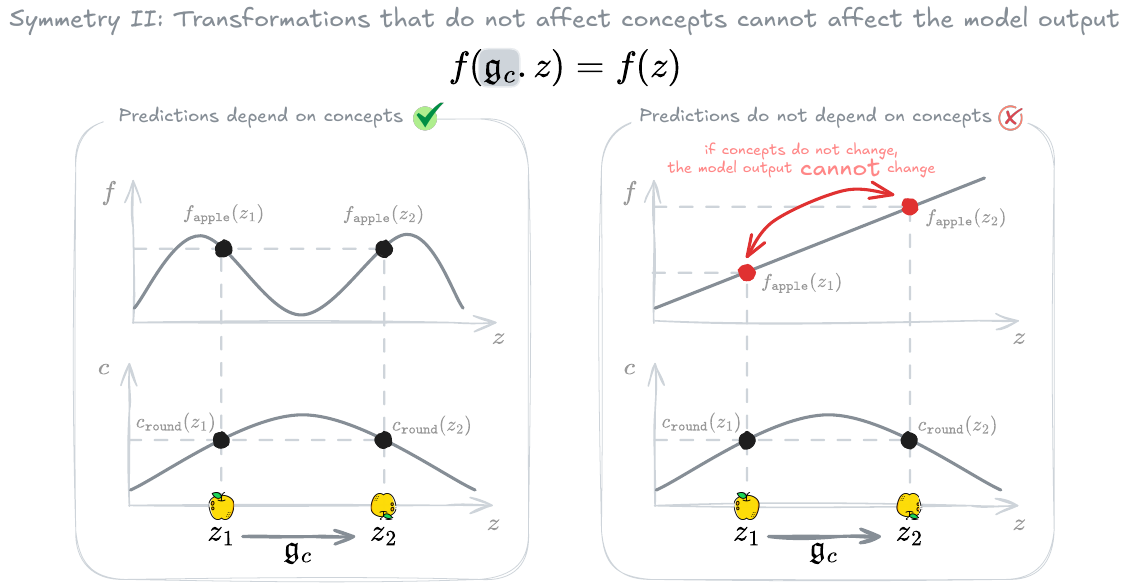}}
    \caption{\symmetriesText{\textbf{Symmetry II.}} A model's prediction $f(z)$ is interpretable with respect to $h$ so long as the prediction does not change if concepts shared between $f$ and $h$ do not change.}
    \label{fig:symmetryII}
\end{figure}

\medskip
\begin{definition3}
    \label{invariance:symm_2}
    Let $\mathfrak{G}_c$ be the set of all transformations near $z$ that do not affect concepts. A function $f$ is interpretable with respect to the concept maps $c$ if it is invariant to all $\mathfrak{g}_c \in \mathfrak{G}_c(z)$:

    \begin{equation}
        f(\mathfrak{g}_c.z) = f(z)
    \end{equation}

\vspace{-1mm}
\end{definition3}

\subsubsection{Symmetry~III: Hypothesis space invariance under tractable transformations}
\label{sec:symmetryIII}

Premise III (hypothesis space alignment) requires restricting the computation $f$ performs on the shared concepts to functions that the target entity~$h$ can tractably reason about. Any object $z$ can be described, locally, by its concepts $c$ and by the remaining features $u$ that the concepts do not capture, so that a generic model can be written as $f(z) = f(c, u)$. Hypothesis space alignment only concerns how the output changes when the concepts change while the remaining features stay the same. We therefore write $f_u(c) = f(c, u)$ for the resulting function of the concepts alone when $u$ is held fixed.

The difficulty of formalising Premise III arises from the fact that writing an explicit equation for a broad class of functions is often hard. Yet, one can more easily characterise the transformations a family of functions must be invariant to~\citep{lie1880integration,lie1891vorlesungen,olver1993applications}. For instance, despite it being difficult to express the general form that all \textit{even} functions must have (e.g., $x^2$ or $|x|$), this set of functions can be easily described by the invariant that their outputs should be constant when a sign flip occurs (i.e.,  $f_u(c) = f_u(-c)$).
Similarly, linear (continuous) functions can be described by the fact that they are invariant to averaging the values on either side of $c$ (i.e., $f(c) = \tfrac{1}{2}\big(f_u(c + \delta) + f_u(c - \delta)\big)$. Notice how this description of the hypothesis class for linear functions nicely captures the identity function $f_u(c) = c$, since $\tfrac{1}{2}\big((c + \delta) + (c - \delta)\big) = c$, but excludes quadratic functions, since $\tfrac{1}{2}\big((c + \delta)^2 + (c - \delta)^2\big) = c^2 + \delta^2 \neq c^2$ for $\delta \neq 0$.

More generally, let $\mathfrak{G}^{[h]}_f$ be a set of transformations, either discrete or continuous, where each $\mathfrak{g}_f \in \mathfrak{G}^{[h]}_f$ maps a function $f_u$ to a new function $\mathfrak{g}_f.f_u$. Given a set $\mathfrak{G}^{[h]}_f$ of transformations that are tractable according to the target entity~$h$, the interpretable models are those that do not change under any transformation in the set, i.e., $\mathfrak{g}_f.f_u = f_u$ for all $\mathfrak{g}_f \in \mathfrak{G}^{[h]}_f$. Since any family of functions can be described in terms of the properties that remain invariant by some transformations, this formulation can characterise any hypothesis space.

\medskip
\begin{definition3}
    \label{invariance:symm_3}
    A model $f$ is \emph{interpretable} for~$h$ if, for any fixed $u$, $f_u$ is invariant to all transformations $\mathfrak{g}_f \in \mathfrak{G}^{[h]}_f$ that are tractable for the target entity~$h$:
    \begin{equation}
        (\mathfrak{g}_f.f_u)(c) = f_u(c)
    \end{equation}
    \vspace{-1mm}
\end{definition3}

\begin{figure}[!t]
    \centering
    \resizebox{.9\columnwidth}{!}{\includegraphics{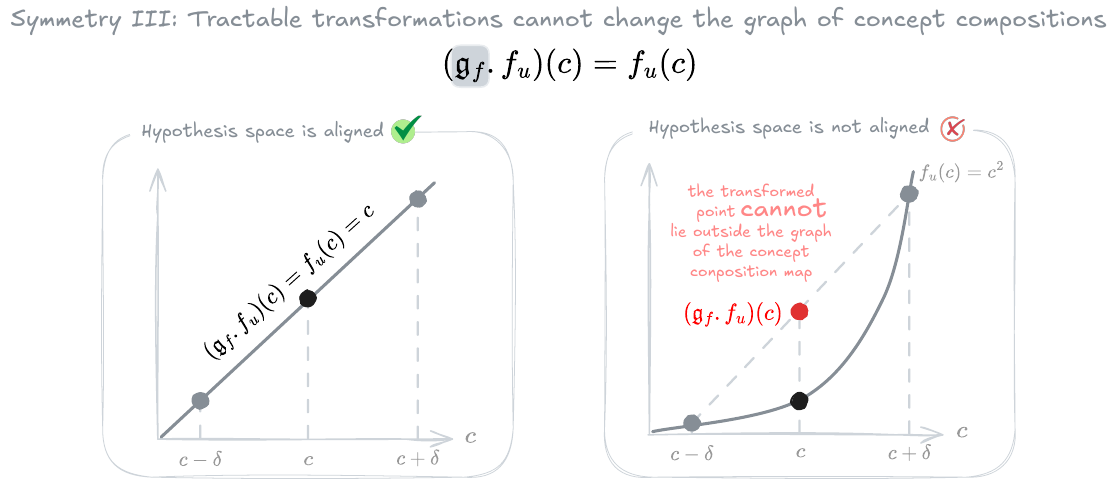}}
    \caption{\symmetriesText{\textbf{Symmetry III.}} A tractable hypothesis space is defined by the transformations its functions must be invariant to. Here, the transformation $\mathfrak{g}_f$ averages the values of a function on either side of $c$, $(\mathfrak{g}_f.f_u)(c) = \tfrac{1}{2}\big(f_u(c+\delta) + f_u(c-\delta)\big)$, and the functions left unchanged by $\mathfrak{g}_f$ are exactly the linear functions. Left: the linear function $f_u(c) = c$ is invariant, since $\tfrac{1}{2}\big((c+\delta) + (c-\delta)\big) = c$, so the average lies on the graph. Right: the quadratic function $f_u(c) = c^2$ is not invariant, since $\tfrac{1}{2}\big((c+\delta)^2 + (c-\delta)^2\big) = c^2 + \delta^2$, so the average lies $\delta^2$ above the graph and $f_u$ falls outside the hypothesis space.}
    \label{fig:symmetryIII}
\end{figure}

\subsection{Interpretability constraints}
\label{sec:constraints}
Having translated each premise into a symmetry, we can now express them as constraints that can be operationalised to measure interpretability and build interpretable architectures.

\subsubsection{Constraint~I: Concept alignment}
Symmetry~\ref{invariance:symm_1} holds if and only if the model concept map $c_w$ changes in the same direction as $c_w^{[h]}$. Whenever $h$ judges the enlarged apple $\mathfrak{g}_w.\inlineimg{Apple}$ to score higher than $\inlineimg{Apple}$ on concept \texttt{size}, the model must agree. In other words, if $\Delta c_\texttt{size}^{[h]} = c_\texttt{size}^{[h]}(\mathfrak{g}_w.\inlineimg{Apple}) - c_\texttt{size}^{[h]}(\inlineimg{Apple}) > 0$, then the model preserves the semantics if and only if $\Delta c_\texttt{size} = c_\texttt{size}(\mathfrak{g}_w.\inlineimg{Apple}) - c_\texttt{size}(\inlineimg{Apple}) > 0$ as well. If instead $\Delta c_\texttt{size} < 0$, then for the model the enlarged apple is \emph{smaller} than $\inlineimg{Apple}$. As this violates human semantics, we need a constraint that penalises such an outcome (Figure~\ref{fig:constraintI}). In practice, if we do not have access to the transformation $\mathfrak{g}_w$ itself, then we can treat any pair of objects in a dataset as an object and its transformed version. For instance, if $h$ judges $\inlineimg{Apple}$ to be rounder than $\inlineimg{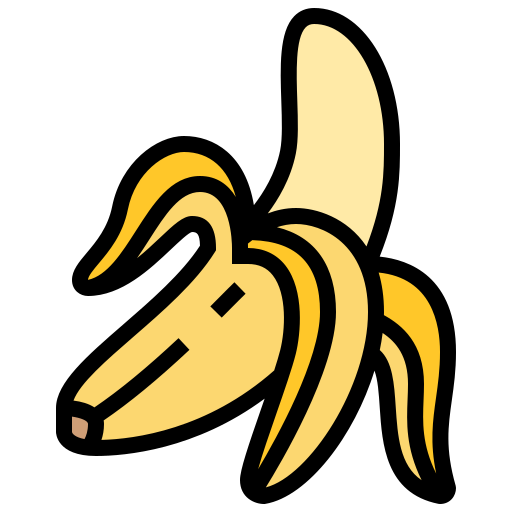}$, i.e., $c_\texttt{round}^{[h]}(\inlineimg{Banana}) < c_\texttt{round}^{[h]}(\inlineimg{Apple})$, we can treat $\inlineimg{Apple}$ as $\mathfrak{g}_w.\inlineimg{Banana}$, and then require that the model must agree with the relative order the human assigns to objects, i.e., $c_\texttt{round}(\inlineimg{Banana}) < c_\texttt{round}(\inlineimg{Apple})$.
\begin{figure}[!t]
    \centering
    \resizebox{\columnwidth}{!}{\includegraphics{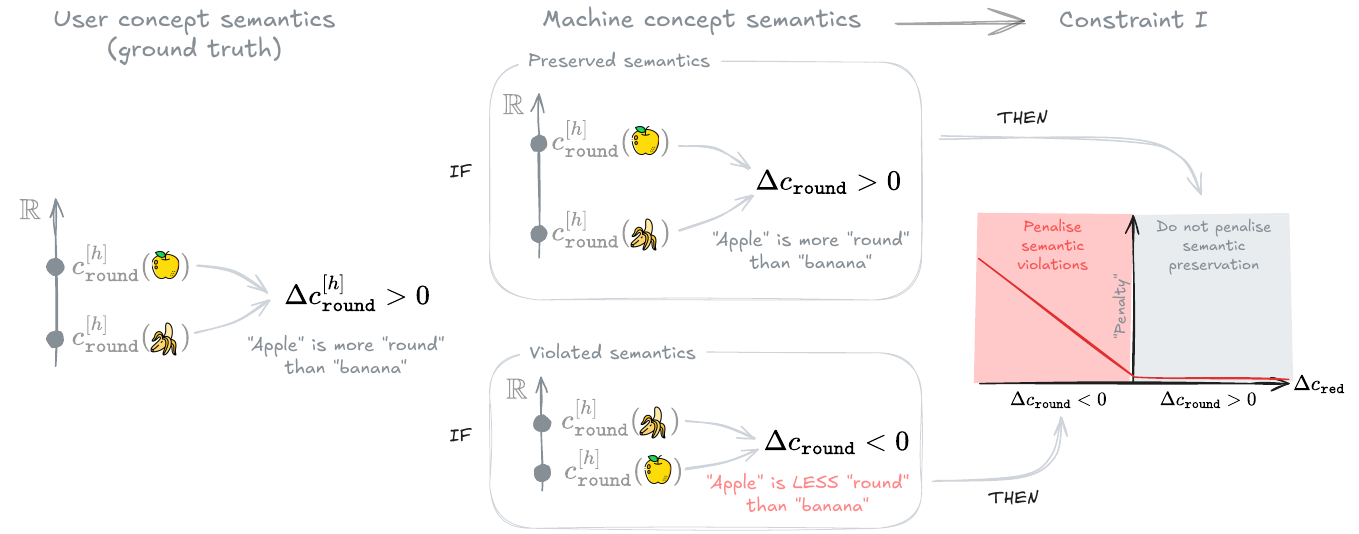}}
    \caption{\symmetriesText{\textbf{Constraint I.}} The constraint penalises model concept maps that change in a different direction than user concept maps.}
    \label{fig:constraintI}
\end{figure}

\noindent
This leads to a constraint corresponding to Symmetry~\ref{invariance:symm_1} (full derivation is in \Cref{app:constraintI}):

\medskip
\begin{definition4}
    \label{constraint:1}
    Let $\Delta c_w = c_w(\mathfrak{g}_w.z) - c_w(z)$ and $\Delta c_w^{[h]} = c_w^{[h]}(\mathfrak{g}_w.z) - c_w^{[h]}(z)$ denote the change of the model and human concept maps, respectively, for any object $z$ and transformation $\mathfrak{g}_w$; let $\mathbb{I}_{\Delta c_w^{[h]} > 0}$ denote the indicator function of $\mathfrak{g}_w \in \mathfrak{G}^{[h]}_w$, and $\gamma: \mathbb{R} \to \mathbb{R}_{0}^+$ a monotonic decreasing function with $\gamma(s) = 0$ for $s > 0$.
    The equivariance $c_w(\mathfrak{g}_w.z) = \rho_{\mathfrak{g}}\big(c_w(z)\big)$ holds iff:
    \begin{equation}
        \mathbb{I}_{\Delta c_w^{[h]} > 0} \cdot \gamma(\Delta c_w) = 0
    \end{equation}
\vspace{-1mm}
\end{definition4}

\subsubsection{Constraint~II: Output-concept dependency}

Symmetry~\ref{invariance:symm_2} holds if and only if the model output does not change along any direction in which the concepts do not change. For instance, if rotating the apple, $\mathfrak{g}_c.\inlineimg{Apple} = \inlineimgrotate{Apple}$, does not change any aligned concept such as $c_\texttt{round}$ and $c_\texttt{size}$, then $f_\texttt{apple}$ must not change either. Equivalently, every local change in the model output can be expressed in terms of local changes in the concepts, i.e., the output gradients $\nabla_z f$ should lie in the subspace spanned by the concept gradients $\nabla_z c$. In other words, we need to penalise whenever the subspace spanned by $\nabla_z f$ is not contained in the subspace spanned by $\nabla_z c$ (Figure~\ref{fig:constraintII}).

To do this, let $Q_c \in \mathbb{R}^{n \times \mathrm{rank}(\nabla_z c)}$ denote the matrix whose columns form an orthonormal basis for $\mathrm{span}\bigl\{\nabla_z c_1,\,\ldots,\,\nabla_z c_l\bigr\}$, and let $Q_f \in \mathbb{R}^{n \times \mathrm{rank}(\nabla_z f)}$ denote the matrix whose columns form an orthonormal basis for $\mathrm{span}\bigl\{\nabla_z f_1,\,\ldots,\,\nabla_z f_v\bigr\}$. Using $Q_c$ and $Q_f$, we consider the similarity matrix $Q_c^\top Q_f \in \mathbb{R}^{\mathrm{rank}(\nabla_z c) \times \mathrm{rank}(\nabla_z f)}$ between the orthonormal bases, whose entries are inner products between basis vectors of the respective subspaces. If an orthonormal basis vector $q \in Q_f$ lies entirely within $\mathrm{span}(Q_c)$, then it can be expressed as a linear combination of the columns of $Q_c$, and consequently the sum of squared inner products is $\|Q_c^\top q\|^2 = 1$. If this holds for all basis vectors $q \in Q_f$, then summing over the columns of $Q_f$ gives $\|Q_c^\top Q_f\|_F^2 = \mathrm{rank}(\nabla_z f)$ (full derivation is in Appendix~\ref{app:constraintII}).

\begin{figure}[!h]
    \centering
    \resizebox{\columnwidth}{!}{\includegraphics{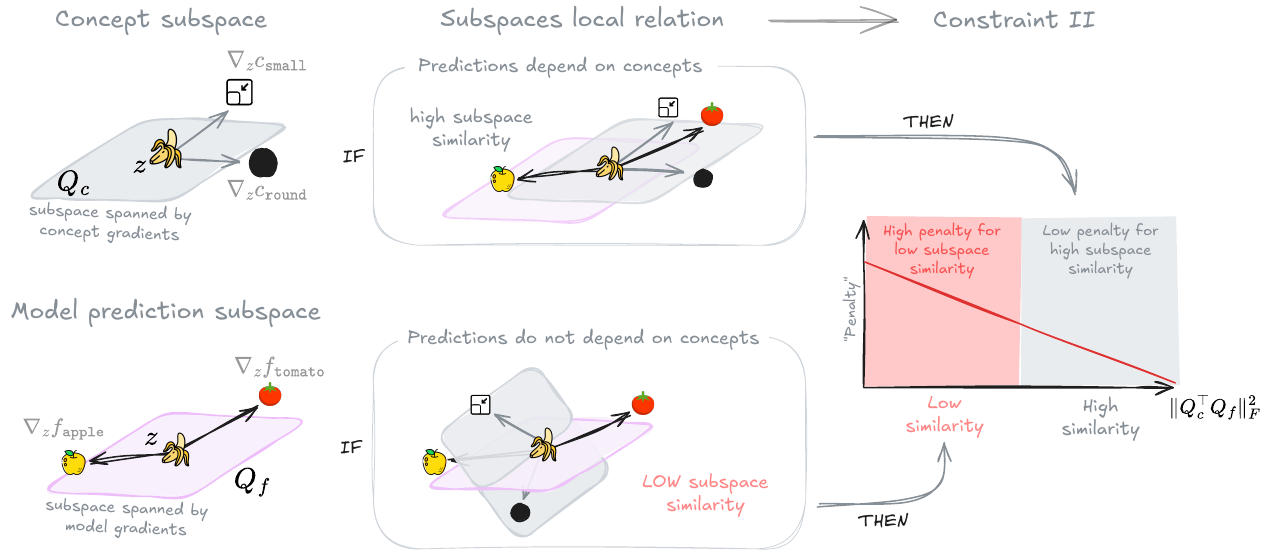}}
    \caption{\symmetriesText{\textbf{Constraint II.}} The constraint penalises any mismatch between the subspace spanned by model prediction gradients $\nabla_z f$ and the subspace spanned by concept gradients $\nabla_z c$.}
    \label{fig:constraintII}
\end{figure}

\begin{definition4}
    \label{constraint:2}
    Let $Q_c$ and $Q_f$ be orthonormal bases for the row span of $\nabla_z c$ and $\nabla_z f$, respectively. The invariance $f(\mathfrak{g}_c.z) = f(z)$ holds locally around $z$ iff:
    \begin{equation}
        1 - \frac{\|Q_c^\top Q_f\|_F^2}{\mathrm{rank}(\nabla_z f)} = 0
    \end{equation}
\vspace{-1mm}
\end{definition4}

\subsubsection{Constraint~III: Hypothesis space alignment}

The set of tractable transformations $\mathfrak{G}^{[h]}_f$ induces an invariance operator $K^{[h]}_f(c, f_u)$ that vanishes if and only if the transformed function equals $f_u$. For a discrete transformation, we can write it directly from the symmetry condition. As an example, for the sign flip $K^{[h]}_f(c, f_u) = (\mathfrak{g}_f.f_u)(c) - f_u(c) = f_u(-c) - f_u(c) = 0$ for all $c$. The absolute value satisfies this constraint as $|-c| - |c| = 0$, whereas the linear function does not as $-c - c = -2c \neq 0$ for $c \neq 0$.

A continuous transformation, instead, contains infinitely many transformations, one for each $\delta$, which we cannot tractably check for. However, it suffices to define the constraint for an infinitesimal transformation, which turns the operator into a differential equation~\citep{lie1880integration,leveque2007finite}. For instance, for the average transformation described in Section~\ref{sec:symmetryIII}, the operator becomes $K^{[h]}_f(c, f_u) = \frac{\partial^2}{\partial c^2} f_u(c) = 0$ for all $c$ when we consider an infinitesimal change $\delta$ around $c$. For example, any linear function of $c$ satisfies  $\frac{\partial^2}{\partial c^2}\, c = 0$, whereas a quadratic function does not will not, since $\frac{\partial^2}{\partial c^2}\, Ac^2 + B c + C = 2A \neq 0$.

\noindent
This leads to a constraint corresponding to Symmetry~\ref{invariance:symm_3}:

\begin{definition4}
    \label{constraint:3}
    Let $K^{[h]}_f$ be the invariance operator induced by the set $\mathfrak{G}^{[h]}_f$. The invariance $\mathfrak{g}_f.f_u = f_u$ holds iff, for all $c$ and any fixed $u$:
    \begin{equation}
        \|K^{[h]}_f(c, f_u)\| = 0
    \end{equation}
    \vspace{-1mm}
\end{definition4}

\begin{figure}[!h]
    \centering
    \resizebox{\columnwidth}{!}{\includegraphics{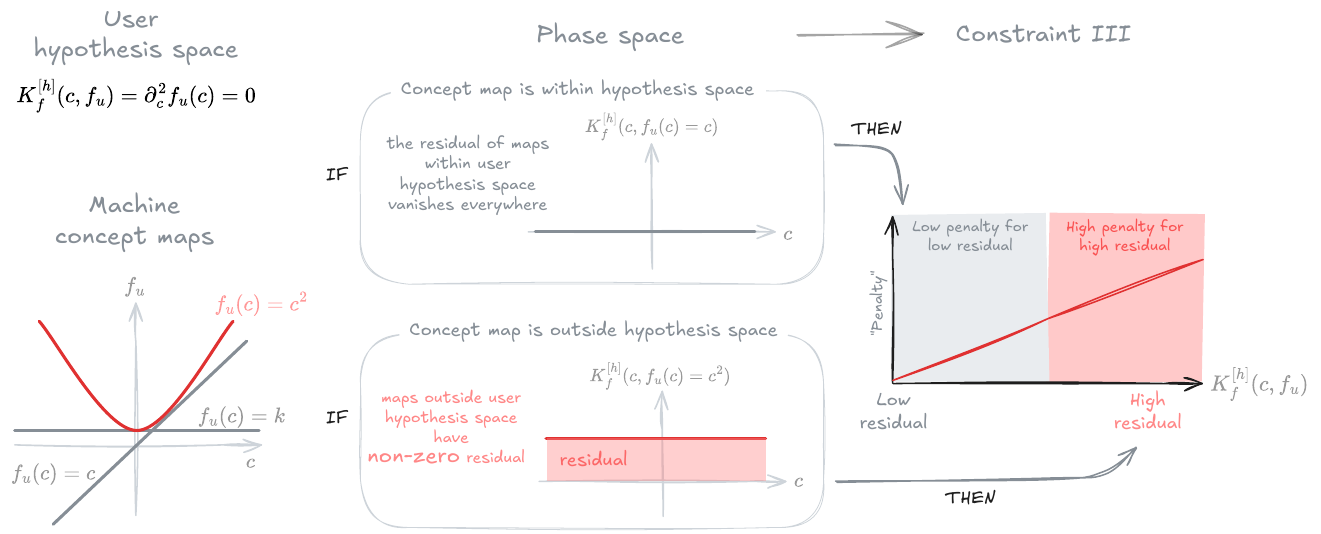}}
    \caption{\symmetriesText{\textbf{Constraint III.}} The averaging transformation $\mathfrak{g}_f$ induces the invariance operator $K^{[h]}_f(c, f_u) = \frac{\partial^2}{\partial c^2} f_u(c)$, and Constraint III penalises any non-zero value of the operator. Top: the linear function $f_u(c) = c$ satisfies the constraint, since $\frac{\partial^2}{\partial c^2}\, c = 0$ for all $c$, so the residual vanishes. Bottom: the quadratic function $f_u(c) = c^2$ violates the constraint, since $\frac{\partial^2}{\partial c^2}\, c^2 = 2$ for all $c$, so the residual is constant and positive.}
    \label{fig:constraintIII}
\end{figure}

\subsection{Interpretability in the loss: Constraint optimisation}
\label{sec:learning}

Combining the constraints above yields the following constrained optimisation problem, where the constraints must hold for all $z$:
\begin{align}
    \min_{\theta_f, \theta_c} \quad & \mathcal{L}(f(z;\theta_f),y) \\
    \text{s.t.}
    \quad & \sum_{w=1}^m \sum_{\mathfrak{g}_w} \mathbb{I}_{\Delta c_w^{[h]}(z, \mathfrak{g}_w.z) > 0} \cdot \gamma\big(\Delta c_w(z, \mathfrak{g}_w.z;\theta_c)\big) = 0 \;\;\;\;\;\;\text{(Constraint~1)} \nonumber \\
    \quad & \left(1 - \frac{\|Q_c^\top Q_f\|_F^2}{\mathrm{rank}(\nabla_z f)}\right) = 0 \;\;\;\;\;\;\text{(Constraint~2)} \nonumber \\
    \quad & \big\| K^{[h]}_f\big(c, f_u\big) \big\| = 0 \;\;\;\;\;\;\text{(Constraint~3)} \nonumber
\end{align}

Building an interpretable model requires solving this optimisation problem. To solve it, we have two options. First, we can convert the constrained problem into an unconstrained one by adding each interpretability constraint to the objective, weighted by a non-negative multiplier $\lambda_i$. Intuitively, each $\lambda_i$ acts as a penalty that grows whenever the corresponding constraint is violated~\citep{hestenes1969multiplier}. The resulting dual unconstrained max-min problem seeks the multipliers $\lambda_i$ that make the trade-off between the objectives as tight as possible:

\begin{definition7}
    The Lagrangian relaxation of the interpretability-constrained optimisation problem for the model~$f$ is given by:
    \begin{equation}
        \label{lagrangian:1}
        \begin{aligned}
        \max_{\lambda_i > 0} \min_{\theta_f, \theta_c} & \mathcal{L}(f(z;\theta_f),y) \\
        & + \lambda_1 \sum_{w=1}^m \sum_{\mathfrak{g}_w} \mathbb{I}_{\Delta c_w^{[h]}(z, \mathfrak{g}_w.z) > 0} \cdot \gamma\big(\Delta c_w(z, \mathfrak{g}_w.z;\theta_c)\big) \\
        & + \lambda_2 \left(1 - \frac{\|Q_c^\top Q_f\|_F^2}{\mathrm{rank}(\nabla_z f)}\right) \\
        & + \lambda_3 \big\| K^{[h]}_f\big(c, f_u\big) \big\|
        \end{aligned}
    \end{equation}
    where all terms are averaged over the training samples $(z, y)$, $c = c(z;\theta_c)$, $f_u$ is the model $f(\cdot\,;\theta_f)$ as a function of the concepts with $u$ held fixed, $\mathcal{L}(\cdot, \cdot)$ is a task-specific loss, $\|\cdot\|_F$ is the Frobenius norm, $\lambda_i \in \mathbb{R}_{>0}$ are Lagrangian multipliers, and $\{\theta_f, \theta_c\}$ are the model and concept maps parameters.
\end{definition7}

This approach can be applied to any function $f$, but it only locally encourages interpretability symmetries, rather than guaranteeing them\footnote{Appendix~\ref{app:lagrangian} describes how to derive a gradient descent algorithm that compiles interpretability constraints using Lagrangian mechanics.}. To ensure that interpretability constraints are satisfied, we need to follow a different approach.

\subsection{Interpretability in the architecture: Constraint compilation}

\label{sec:architecture}
Instead of relaxing the optimisation problem, the alternative is to embed interpretability directly into the model's architecture. This can be framed as a \emph{constraint compilation problem} \citep{selman1991knowledge,selman1996knowledge}, where interpretability constraints are built into the model, ensuring that any solution produced is interpretable by construction without needing to optimise for those constraints separately.

\subsubsection{Architecture I: Instance-based concept maps}

To compile the first constraint, we need to express $c_w$ as a function of observable data, since we assume no direct access to the parameters of the human concept map $c_w^{[h]}$. We assume a typical setting where we have a set of training instances $z_i \in \{\inlineimg{Apple}, \inlineimg{Banana}, \inlineimg{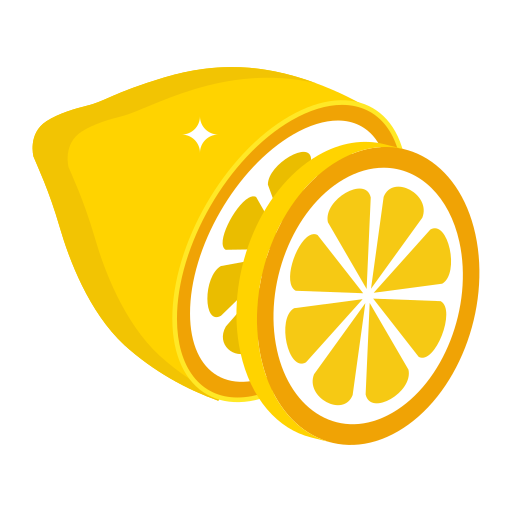}\}$ with known human-assigned scores $c_\texttt{round}^{[h]}(z_i)$~\citep{cbm}. As in Constraint~\ref{constraint:1}, we treat each pair of objects with $c_\texttt{round}^{[h]}(z_j) < c_\texttt{round}^{[h]}(z_i)$ as an object and its transformed version, $z_i = \mathfrak{g}_w.z_j$. To build $c_\texttt{round}$ so that it changes in the same direction as $c_\texttt{round}^{[h]}$, we compute the difference $d_\texttt{round}(z_i,z_j) = z_i - z_j$ for each such pair. Each of these vectors points towards a ``more \texttt{round} direction''. The overall concept map $c_\texttt{round}$ is then constructed as a combination of such directions, as in ridge regression~\citep{hoerl1970ridge,saunders1998ridge} (Figure~\ref{fig:architectureI}; derivation in Appendix~\ref{app:architectureI}).
\begin{figure}[!h]
    \centering
    \resizebox{\columnwidth}{!}{\includegraphics{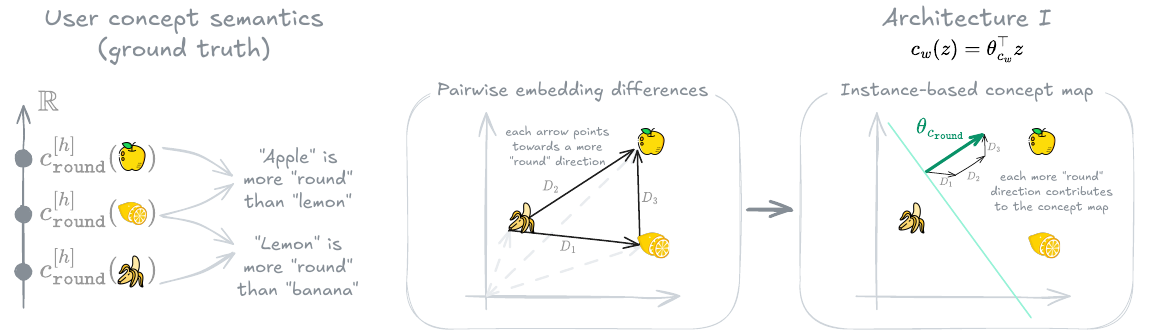}}
    \caption{\architecturesText{\textbf{Architecture I.}} The architecture implements a concept map that structurally preserves concept semantics. Pairs of objects that the user concept map ranks differently are used to identify the directions in the model's latent space along which the concept increases.}
    \label{fig:architectureI}
\end{figure}

\begin{definition5}
    \label{architecture:1}
    Assume we are given instances $\{z_i\}_{i=1}^d\in \mathbb{R}^n$ and a map $c_w^{[h]}(z_i)$ denoting the human-assigned score of $z_i$. Let $D_w \in \mathbb{R}^{p\times n}$ be the matrix whose rows are, for each pair $(i,j)$ with $c_w^{[h]}(z_j) < c_w^{[h]}(z_i)$, the latent space direction pointing from the lower-scored toward the higher-scored instance, $D_{w,k} = z_i - z_j$. Then, the architecture compiling Symmetry~\ref{invariance:symm_1} (i.e., $c_w(\mathfrak{g}_w.z) = \rho_{\mathfrak{g}}\big(c_w(z)\big)$) is:
    \begin{equation}
        c_w(z) = \theta_{c_w}^\top z
    \end{equation}
    where $\theta_{c_w} = D_w^\top\left(D_wD_w^\top + \varepsilon I\right)^{-1} \mathbf{1}$ is the closed-form ridge solution to $D_w\theta_{c_w} \approx \mathbf{1}$, i.e., the direction along which $c_w(z_j) < c_w(z_i)$ whenever $c_w^{[h]}(z_j) < c_w^{[h]}(z_i)$ for linearly independent pairs and small $\epsilon$.
\vspace{1mm}
\end{definition5}

\subsubsection{Architecture II: Concept-based model}
\begin{figure}[!b]
    \centering
    \resizebox{\columnwidth}{!}{\includegraphics{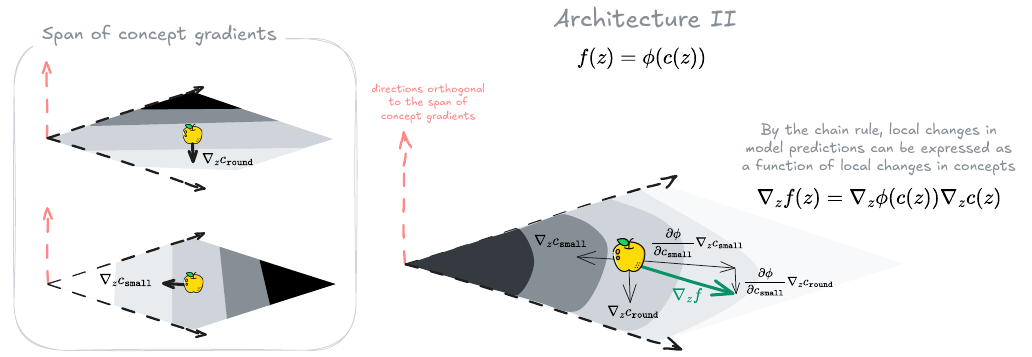}}
    \caption{\architecturesText{\textbf{Architecture II.}} The architecture implements a function of concept maps. By the chain rule, this implies that the function gradients lie in the span of concept gradients.}
    \label{fig:architectureII}
\end{figure}
Symmetry~\ref{invariance:symm_2} requires $\nabla_z f$ to lie in the row span of the concept map Jacobian $\{\nabla_z c_1, \ldots, \nabla_z c_l\}$. This means that any movement in the input space $z$ that leaves all concept maps unchanged must also leave~$f$ unchanged (i.e., if $c_1, \ldots, c_l$ are all constant along some direction,~$f$ cannot vary along that direction either). Thus, we need~$f$ to only change when at least one concept map $c_i$ changes, implying that~$f$ must be a function of concept maps alone, i.e., $f(c, u)$ must not depend on the remaining features $u$ (Figure~\ref{fig:architectureII}; see derivation in Appendix~\ref{app:architectureII}).

\medskip
\begin{definition5}
    \label{architecture:2}
    The architecture compiling Symmetry~\ref{invariance:symm_2} (i.e., $f(\mathfrak{g}_c.z) = f(z)$) is:
    \begin{equation}
        f(z) = f(c_{1}(z), \ldots, c_{l}(z))
    \end{equation}
\vspace{1mm}
\end{definition5}

\subsubsection{Architecture III: General solutions of the invariance equation}\label{sec:architecture-III}

Compiling Constraint~\ref{constraint:3} requires finding the functions $f_u$ that solve the invariance equation $K^{[h]}_f(c, f_u) = 0$. For instance, the solutions of the averaging equation $\frac{\partial^2}{\partial c^2} f_u(c) = 0$ form the family $F^{[h]}(c; \theta_1(u), \theta_0(u)) = \theta_1(u)\, c + \theta_0(u)$, whose slope and intercept may depend on the remaining features $u$. For instance, if colour is not among the concepts, the model may use one line for red apples, $\theta_1(u_\text{red})\, c + \theta_0(u_\text{red})$, and a different line for green apples, $\theta_1(u_\text{green})\, c + \theta_0(u_\text{green})$, as both are linear in the concept $c$. The quadratic function $c^2$, instead, cannot be obtained for any choice of slope and intercept (Figure~\ref{fig:architectureIII}).

More generally, the solutions of $K^{[h]}_f = 0$ form a parametric family $F^{[h]}(c; \theta(u))$, whose parameters $\theta(u)$ are the degrees of freedom that the equation leaves unspecified. The family $F^{[h]}$ is the same for every $u$, since $K^{[h]}_f$ only involves the concepts, while $u$ selects a member of the family through $\theta(u)$, an arbitrary learnable function that takes $u$ as input. Every choice of $\theta(u)$ yields a model satisfying Constraint~\ref{constraint:3}, and every model satisfying Constraint~\ref{constraint:3} is obtained for some choice of $\theta(u)$.

\begin{figure}[!h]
    \centering
    \resizebox{.8\columnwidth}{!}{\includegraphics{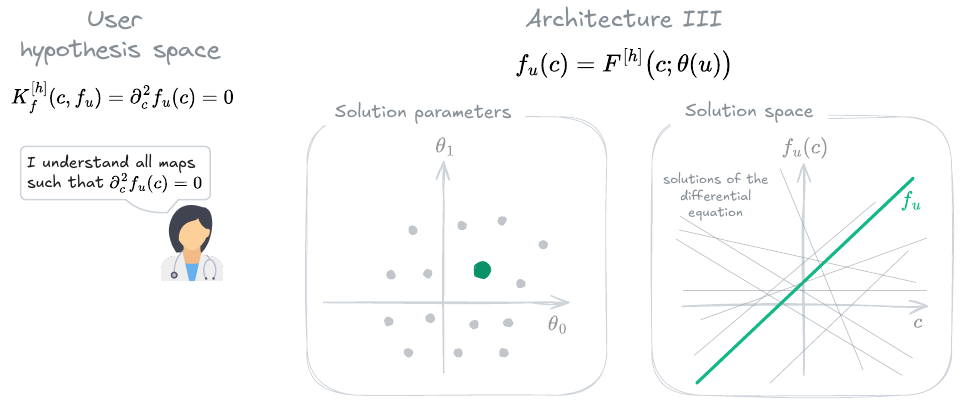}}
    \caption{\architecturesText{\textbf{Architecture III.}} The architecture implements the general solution of the invariance equation $\frac{\partial^2}{\partial c^2} f_u(c) = 0$, given by the family $F^{[h]}(c;\theta(u)) = \theta_1(u)\,c + \theta_0(u)$. Middle: each choice of slope $\theta_1$ and intercept $\theta_0$ yields a solution, and the remaining features $u$ select a member of the family through the learnable function $\theta(u)$, for instance one line for red apples, $\theta(u_\text{red})$, and another for green apples, $\theta(u_\text{green})$. Right: every member of the family is linear in the concept $c$.}
    \label{fig:architectureIII}
\end{figure}

\medskip
\medskip
\begin{definition5}
    \label{architecture:3}
    Let $F^{[h]}(c; \theta(u))$ be the parametric family of all solutions of $K^{[h]}_f = 0$. The general model compiling Symmetry~\ref{invariance:symm_3} (i.e., $(\mathfrak{g}_f.f_u)(c) = f_u(c)$) is:
    \begin{equation}
        f_u(c) = F^{[h]}\big(c; \theta(u)\big)
    \end{equation}
    where the parameters $\theta(u)$, corresponding to the degrees of freedom left unspecified by $K^{[h]}_f$, are given by an arbitrary learnable function that takes $u$ as input.
\vspace{1mm}
\end{definition5}

In Table~\ref{tab:pdes} we summarise tractable transformations, their invariance equations, and the solutions that yield common architectures used in the interpretability literature.
\begin{table}[h]
\centering
\renewcommand{\arraystretch}{1.5}
\resizebox{\linewidth}{!}{%
\begin{tabular}{lccc}
\hline
\textbf{Case} & \textbf{Transformation} $(\mathfrak{g}_f.f_u)(c)$ & \textbf{Constraint} $K^{[h]}_f = 0$ & \textbf{Solutions} $F^{[h]}(c; \theta(u))$ \\ \hline
\textbf{Linear} & $\tfrac{1}{2}\big(f_u(c + \delta) + f_u(c - \delta)\big)$ & $\nabla_c^2 f_u = 0$ & $\theta_1(u)^\top c + \theta_0(u)$ \\
\textbf{Even} & $f_u(-c)$ & $f_u(-c) - f_u(c) = 0$ & $g(c) + g(-c)$ \\
\textbf{Sparse} & $f_u(c + \delta e_j)$ & $\partial_{c_j} f_u = 0$ & $g(c_{-j})$ \\
\textbf{Piece-wise constant} & $f_u(c + \delta)$, with $c, c + \delta$ in the same $R_i$ & $\nabla_c f_u = 0$ for $c \notin B$ & $\sum_i \theta_i(u)\, \mathbb{I}_{R_i}(c)$ \\ \hline
\end{tabular}%
}
\caption{Summary of tractable transformations, their invariance equations, and the solutions that yield common architectures used in the interpretability literature. $\nabla_c^2 f_u$ denotes the Hessian matrix, $\delta$ ranges over all shifts, $e_j$ is the $j$-th unit vector, and $c_{-j}$ denotes all concepts except $c_j$. The boundaries $B$ are fixed by the user, and $R_i$ are the regions into which $B$ partitions the concept space. In the even and sparse cases, the equation leaves a whole function free, and $g$ denotes an arbitrary function, which may depend on $u$, whose choice plays the role of the parameters $\theta(u)$.}
\label{tab:pdes}
\end{table}

\subsubsection{A general solution for interpretable architectures}

The symmetries above constrain different components of a model's architecture. Symmetry~\ref{invariance:symm_1} constrains the structure of the concept maps $c_w$. Symmetry~\ref{invariance:symm_2} enforces that the predictions of~$f$ depend on inputs $z$ only through the concept maps $c_w$, i.e., that $f$ does not depend on the remaining features $u$. Finally, Symmetry~\ref{invariance:symm_3} restricts the form of that dependence to the family $F^{[h]}$. Since Symmetry~\ref{invariance:symm_2} removes any dependence on $u$, the parameters $\theta(u)$ of Architecture~\ref{architecture:3} reduce to fixed learnable parameters $\theta$. Satisfying all symmetries yields a general interpretable model:

\medskip
\begin{definition5}
    \label{architecture:4}
    The architecture compiling all interpretability symmetries is:
    \begin{equation}
        f(z; \theta, \theta_c) = F^{[h]}\Big(\big[\theta_{c_w}^\top z\big]_{w=1}^{l};\; \theta\Big)
    \end{equation}
    such that $\theta_{c_w} = D_w^\top\left(D_wD_w^\top + \varepsilon I\right)^{-1} \mathbf{1}$ and $F^{[h]}$ is the family of all solutions of $K^{[h]}_f = 0$.
    \vspace{1mm}
\end{definition5}
\smallskip
This gives a general solution for interpretable architectures given the initial premises. Any model of this form satisfies all three interpretability symmetries, and, given the concept maps, any model satisfying Symmetries~\ref{invariance:symm_2} and~\ref{invariance:symm_3} takes this form. The learnable components (i.e., the concept map parameters $\theta_c$ and the parameters $\theta$ of the solution family) are the only degrees of freedom left once the symmetries are imposed.

Together, all of this shows that the SIM not describes what interpretability is, but it yields (1)~a concrete recipe for building interpretable models (premises $\to$ symmetries $\to$ constraints $\to$ objective function/architecture), and (2)~a formal criterion for verifying model interpretability. Given a target entity~$h$ and its associated premises, one can systematically derive the corresponding constraints and read off the admissible architecture class directly. Conversely, given an arbitrary model, one can audit it against the same constraints to deductively determine whether it is interpretable for $h$.

\section{Beyond the expressivity-interpretability trade-off}
\label{sec:expressivity}

In theory, each interpretability symmetry reduces the information in $z$ that the machine learning model $f$ can use to predict the target $y$. Since concepts are computed from the input $c = c(z)$ according to Symmetry~\ref{invariance:symm_1}, the data processing inequality~\citep{shannon1948mathematical} implies that concepts cannot hold more information about $y$ than $z$ itself, $I(Y;Z) \geq I(Y;C)$. If $z$ holds information about $y$ beyond what the concepts capture, a model satisfying Symmetry~\ref{invariance:symm_2} will have strictly less information available for predicting $y$ than a model that does not, $I(Y;C) < I(Y;Z)$. This may lead to an \emph{incomplete} specification for a classifier to use to accurately predict the downstream task of interest~\citep{yeh2020completeness, cem}. In addition, Symmetry~\ref{invariance:symm_3} restricts the function class that controls the concept dependency to the family $F^{[h]}$, which limits the information $I(Y;\hat{Y})$ that the predictions $\hat{y} = F^{[h]}(c;\theta^*)$ of the best model in the family can retain about $y$:
\begin{equation}
    I(Y;\hat{Y}) \;\leq\; I(Y;C) \;\leq\; I(Y;Z)
\end{equation}
where the first inequality is strict whenever $F^{[h]}$ does not contain the optimal predictor of $y$ from $c$. For instance, if $y$ depends quadratically on a concept, no linear model can extract all the information that $c$ holds about $y$. Since $\hat{Y}$ is a function of $Z$, the residual information that cannot be captured by the interpretable model can be quantified as
\begin{equation}
    I(Y;Z) - I(Y;\hat{Y}) = I(Y;Z \mid \hat{Y}) > 0
\end{equation}
In the following we show how to design \emph{residual interpretable models}, that is, models that fully recover this residual information without having any effect on the ideal interpretable model $F^{[h]}(c;\theta^*)$. The two design principles behind residual interpretable models are:
\begin{itemize}
    \item \textbf{Isolation (Section~\ref{sec:isolation}).} The residual cannot have any effect on the interpretable model, neither in the forward pass nor in the backward pass. Intuitively, information and gradients can only flow from the interpretable model to the residual, but never the other way around.
    \item \textbf{Consistency (Section~\ref{sec:consistency}).} The residual model may preserve properties of the ideal interpretable model. For instance, we may require that intervening on a concept changes the prediction of the residual model in the same way as the prediction of the ideal interpretable model.
\end{itemize}

\subsection{Isolation: shielding the ideal interpretable model from the residual}
\label{sec:isolation}


Isolation requires that the residual cannot affect the interpretable model during training. For instance, consider a model that predicts $f_\texttt{apple}$ from the concepts $c_\texttt{round}$ and $c_\texttt{size}$ and, in addition, from residual features such as ``\texttt{texture}''. Predicting the apple from texture may be easier than predicting the apple from the $c_\texttt{round}$ and $c_\texttt{size}$ concepts. In this case, the residual quickly fits the target, and the interpretable model stops learning. To avoid this, isolation requires that discarding texture first and then training the model on $c_\texttt{round}$ and $c_\texttt{size}$ alone, that is, training the ideal interpretable model, yields the same model as training the full model using $c_\texttt{round}$, $c_\texttt{size}$ and texture, and then discarding texture. In other words, texture must not affect how the model learns from $c_\texttt{round}$ and $c_\texttt{size}$.

To formalise this, a residual interpretable model requires two things. First, there exists a transformation $\mathfrak{g}_0$ that switches off the residual~\citep{debot2025quantifying}, that is, a transformation that removes from the model every contribution that does not come from the concepts. For instance, applying $\mathfrak{g}_0$ to the apple model discards the contribution of texture and leaves a model that predicts $f_\texttt{apple}$ from $c_\texttt{round}$ and $c_\texttt{size}$ alone. Second, let $\mathfrak{u}$ be a training step acting on the model, so that training an untrained model $f^{(0)}$ for $t$ steps yields the trained model $\mathfrak{u}^t.f^{(0)} = f^{(t)}$, isolation requires that switching off the residual commutes with training:
\begin{equation}
    \mathfrak{g}_0.(\mathfrak{u}^t.f^{(0)}) = \mathfrak{u}^t.(\mathfrak{g}_0.f^{(0)}) \qquad \text{for all } t \geq 0
\end{equation}
Since the right-hand side is the interpretable model trained alone, isolation guarantees that the ideal interpretable model $F^{[h]}(c;\theta^*)$ can always be extracted from a trained residual interpretable model by switching off the residual. In particular, once training converges, the right-hand side reaches the ideal interpretable model, and therefore the trained residual interpretable model $f^{*}$ satisfies:
\begin{equation}
    (\mathfrak{g}_0.f^{*})(z) = F^{[h]}\big(c(z);\theta^*\big)
\end{equation}

As an example, consider an additive residual model~\citep{yuksekgonul2022post}:
\begin{equation}
    f(z) = F^{[h]}\big(c(z);\theta\big) + r(z;\phi)
\end{equation}
where $\mathfrak{g}_0$ maps $r \mapsto 0$. If the two components are trained jointly on $\mathcal{L}(f(z), y)$, the gradients with respect to $\theta$ depend on the residual $r$, and isolation is violated. Isolation is instead satisfied when the interpretable model is trained on its own loss $\mathcal{L}(F^{[h]}(c(z);\theta), y)$, while the residual is trained on:
\begin{equation}
    \mathcal{L}\Big(\mathrm{sg}\big[F^{[h]}\big(c(z);\theta\big)\big] + r(z;\phi),\; y\Big)
\end{equation}
where $\mathrm{sg}[\cdot]$ denotes the stop-gradient operator. In this case, the residual receives the predictions of the interpretable model in the forward pass, but no gradient flows back from the residual to $\theta$. As a result, $\theta$ follows exactly the same updates as in the ideal interpretable model trained alone, while the residual learns the information about $y$ that the ideal interpretable model cannot capture.

\subsection{Consistency: preserving interpretable properties in the residual model}
\label{sec:consistency}
Consistency requires that the residual model preserves the properties of the ideal interpretable model. For instance, consider again the model predicting $f_\texttt{apple}$ from the concepts $c_\texttt{round}$ and $c_\texttt{size}$ and from residual features such as ``\texttt{texture}''. If setting $c_\texttt{round} = 0$ for the apple~\inlineimg{Apple} lowers the ideal interpretable model's logit by $0.5$, we may require that the same intervention lowers the residual model's logit by exactly $0.5$ as well. In other words, texture may change the residual model's prediction, but not how the prediction responds to concept interventions.

To formalise this, let $\mathfrak{g}_{do} \in \mathfrak{G}_{do}$ be an intervention that changes the concepts $c$ while leaving the input $z$ unchanged. Consistency with respect to concept interventions requires that the residual model $f$ responds to every intervention exactly as the ideal interpretable model, with both models compared in logit space:
\begin{equation}
    f(\mathfrak{g}_{do}.c, z) - f(c, z) = F^{[h]}(\mathfrak{g}_{do}.c;\theta^*) - F^{[h]}(c;\theta^*)
\end{equation}
Rearranging, consistency states that the residual contribution $f(c, z) - F^{[h]}(c;\theta^*)$ is invariant to all concept interventions $\mathfrak{g}_{do} \in \mathfrak{G}_{do}$.

\section{Empirical evaluation}
\label{sec:experiments}

We validate our theoretical analysis empirically in controlled settings (Appendix~\ref{app:validation}). The main outcome is that interpretability symmetries enforced through the loss hold only approximately and only on the training distribution, whereas compiling the symmetries into the architecture makes them hold by design across the whole input domain, including on out-of-distribution samples.

Moving beyond controlled settings, we then use the theory to identify limitations of large-scale models and improve their interpretability. Finally, we compare existing interpretable methods for language models against the solutions prescribed by the SIM, assessing language models' concept semantics, accuracy, and formal guarantees.

\subsection{Research questions}

In this section, we analyse the following research questions:
\begin{itemize}
    \item \textbf{Interpretability of pre-trained large-scale models:} Do existing pre-trained large-scale models learn concept maps that match known concept semantics? Can broken concept semantics be fixed in these models without training? To what extent are these models' predictions driven by the concepts they predict?
    \item \textbf{Building interpretable language models (ILMs):} To what extent do ILMs preserve ground-truth concept semantics? Does the expressivity symmetry allow ILMs to match the predictive accuracy of opaque LMs? Is there an advantage in using expressive concept composition maps? What formal guarantees can ILMs provide?
\end{itemize}
Note that the purpose of these experiments is not to claim state-of-the-art results on any given benchmark, but to demonstrate its utility in understanding and controlling large-scale models.

\subsection{Interpretability of pre-trained large-scale models}

We study the interpretability of three Vision-Language Models (VLMs) commonly used in the literature to generate label-free concept annotations (e.g., \citealt{oikarinenlabel} and \citealt{labo}), namely CLIP \citep{radford2021learning}, Moondream2 \citep{moondream2}, and Qwen2 \citep{bai2025qwen3}. We then examine the interpretability of Steerling-8B \citep{steerling2026paper,steerling2026github}, the largest concept-based model currently available. The code to reproduce these experiments is available at \url{https://github.com/pyc-team/pytorch_concepts/tree/master/examples/standard-model}.

\subsubsection{Concept semantics in vision-language models}

\paragraph{Label-free concept maps may have broken concept semantics (Figure~\ref{fig:symI-vlm}).}
One common way to obtain concept annotations at scale is to generate them using pre-trained multimodal models \citep{yuksekgonul2022post,oikarinenlabel,feng2026bayesian}. Here, we examine \emph{to what extent pre-trained multimodal models satisfy Symmetry~\ref{invariance:symm_1}}. In particular, we ask whether the semantics of the concept \texttt{red}, as encoded by vision-language models (VLMs), corresponds to the known ground truth ordering.

To this end, we take an image from the MNIST dataset \citep{mnist} and produce a sequence of copies with progressively increasing red intensity. For CLIP, we test the induced semantics by computing the dot product between each image embedding and the text embedding of the prompt ``red''.
For Moondream2 and Qwen2, we use direct pairwise prompting. Given two images where the second has strictly higher red intensity than the first, we prompt these two VLMs: \textit{``You are shown two images side by side, labelled~A~(left) and~B~(right). Which image shows a stronger degree of red? Reply with a single letter: A or B.''} To improve robustness, we sample each response $3$ times and take the majority vote as the final answer.
Figure~\ref{fig:symI-vlm} shows the full prediction matrix for each model, alongside the expected ground truth.

Our results show that the three models disagree in most cases and significantly deviate from the ground truth. In particular, Moondream2 and Qwen's mistakes are asymmetric: comparing image $A$ against image $B$ can give a different result from comparing image $B$ against image $A$. Thus the inferred ordering depends on the presentation order, which violates a basic consistency requirement for any concept ranking.
These results show that relying on pre-trained models to annotate concepts carries an inherent risk: their concepts can be internally inconsistent and misaligned with the intended human or ground-truth ordering, thereby failing  to satisfy Symmetry~\ref{invariance:symm_1}.
\begin{figure}[h]
    \centering
    \includegraphics[width=0.25\linewidth]{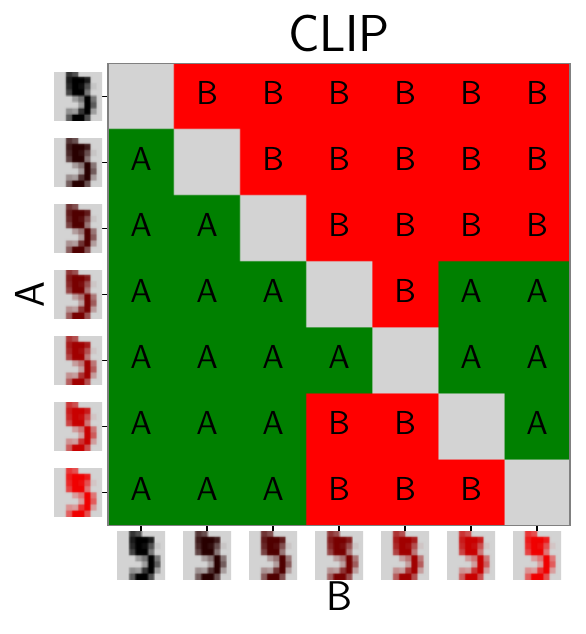}
    \includegraphics[width=0.25\linewidth]{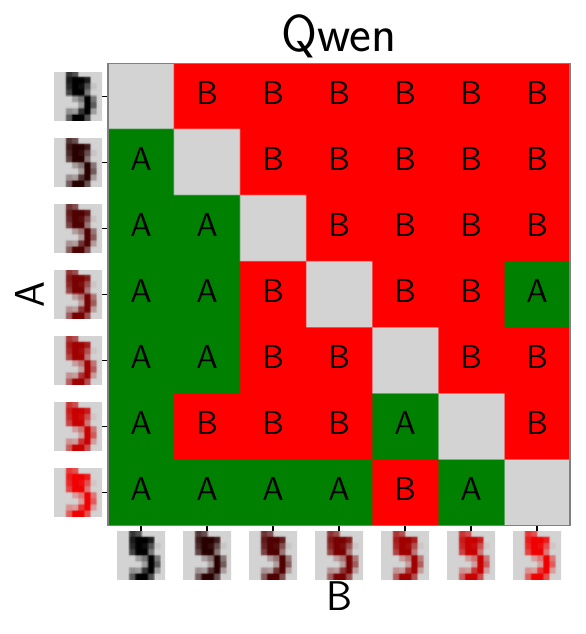}
    \includegraphics[width=0.25\linewidth]{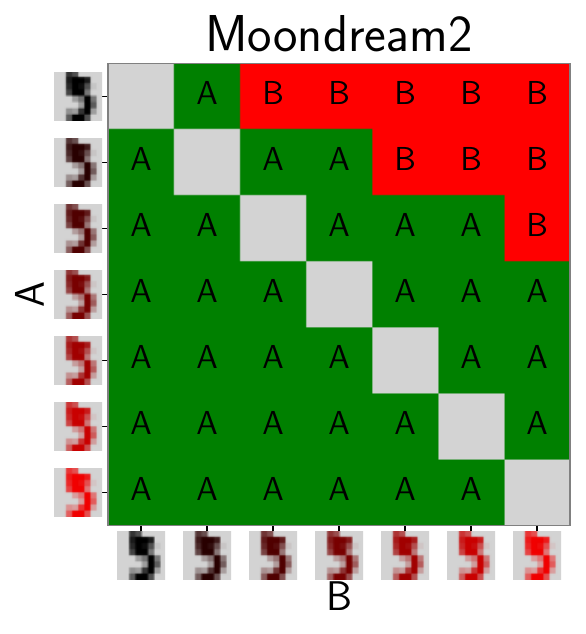}
    \caption{\textbf{Vision-language concept maps.} Heatmaps show all pairwise comparisons of images with increasing red intensity. $B$/red ($A$/green) cells indicate that the model judges image $B$ ($A$) as more red than image $A$ ($B$). A semantics-preserving model should mark the upper triangle as $B$/red and the lower triangle as $A$/green. Pre-trained VLMs can produce inconsistent pairwise rankings, showing that label-free concept annotations need not preserve concept semantics.}
    \label{fig:symI-vlm}
\end{figure}

\paragraph{Label-free concept maps with broken concept semantics can be fixed without training (Figure~\ref{fig:symI-vlm-bis}).}
Having established that the concept semantics of pre-trained models can be unreliable, we next ask whether such misaligned semantics can be repaired \textit{without} fine-tuning. We focus on Moondream2, which made the most errors in the previous experiment, and attempt to correct its semantics for the concept \texttt{red}. To do so, we introduce an architectural bias based on ordered prototypes. We generate a set of prototypical images spanning increasing red intensities, sort them by intensity, and extract their latent representations from the Moondream2 image encoder. Each test image is then assigned to its nearest prototype in latent space (Figure~\ref{fig:symI-vlm-bis}, left), and pairs of test images are ranked according to the ordering of their assigned prototypes (Figure~\ref{fig:symI-vlm-bis}, right). The recovered ranking matches the ground truth perfectly. This suggests that the Moondream2 image encoder contains enough information to recover the correct semantics of \texttt{red}, and that the earlier failure likely arises from how this information is decoded into pairwise judgements.
\begin{figure}[h]
    \centering
    \includegraphics[width=0.48\linewidth]{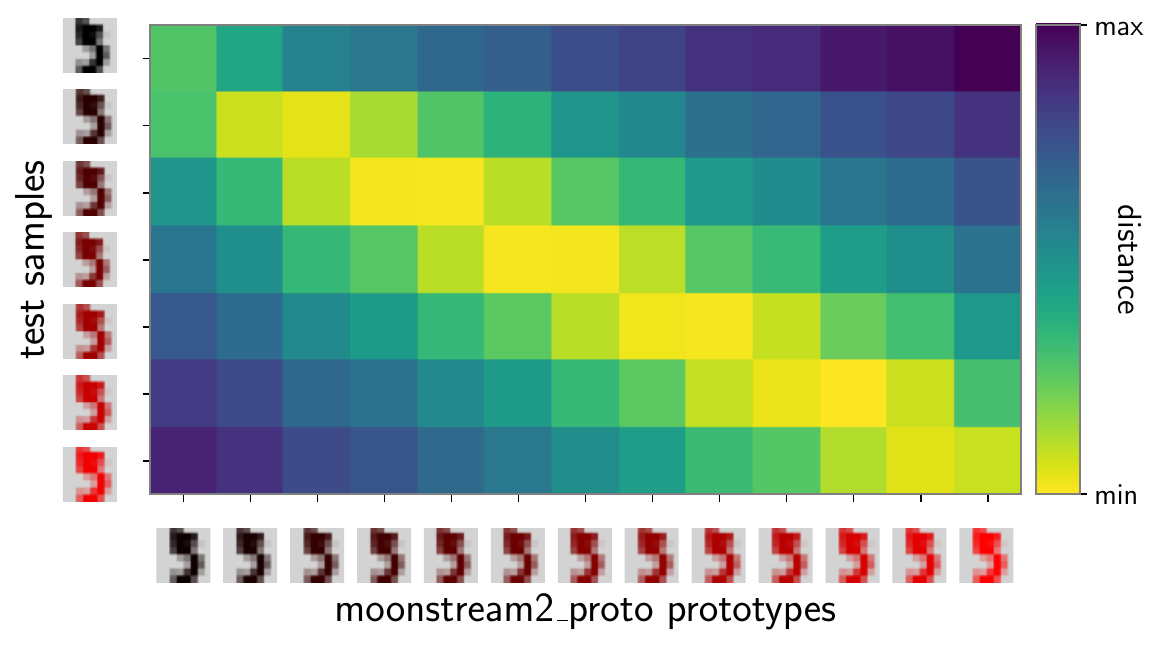}
    \includegraphics[width=0.25\linewidth]{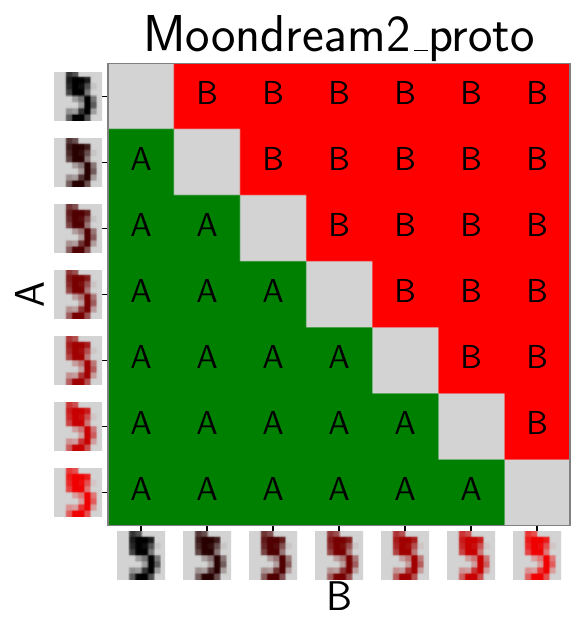}
    \caption{\textbf{Label-free vision-language concept semantics can be perfectly fixed without training.} Left: similarity matrix between test samples and prototypical examples ordered by red intensity. Right: pairwise ranking of test samples obtained by assigning each test sample to its nearest prototype. A semantics-preserving ranking should mark the upper triangule as $B$/red and the lower triangle as $A$/green. Ordered prototypes can recover the correct concept semantics without fine-tuning the VLM.}
    \label{fig:symI-vlm-bis}
\end{figure}

\subsubsection{Prediction-concept dependency in large language models}

\paragraph{Chain-of-thought explanations are not interpretable.}
Chain-of-thought (CoT) \citep{wei2022chain,jie2024interpretable} explanations are natural language sentences produced by a model alongside its predictions, often presented as intermediate reasoning steps leading to the answer. However, under Symmetry~\ref{invariance:symm_2}, an explanation is interpretable only if the prediction actually depends on it. We therefore  test the corresponding constraint by measuring whether the prediction Jacobian $\nabla f$ lies in the subspace spanned by the CoT Jacobian $\nabla c$ (we normalise the score so that $0$ means that the symmetry holds, and $1$ means it is maximally violated). 

We repeatedly prompt Qwen2.5 to produce a prediction together with a CoT explanation, and find consistently high constraint violation scores (above $0.8$). This indicates that the information used to generate the next token is largely independent of the information encoded in the CoT explanation. Thus, although CoT may provide plausible natural-language rationale, it does not necessarily provide an interpretable explanation in the sense of Symmetry~\ref{invariance:symm_2}. This supports prior observations that CoT explanations can be unfaithful \citep{turpin2023language,barez2025chain}, and provides evidence against trating CoT alone as sufficient for model interpretability.

\paragraph{Predictions in large-scale concept-based models depend on a small fraction of concepts (Figure~\ref{fig:steerling}).}

Finally, we ask whether the concept–prediction dependency required by Symmetry~\ref{invariance:symm_2} is better satisfied in large-scale concept-based models. We study Steerling-8B \citep{steerling2026github}, a  publicly available concept-based LLM whose internal representations are organised into $\sim33{,}000$ supervised and $\sim101{,}000$ unsupervised concepts. We restrict our analysis to the supervised concepts, excluding the unsupervised ones from the computation, and evaluate how many concepts the next-token prediction effectively depends on. 

To answer this, we measure the violation of Constraint~\ref{constraint:2} related to Symmetry~\ref{invariance:symm_2} as a function of the concept bottleneck size. Unlike in the CoT experiment above, where the explanation was external to the model computation, Steerling contains and explicit  concept-based hidden state $\bar{h}$ (see the Steerling white paper \citep{steerling2026paper}). We use $\bar{h}$ as a proxy for the next-token prediction $f$, since the two coincide up to a linear map and computing the Jacobian $\nabla_z \bar{h}$ is substantially more efficient. Specifically, we compute the prediction Jacobian $\nabla_z \bar{h}$ and project it onto the column space of the concept Jacobian $\nabla_z c^{(K)}_{\text{sup}}$, obtained by retaining only the top-$K$ activated concepts\footnote{We approximate $\mathrm{rank}(\nabla_z \bar{h})$ by the $99.9\%$-energy effective rank.}. 

\begin{figure}[t]
    \centering
    \includegraphics[width=0.6\linewidth]{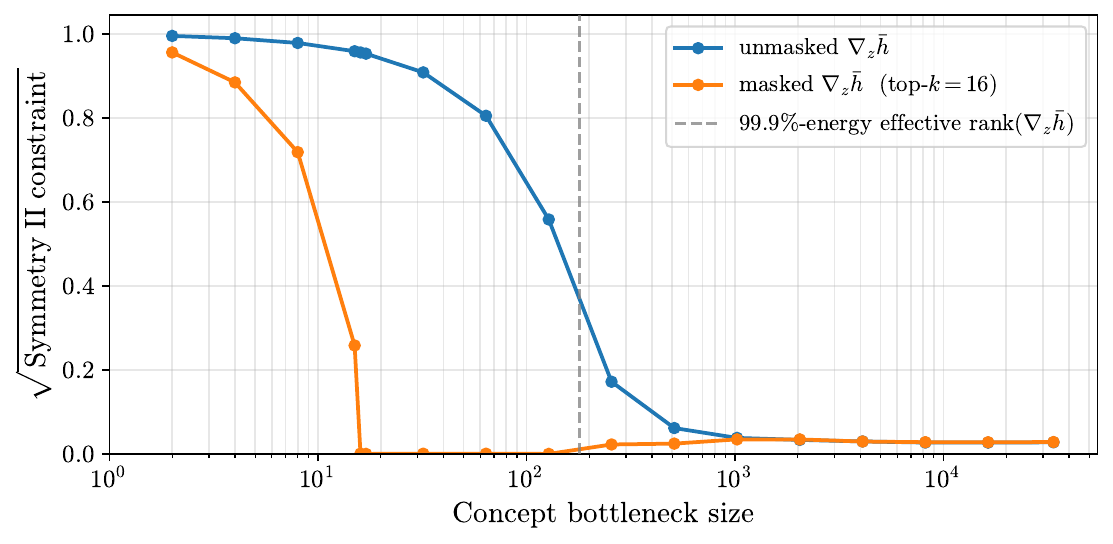} \quad
    \caption{\textbf{In large-scale concept-based models, predictions depend on a small fraction of concepts.} Blue: violation of Constraint~\ref{constraint:2} as the number of retained supervised concepts increases. Predictions in Steerling depend almost exclusively on $\sim500$ concepts. Orange: the same analysis after masking all but the top-$16$ concept activations at test time. The concept dependence of Steerling predictions can be sparsified at test time without retraining or fine-tuning. 
    }
    \label{fig:steerling}
\end{figure}

In our experiment, we progressively increase $K$ (Figure~\ref{fig:steerling}, blue), and observe that the constraint violation rapidly decreases. In particular, the violation drops below $0.1$ at $K \approx 500$, indicating that next-token predictions depend almost entirely on $\sim 1.5\%$ of the supervised concepts. 
This contrasts sharply with the CoT result above: whereas CoT explanations were largely independent of the prediction computation, large-scale concept-based model's predictions are substantially mediated by human-supervised concepts. 

However, even $\sim500$ relevant concepts remain difficult for a human to inspect, compose, or steer directly. We therefore test whether the dependency can be made more sparse at inference time. As a feasibility check, we mask all but the top-$16$ concept activations and recompute $\nabla_z \bar{h}$ and its projections. In this setting, the constraint drops to the numerical floor and, for $K \le 16$, follows $\sqrt{1 - K/16}$ exactly (Figure~\ref{fig:steerling}, orange), quantitatively confirming that the modified prediction structurally depends on exactly $16$ concepts.

This tight concept bottleneck degrades predictive performance, so we do not propose it is a practical deployment setting. Rather, it is illustrates two points. First, large-scale concept-based models can make predictions genuinely depend on supervised concepts, unlike post-hoc rationales that may not mediate the computation. Second, these architectures are amenable to test-time sparsification, enabling fine-grained analyses of which concepts drive a prediction.  A key challenge for future concept-based models is therefore to preserve predictive performance while making this concept dependency sparse enough for human inspection and steering.

\subsection{Interpretable language models}
Having examined the interpretability of pre-trained models, we now turn to interpretable models trained from scratch. In particular, we focus on Interpretable Language Models (ILMs), given the community's recent efforts in this direction~\citep{sun2025concept,steerling2026paper}. We compare the predictive accuracy and interpretability of four ILMs against a non-interpretable LM of similar size, all trained from scratch: a vanilla concept bottleneck LM without any residual (CBM~\citep{cbm,sun2025concept}), a concept bottleneck LM with an unconstrained residual (CRM~\citep{promises_and_pitfalls}), a Steerling-like concept bottleneck LM (Steerling~\citep{steerling2026paper}), and a concept bottleneck LM with an isolated and consistent residual (ICRM). All models share a GPT-2 backbone~\citep{radford2019language}, trained jointly with the rest of the model, and a linear layer mapping the concept bottleneck representation to token predictions~\citep{cbm}. In all ILMs, we apply teacher forcing on concepts~\citep{williams1989learning}, also known as independent training~\citep{cbm}, to train the concept-based predictor efficiently and to prevent task leakage into the concept representations. All models are trained on the same text corpus, based on TinyStories~\citep{eldan2023tinystories}, using $623$ concepts and a vocabulary of $2{,}000$ tokens, with concepts extracted automatically following the Atlas pipeline described by~\citet{steerling2026paper}. Appendix~\ref{app:hyperparams} reports details on the main hyper-parameters, while the code to fully reproduce these experiments is available at \url{https://github.com/pyc-team/babysteerling}.

\paragraph{ILMs compile different interpretability symmetries.}
The four ILMs differ in which interpretability symmetries they compile, and how. Symmetry~\ref{invariance:symm_1} is enforced in the loss for all models through a binary cross-entropy on concepts, while Symmetry~\ref{invariance:symm_3} is architectural for all of them, since the concept-to-token map is linear. Symmetry~\ref{invariance:symm_2} is instead architectural only for CBM, Steerling, and ICRM, and does not hold for CRM, whose residual feeds the prediction without any constraint. The models further differ in how they treat the residual: CBM has none, CRM leaves it unconstrained, Steerling encourages isolation and consistency through its training objective, and ICRM compiles isolation in the architecture while encouraging consistency through its training objective.

\paragraph{Residual ILMs close most of the perplexity gap with opaque LMs (Figure~\ref{fig:perplexity-roc}, left).}
We first ask whether interpretability comes at a cost in fluency, and compute the perplexity of each model on held-out text. The vanilla CBM has a perplexity more than an order of magnitude higher than the opaque LM, confirming that concepts alone are not a sufficient statistic for next-token prediction. Adding a residual closes most of this gap, bringing CRM and Steerling to within a small factor of the opaque LM, and ICRM achieves the lowest perplexity among ILMs. When its residual is switched off, ICRM generates from concepts alone with a perplexity in the same regime as the vanilla CBM. We attribute ICRM's advantage over Steerling to the different roles played by the residual: Steerling's residual is trained to avoid encoding concept information, which reduces its capacity, whereas ICRM's residual is free to reuse concept information, since isolation already prevents it from affecting the interpretable model.

\begin{figure}
    \centering
    \includegraphics[width=0.49\linewidth]{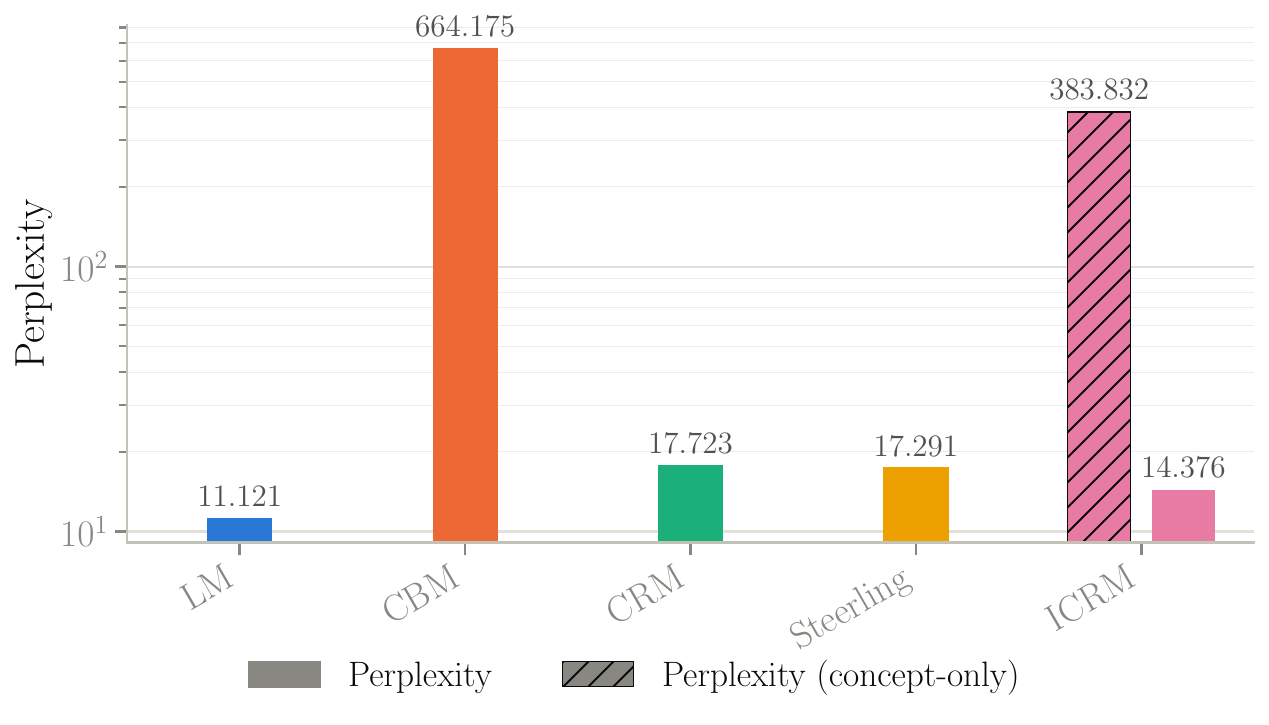}
    \includegraphics[width=0.49\linewidth]{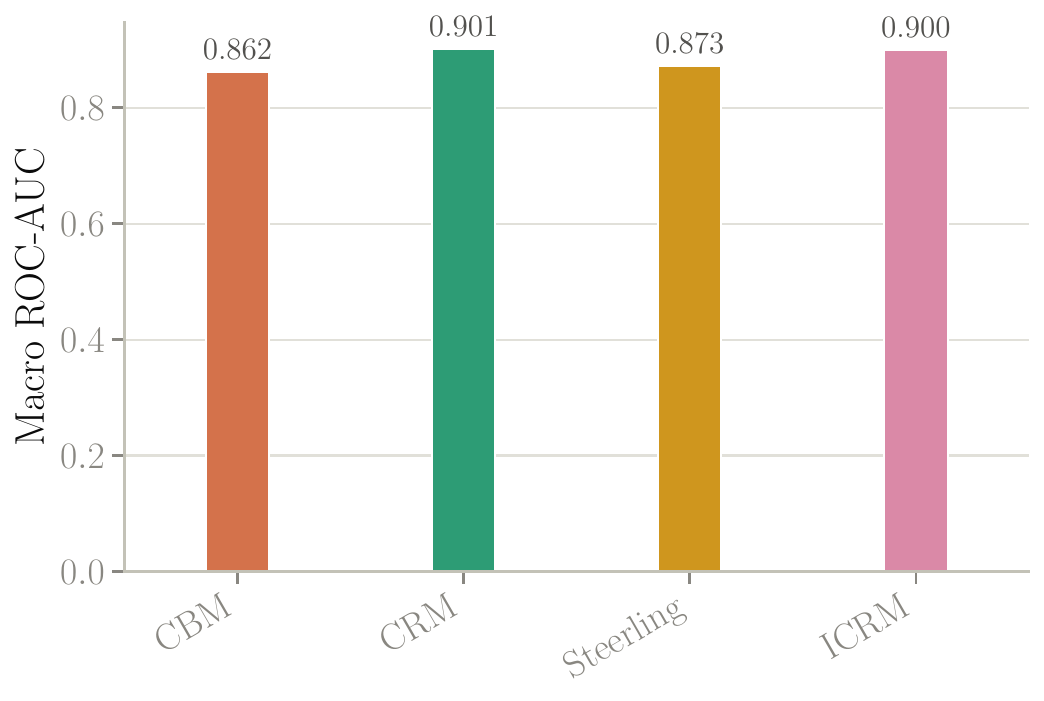}
    \caption{\textbf{Predictive performance and concept alignment of ILMs.} Left: validation perplexity (log scale) of an opaque LM and four ILMs; the hatched bar shows ICRM generating from concepts only, with the residual switched off. Right: macro-averaged concept ROC-AUC.}
    \label{fig:perplexity-roc}
\end{figure}

\paragraph{Adding a residual does not degrade concept alignment (Figure~\ref{fig:perplexity-roc}, right).}
Before testing whether ILMs can be controlled through concepts, we check Symmetry~\ref{invariance:symm_1}, that is, whether concepts are semantically aligned with their annotations, by computing the macro-averaged concept ROC-AUC on held-out data. All models predict concepts accurately, with CRM and ICRM slightly outperforming CBM and Steerling. Adding a residual therefore does not compromise the semantic alignment of concepts.

\paragraph{ICRM preserves the causal effect of concepts on predictions (Figure~\ref{fig:cace}, left).}
We next quantify how strongly predictions depend on concepts by computing the average causal effect of a concept intervention on the next-token distribution. An unconstrained residual absorbs most of this effect, which drops to about a third of the vanilla CBM's in CRM. ICRM retains most of the effect, about twice that of CRM, and the effect stays the same whether the residual is active or switched off. This is the behaviour that consistency targets: intervening on a concept changes the prediction of the full model as it changes the prediction of the interpretable model, rather than being absorbed by the residual. Steerling shows the weakest causal effect among the four models. Unlike the other models, which are intervened on by changing concept activations, Steerling requires injecting concept embeddings into the early layers of the backbone~\citep{steerling2026paper}, so that the effect of an intervention on the prediction is mediated by the rest of the backbone.

\begin{figure}
    \centering
    \includegraphics[width=0.49\linewidth]{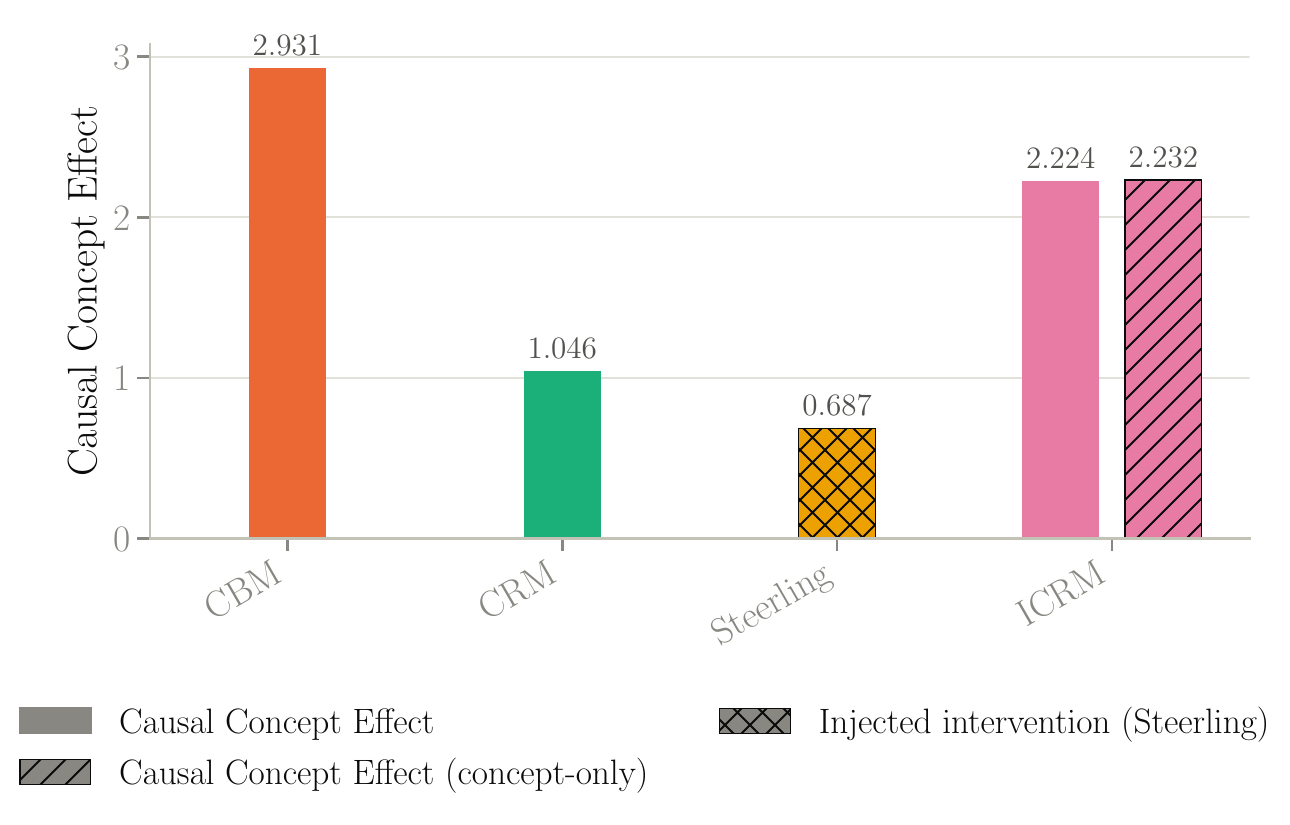}
    \includegraphics[width=0.49\linewidth]{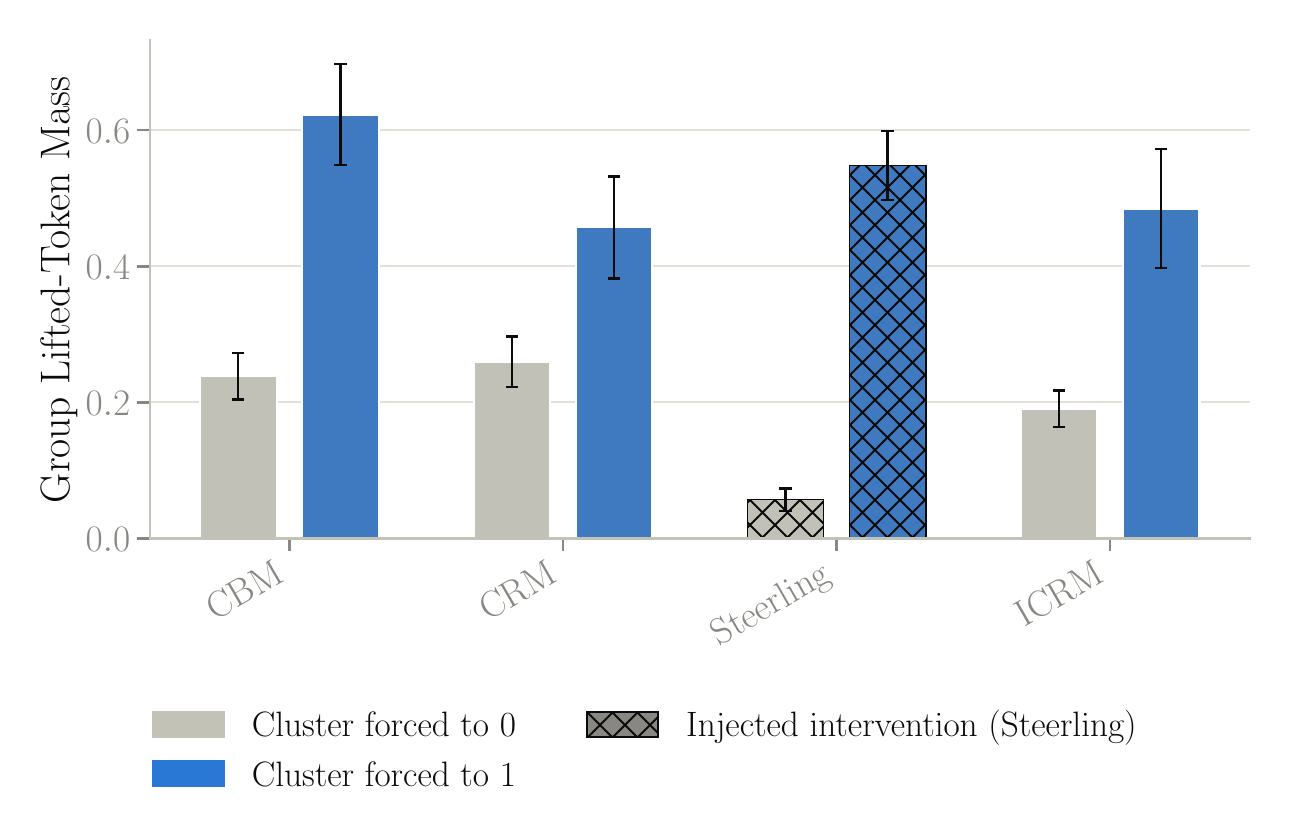}
    \caption{\textbf{Concept-based control of ILMs.} Left: average causal effect of interventions on concept activations on the next-token distribution; the hatched bar shows ICRM with the residual switched off. Right: probability mass on the tokens associated with a concept group when the concept activations of the group are set to $0$ (grey) or $1$ (blue); error bars show $95\%$ confidence intervals of the mean across $64$ evaluation samples. Steerling bars have a different texture, as Steerling interventions inject concept embeddings into early backbone layers, whereas all other models are intervened on by changing concept activations only.}
    \label{fig:cace}
\end{figure}

\paragraph{Concept interventions steer generation towards the corresponding tokens (Figure~\ref{fig:cace}, right).}
Finally, we test whether concept interventions translate into controllable generation. For each group of related concepts, we set the concept activations of the group to $0$ or to $1$, or inject the corresponding concept embeddings in Steerling, and measure the probability mass that the model assigns to the tokens associated with that group. In CBM, setting the activations to $1$ more than doubles this mass, while in CRM the increase is smaller. ICRM behaves like the vanilla CBM, with the mass more than doubling. Steerling shows the sharpest contrast between the two conditions, as suppressing a concept group drives the mass close to zero, while injecting the group raises the mass to a level comparable with CBM.


\section{Discussion} 
\label{sec:discussion}

\subsection{Limitations and opportunities in interpretability research}
\label{sec:instantiations}
The SIM subsumes well-known families of interpretable models as special cases, enabling a systematic comparison of their interpretability properties. The comparison exposes limitations that are sometimes overlooked in the original formulations and points to under-explored research directions. We discuss these connections and limitations below, and provide an overview of representative models from the literature in Table~\ref{tab:instantiations}.


\paragraph{Traditional interpretability: concept semantics is assumed rather than compiled.}
Linear models~\citep{cox1958regression}, decision trees~\citep{breiman1984classification,hu2019optimal}, neural additive models~\citep{agarwal2021neural}, and feature attribution methods~\citep{lime,shap} assume that their inputs already reside in an aligned concept space, that is, they assume identity concept maps. Symmetry~\ref{invariance:symm_1} is therefore not built into these models but imposed as an external assumption, which rarely holds in practice and fails by default for high-dimensional, unstructured inputs such as images or text. Since Symmetries~\ref{invariance:symm_2} and~\ref{invariance:symm_3} are defined with respect to concepts, the two symmetries inherit the same failure, because when input features do not correspond to human concepts, a prediction that depends only on input features, or that composes input features through a tractable formula, provides no guarantee about its dependence on concepts. The few traditional models that reach full expressivity do not compensate for the missing symmetries. Decision trees become universal approximators only as their depth grows beyond what a human can inspect, and post-hoc attribution methods explain an unconstrained black box instead of constraining the black box. A black box can combine input features in infinitely many ways that produce almost identical predictions, so an attribution need not reflect the mechanism the black box actually uses~\citep{adebayo2018sanity,bilodeau2024impossibility}.

\paragraph{Intrinsic interpretability: few models combine the interpretability symmetries with expressivity.}
Concept-based models address Symmetry~\ref{invariance:symm_1} by supervising concept maps with human annotations, but every model in Table~\ref{tab:instantiations} enforces Symmetry~\ref{invariance:symm_1} through a loss rather than compiling the symmetry in the architecture. Moreover, standard formulations require predicted concept scores to match annotated scores, rather than preserving the ordering induced by the annotations, which over-constrains the concept maps and degrades performance when annotations are noisy~\citep{preference_optimization_cbms} or incomplete~\citep{cem}. Symmetries~\ref{invariance:symm_2} and~\ref{invariance:symm_3} are compiled architecturally by pure bottleneck models~\citep{cbm,glancenets,probabilistic_cbms,stochastic_cbms}, but pure bottleneck models give up expressivity, since concepts are rarely a sufficient statistic for the task. Most models recover expressivity by adding a residual~\citep{promises_and_pitfalls,cem,ismail2024concept,xu_energy_based_cbm_iclr_24}, breaking Symmetry~\ref{invariance:symm_2} and often Symmetry~\ref{invariance:symm_3}. A few exceptions include Post-hoc CBMs~\citep{yuksekgonul2022post}, which compile Symmetries~\ref{invariance:symm_2} and~\ref{invariance:symm_3}, full expressivity, and isolation, since the residual is fitted after the concept-based predictor, but do not consider consistency. Prototype-based CBMs~\citep{colamonaco2026prototype} enforce Symmetry~\ref{invariance:symm_1} in the loss while allowing concept semantics to be fixed post hoc through prototypes. Steerling~\citep{steerling2026paper} compiles Symmetries~\ref{invariance:symm_2} and~\ref{invariance:symm_3} together with full expressivity, but enforces isolation and consistency only through the training objective. ICRM addresses every dimension in Table~\ref{tab:instantiations}, compiling Symmetries~\ref{invariance:symm_2} and~\ref{invariance:symm_3}, full expressivity, and isolation in the architecture, and enforcing Symmetry~\ref{invariance:symm_1} and consistency in the loss.

{
\citestyle{plain}
\setlength{\LTcapwidth}{\textwidth}
\setlength{\LTleft}{0pt}
\setlength{\LTright}{0pt}
\footnotesize
\renewcommand{\arraystretch}{1.2}
\setlength{\tabcolsep}{3pt}
\begin{xltabular}{\textwidth}{>{\raggedright\arraybackslash}p{0.22\textwidth} *{6}{>{\centering\arraybackslash}X}}
\caption{Comparison of interpretable and residual interpretable models, grouped into traditional interpretable AI, intrinsic interpretability, and mechanistic interpretability. \yes~compiled in the architecture, \soft~enforced in the loss or only partially, \no~not considered, \na~not applicable, \no$\,\rightarrow\,$\yes~not satisfied by default but achievable post hoc.}
\label{tab:instantiations} \\
\toprule
& \multicolumn{3}{c}{\textbf{Interpretability Symmetries}} & \multicolumn{3}{c}{\textbf{Expressivity}} \\
\cmidrule(lr){2-4} \cmidrule(lr){5-7}
\textbf{Method} & Symmetry~\ref{invariance:symm_1} & Symmetry~\ref{invariance:symm_2} & Symmetry~\ref{invariance:symm_3} & Expressivity~\ref{sec:expressivity} & Isolation~\ref{sec:isolation} & Consistency~\ref{sec:consistency} \\
\midrule
\endfirsthead
\multicolumn{7}{l}{\textit{Table~\ref{tab:instantiations} continued from previous page}} \\
\toprule
& \multicolumn{3}{c}{\textbf{Interpretability Symmetries}} & \multicolumn{3}{c}{\textbf{Expressivity}} \\
\cmidrule(lr){2-4} \cmidrule(lr){5-7}
\textbf{Method} & Symmetry~\ref{invariance:symm_1} & Symmetry~\ref{invariance:symm_2} & Symmetry~\ref{invariance:symm_3} & Expressivity~\ref{sec:expressivity} & Isolation~\ref{sec:isolation} & Consistency~\ref{sec:consistency} \\
\midrule
\endhead
\midrule
\multicolumn{7}{r}{\textit{Continued on next page}} \\
\endfoot
\bottomrule
\endlastfoot
\multicolumn{7}{l}{\textit{Traditional interpretability}} \\
LR~\citep{cox1958regression} & \no & \no & \no & \no & \no & \no \\
DT~\citep{breiman1984classification} & \no & \no & \no & \yes & \no & \no \\
LIME~\citep{lime} & \no & \no & \no & \yes & \no & \no \\
SHAP~\citep{shap} & \no & \no & \no & \yes & \no & \no \\
NAM~\citep{agarwal2021neural} & \no & \no & \no & \no & \no & \no \\
OSDT~\citep{hu2019optimal} & \no & \no & \no & \soft & \no & \no \\
\midrule
\multicolumn{7}{l}{\textit{Intrinsic interpretability}} \\
DeepProbLog~\citep{manhaeve2018deepproblog} & \no & \yes & \yes & \no & \yes & \yes \\
SENN~\citep{senn} & \noyes & \no & \yes & \yes & \no & \no \\
ProtoPNet~\citep{protopnet} & \noyes & \yes & \yes & \yes & \no & \no \\
CBM~\citep{cbm} & \soft & \yes & \yes & \no & \yes & \yes \\
CRM~\citep{promises_and_pitfalls} & \soft & \no & \yes & \yes & \no & \no \\
CBM-AUC~\citep{sawada2022concept} & \soft & \no & \yes & \yes & \no & \no \\
CEM~\citep{cem} & \soft & \no & \no & \yes & \no & \soft \\
GlanceNets~\citep{glancenets} & \soft & \yes & \yes & \no & \yes & \yes \\
Post-hoc CBM~\citep{yuksekgonul2022post} & \soft & \yes & \yes & \yes & \yes & \no \\
Label-free CBM~\citep{oikarinenlabel,labo} & \soft & \yes & \yes & \no & \yes & \yes \\
ProbCBM~\citep{probabilistic_cbms} & \soft & \yes & \yes & \no & \yes & \yes \\
IntCEM~\citep{intcem} & \soft & \no & \no & \yes & \no & \soft \\
CBGM~\citep{ismail2024concept} & \soft & \no & \no & \yes & \no & \soft \\
ECBM~\citep{xu_energy_based_cbm_iclr_24} & \soft & \no & \no & \yes & \no & \no \\
Res-CBM~\citep{shang2024incremental} & \soft & \yes & \yes & \yes & \yes & \no \\
PPN~\citep{tan2024post} & \noyes & \no & \no & \yes & \no & \no \\
CGM~\citep{dominici2024causal} & \soft & \no & \no & \yes & \no & \no \\
CMR~\citep{debot2024interpretable} & \soft & \yes & \yes & \yes & \no & \no \\
R-CBM~\citep{barbiero2024relational} & \soft & \yes & \yes & \no & \yes & \yes \\
SCBM~\citep{stochastic_cbms} & \soft & \yes & \yes & \no & \yes & \yes \\
VLG-CBM~\citep{srivastava2024vlg} & \soft & \yes & \yes & \no & \yes & \yes \\
IBB~\citep{beyond_cbms} & \soft & \no & \no & \yes & \no & \soft \\
CF-CBM~\citep{dominici2025counterfactual} & \soft & \yes & \yes & \no & \yes & \yes \\
CB-LLM~\citep{sun2025concept} & \soft & \no & \soft & \yes & \soft & \soft \\
Hybrid CBM~\citep{liu2025hybrid} & \soft & \yes & \yes & \no & \yes & \yes \\
C2BM~\citep{de2025causally} & \soft & \no & \yes & \yes & \no & \no \\
PCBM~\citep{colamonaco2026prototype} & \softyes & \yes & \yes & \no & \yes & \yes \\
Steerling~\citep{steerling2026paper} & \soft & \yes & \yes & \yes & \soft & \soft \\
\midrule
\multicolumn{7}{l}{\textit{Mechanistic interpretability}} \\
Probing~\citep{alain2016understanding} & \soft & \no & \no & \yes & \no & \no \\
Feature visualization~\citep{olah2017feature} & \no & \no & \no & \yes & \no & \no \\
TCAV~\citep{tcav} & \soft & \no & \no & \yes & \no & \no \\
Circuits~\citep{olah2020zoom} & \no & \no & \no & \yes & \no & \no \\
Logit lens~\citep{nostalgebraist2020interpreting,belrose2023eliciting} & \no & \no & \no & \yes & \no & \no \\
Activation patching~\citep{meng2022locating,zhang2024towards} & \no & \no & \no & \yes & \no & \no \\
Induction heads~\citep{olsson2022context} & \no & \no & \no & \yes & \no & \no \\
Toy models~\citep{elhage2022toy} & \no & \no & \no & \no & \no & \no \\
SoLU~\citep{elhage2022solu} & \no & \no & \no & \yes & \no & \no \\
Causal scrubbing~\citep{chan2022causal} & \soft & \no & \no & \yes & \no & \no \\
CCS~\citep{burns2022discovering} & \no & \no & \no & \yes & \no & \no \\
RepE~\citep{zou2023representation} & \soft & \no & \no & \yes & \no & \no \\
ACDC~\citep{conmy2023towards} & \no & \no & \no & \yes & \no & \no \\
SAE~\citep{cunningham2023sparse} & \no & \no & \no & \yes & \no & \no \\
Transcoders~\citep{dunefsky2024transcoders} & \no & \no & \no & \yes & \no & \no \\
Causal abstraction~\citep{geiger2024finding,geiger2025causal} & \no & \no & \no & \yes & \no & \no \\
Attribution graphs~\citep{ameisen2025circuit} & \no & \no & \no & \yes & \no & \no \\
Weight-sparse transformers~\citep{gao2025weight} & \no & \no & \no & \soft & \no & \no \\
NLA~\citep{frasertaliente2026nla} & \no & \no & \no & \yes & \no & \no \\
J-lens~\citep{gurnee2026verbalizable} & \no & \no & \no & \yes & \no & \no \\
Block-sparse featurizers~\citep{fel2026structuring} & \no & \no & \no & \yes & \no & \no \\
Clean~\citep{cammarata2026clean} & \no & \no & \no & \yes & \no & \no \\
\midrule
ICRM (this work) & \softyes & \yes & \yes & \yes & \yes & \soft \\
\end{xltabular}
}

\paragraph{Mechanistic interpretability: full expressivity without interpretability guarantees.}
Mechanistic interpretability methods analyse a model after training, so the analysed model retains full expressivity by construction, but for the same reason none of the interpretability symmetries is constrained. Symmetry~\ref{invariance:symm_1} is enforced at best through a loss, when concepts are defined by human-labelled data as in probing~\citep{alain2016understanding}, concept activation vectors~\citep{tcav}, and representation engineering~\citep{zou2023representation}. Symmetry~\ref{invariance:symm_1} is otherwise absent, when features are discovered without supervision as in sparse autoencoders~\citep{cunningham2023sparse}, transcoders~\citep{dunefsky2024transcoders}, block-sparse featurizers~\citep{fel2026structuring}, or lenses that read activations through the model's own vocabulary~\citep{nostalgebraist2020interpreting,gurnee2026verbalizable}. Discovered concepts must then be aligned with human semantics post hoc, and since each new model version may produce an entirely different set of concepts, the alignment must restart with every release, which limits scalability. Symmetries~\ref{invariance:symm_2} and~\ref{invariance:symm_3} fail because the analysed model remains free to compute its outputs from any information in its activations. Circuits, attribution graphs, and causal abstractions~\citep{olah2020zoom,geiger2025causal,ameisen2025circuit} therefore describe a subset of the computation without constraining the computation, which leads to limited steerability~\citep{wu2025axbench} and to accuracy losses when predictions are forced through the extracted features~\citep{luo2026_skop}. Isolation fails as well, since the extracted features are entirely determined by the model being explained. Methods that intervene on the architecture, such as softmax linear units~\citep{elhage2022solu} or weight sparsity~\citep{gao2025weight}, constrain how a model computes but not what the model's components mean.

\paragraph{Research opportunities.}
Two patterns emerge from this comparison. First, there is consensus on how to compile Symmetries~\ref{invariance:symm_2} and~\ref{invariance:symm_3} through architectural choices, whereas Symmetry~\ref{invariance:symm_1} is never compiled, since existing methods either assume Symmetry~\ref{invariance:symm_1}, enforce Symmetry~\ref{invariance:symm_1} through a loss, or, at best, fix concept semantics post hoc, as in prototype-based models such as ProtoPNet~\citep{protopnet} and prototype-based CBMs~\citep{colamonaco2026prototype}. How to compile Symmetry~\ref{invariance:symm_1} efficiently at scale therefore remains an open research question. Second, compiling the interpretability symmetries and building expressive models need not be alternatives, as Post-hoc CBMs~\citep{yuksekgonul2022post} and ICRM achieve both by combining a concept-based predictor with an isolated residual. Consistency, however, remains the least explored dimension, and which consistency constraints are useful in practice, beyond the preservation of concept interventions studied here, is an under-explored research direction.\looseness-1

\subsection{Falsifiability and evaluation protocols}

The SIM makes interpretability claims falsifiable along three distinct axes:
\begin{itemize}
    \item \textbf{Analytical:} Claims such as ``$f$ satisfies symmetry $S$'' can be disproven analytically by checking whether the corresponding invariance or constraint holds.
    \item \textbf{Empirical:} Claims that a ``model satisfying constraint $C$ cannot solve a task $T$ or represent a function $F$'' can be empirically falsified (e.g., through controlled benchmark tests such as Weisfeiler-Leman-style evaluations).
    \item \textbf{Epistemic:} Claims of the form ``a human can correctly predict or intervene on this model's behaviour on unseen inputs'' can in principle be falsified via user studies.
\end{itemize}

\paragraph{Evaluation protocols.}
A direct consequence of the proposed theoretical framework and the SIM is the explicit separation between the mathematical properties of interpretability and the user-dependent choices surrounding them. Given a fixed set of premises, the extent to which a model is interpretable is a formal property: it is fully and objectively determined by the corresponding interpretability symmetries and constraints. By contrast, which premises are adopted in the first place, and how the resulting model behaviour visualised or communicated, are user-dependent choices that may vary across users, tasks and domains. 

This separation has a direct and important consequence for evaluation protocols. User-dependent components should be assessed through user studies or qualitative data analysis, since they are inherently tied to human judgment and context. In contrast, \emph{formal properties of interpretability} are user-agnostic and must therefore be evaluated either analytically or quantitatively (i.e., they do not necessitate a user study) \citep{freiesleben2023dear}.


\subsection{Related works}
\label{sec:related-works}

\paragraph{Interpretability as constrained optimisation.}
The perspective that interpretability can be understood through constrained optimisation has been discussed by~\citet{rudin2019stop, rudin2022interpretable}. However, these works leave two gaps that the SIM aims to address. First, they do not provide a unified methodology for deriving interpretability constraints from underlying symmetries, nor for translating those constraints into architectures or learning procedures. Second, they identify \emph{sparsity} as a central interpretability principle and constraint. While sparsity can be useful in many setting, it is not a universal requirement for interpretability. Indeed, as noted by \citet{rudin2019stop}, ``in some domains, sparsity is useful, and in others is it not''. Sparsity should therefore be treated as one possible modelling constraint, rather than a fundamental invariant principle underlying all interpretable methods.

\paragraph{Constraining machine learning models.}
A long line of work incorporates prior knowledge into ML models by adding constraint-violation terms to their training objectives. In these contexts, constraints act as soft preferences: the learner is encouraged, but not guaranteed, to satisfy them. Influential examples include posterior regularisation \citep{ganchev2010posterior}, which restricts posterior distributions through expectation-based constraints, and semantic-based regularisation \citep{diligenti2017semantic, xu2018semantic, marra2019lyrics, fischer2019dl2, marra2024statistical}, which compiles logical knowledge into differentiable penalties. 
Constraint penalties are also central to physics-informed neural networks (PINNs) \citep{raissi2017physics}, and to Lagrangian-based constrained empirical risk minimisation \citep{narasimhan2018learning, cotter2019training, kervadec2022constrained}. Nevertheless, in all of these cases, unless the penalty is optimised exactly and is well behaved, the learned predictor may still violate the desired specification at test time.

A complementary line of work enforces constraints by construction, changing the class of functions the model can represent to guarantee the satisfiability of a constraint. This can be done by: (1)~appending differentiable layers that map raw predictions into the feasible set while preserving end-to-end trainability \citep{marquez2017imposing, min2024hardnet, yang2023safe};
(2)~embedding optimisation problems directly as neural layers, allowing backpropagation through the solution map \citep{amos2017optnet, agrawal2019differentiable, gould2021deep}; or (3)~imposing structural biases, as in monotonic networks \citep{sill1997monotonic}, deep lattice networks \citep{deep_lattice_networks}, and deep sets \citep{zaheer2017deep}.

The SIM incorporates both approaches: ``soft'' constraint enforcement during optimisation in \learningText{Phase~II}, and ``hard'' architectural enforcement in \architecturesText{Phase~III}. The main difference is that, in this work, constraints are not arbitrary domain specifications, but are derived from interpretability premises via symmetries. The SIM therefore provides a method for designing constraints that capture properties traditionally associated with interpretable models, and then proposes two distinct, but complementary paths to enforce those constraints in practice.

\paragraph{Geometric and physics-inspired machine learning.}
Geometric \citep{bronstein2021geometric} and physics-inspired \citep{guo2025physics} machine learning frameworks align closely with the methodology proposed in this work, since they also organise model design around invariances, symmetries and variational principles. In particular, both frameworks seek to identify a shared set of common principles, or symmetries, from which which model classes, learning objectives, and architectures can be derived. However, these frameworks primarily target general machine learning, rather than the specific requirements of interpretability. They tehrefore do not address the central requirement considered here: aligning model computation with human-understandable concepts, ensuring that predictions depend on those concepts, and accounting for the bounded nature of human reasoning when designing or training models.

\paragraph{General interpretability theories.}

Foundational works in Explainable AI (XAI) \citep{xai_darpa, xai_darpa_2} have sought to clarify what interpretability is and how it should be evaluated. However, most of these efforts remain primarily conceptual, without offering a formal or measurable account of interpretability. For example, early work by \citet{doshi2017towards} argues that interpretability claims must be operationalised and matched to an appropriate evaluation level, but stops short of providing a formal definition that can be tested directly. Follow-up work by \citet{lipton2018mythos} argues that interpretability is an underspecified umbrella term that conflates distinct desiderata such as simulatability, decomposability, and algorithmic transparency, and that these properties can compete rather than coexist. In contrast, we treat several of these desiderata as core components of interpretability (e.g., simulatability and decomposability are captured through Symmetry~\ref{invariance:symm_3}). Other works, such as \citet{miller2019explanation}, draw on philosophy, cognitive science, and social psychology to argue that explainable AI should be grounded in the ways humans actually explain to one another, a component the SIM captures via its target entity $h$ \citep{watson2021explanation}. Finally, more recent works, such as those by \citet{giannini2024categorical} and \citep{ bordt2025rethinking}, provide formal frameworks for capturing particular aspects of interpretability. However, they do not show how such frameworks can be used to derive concrete interpretable architectures, as can be done within the SIM. 

More recently, there have been calls (e.g., by~\citealt{rudin2019stop} and \citealt{rudin2022interpretable}) to design architectures that are interpretable by construction \citep{senn, protopnet, agarwal2021neural, ma2024looks, steerling2026github}. These calls came amidst results suggesting that popular post-hoc explainability methods, such as feature importance methods \citep{lime, anchors, shap, vstrumbelj2014explaining} or saliency methods \citep{og_saliency, grad_cam, integrated_gradients}, can be fragile and misleading \citep{fragile_saliency_maps, adebayo2018sanity, saliency_manipulation, unreliability_saliency, molnar2020general}. As a result, concept-based interpretable models have received increasing attention \citep{concept_whitening, cbm, cem, intcem, glancenets, xu_energy_based_cbm_iclr_24, yuksekgonul2022post, labo, oikarinenlabel, barbiero2024relational, dominici2024causal, almudevar2025there, feng2026bayesian}. Our framework and the SIM directly support research in this direction by introducing formal, operational mechanisms for compiling interpretability premises into architectural components (Section~\ref{sec:architecture}), and by framing concepts as first-class objects of study in interpretability (e.g., Symmetry~\ref{invariance:symm_1}).

\paragraph{Pragmatic and actionable interpretability.} 
A more recent line of work argues that progress in interpretability has been stalled by a lack of ``purpose'', and calls for interpretability research to be tied to concrete downstream outcomes. For example, \citet{nanda2025pragmatic} advocate a pragmatic turn in mechanistic interpretability research, where methods are selected by their ability to solve problems on the critical path to safe AGI. Similarly, \citet{orgad2026interpretability} argue that XAI research should include \emph{actionability}, defined along the axes of concreteness and validation, as a core evaluation criterion. Although the sentiment in both of these efforts is similar to the motivation of this paper, they do not prescribe how to operationalise actionability at the level of model design. Our framework addresses this gap by tying interpretability premises to symmetries that are both testable and directly implementable in a model's training or architecture.

\subsection{Open challenges}

We identify three primary open challenges related to the Standard Interpretable Model. First, it remains unclear  whether the premises and symmetries introduced here offer an exhaustive characterisation of interpretability, or whether additional fundamental symmetries are needed. Second, further analysis is needed to determine whether the derivation of constraints, architectures, and learning problems could be automated, and whether constraints could be rewritten in more efficient ways to ensure a systematic and scalable implementation. Finally, the theoretical consequences of this work require deeper exploration. Since the SIM builds on existing theories from geometry, physics and constraint optimisation, it likely inherits a wealth of existing results that have yet been fully mapped and exploited in the context of interpretable machine learning.

\section{Conclusion}

In this work, we propose the Standard Interpretable Model (SIM), a general theory and framework to guide interpretable machine learning research. Specifically, given a set of interpretability desiderata captured as premises, we show how the SIM can derive symmetries from which one can deductively derive constraints, architectures, and learning dynamics for building interpretable models. These theoretical contributions are supported by empirical analyses demonstrating the utility of the SIM in evaluating and improving models. Finally, we use our framework to systematically characterise gaps in the existing interpretability literature, identify concrete research directions, and provide a software library for exploring them.

\subsection{Significance for interpretability research}

The primary significance of this work is that it provides a general, standard method for deductively deriving interpretability methods across the entire developmental pipeline, from abstract theoretical properties down to the granular details of parameter updates and implementation. Through this standardisation, the SIM  enables researchers to compare existing methods systematically and identify their limitations. Finally, by anchoring interpretability in established formalisms of logic, geometry and physics, this approach opens new research directions, such as exploring alternative Lagrangian formulations to derive specialised gradient-descent algorithms with interpretability guarantees.

\subsection{Broader impact}
Beyond its research implications, we hope the SIM can serve as a pedagogical foundation for teaching interpretability in higher education, since it is built on principles familiar to students from diverse quantitative backgrounds can relate to (e.g., constrained optimisation, geometry and probability). Over time, we hope that this work will shift the perception of interpretability, both within and beyond machine learning, from an informal collection of methods to a rigorous research field.


\section*{Acknowledgments}
We thank Aya Abdelsalam Ismail, Giang Nguyen, Alberto Tonda, and Alberto Termine for their useful feedback on a preliminary version of this work and Valentina Bucca for the logo design.
This work is supported by the Swiss National Science Foundation (SNSF), through the project ``IMAGINE'' (grant ID 224226) and the project ``PROSELF'' (grant ID 205121), by the Hasler Foundation, through the project ``Towards Scalable Multimodal Causal Deep Learning'' (grant ID 2024-05-15-70), and by the Research Foundation Flanders, through the project ``Relational Concept-Based Models'' (grant ID G033625N).  M.E.Z. acknowledges support from Trinity College, Oxford, as part of a Junior Research Fellowship (JRF). R.N. was supported by the European Union and the Czech Ministry of Education,
Youth and Sports (Project: MSCA Fellowship CZ FZU III - CZ.02.01.01/00/22 010/0008598). We used language models to double-check and stress-test mathematical formalisations, rephrase sentences, and assist with coding.


\bibliography{main}
\bibliographystyle{tmlr}


\clearpage
\newpage
\appendix

\etocdepthtag.toc{appendix}
\etocsettagdepth{main}{none}       
\etocsettagdepth{appendix}{toc}   

\section{Dummy standard model example}
\label{app:dummy}
To provide an immediate intuition of how the SIM works in practice, we illustrate the proposed procedure on a simple use case.

Let $\mathcal{L}(f(z;\theta),y)$ be a function capturing the error made by a function~$f: Z\to Y$ parametrised by~$\theta$ when predicting a target $y$:
\begin{enumerate}
    \item 
    \symmetriesText{\emph{Premise}}: We first specify what $f$ should satisfy. For our dummy model this is: ``The output of $f$ must remain the same regardless of the order of its input features $z$''.
    \item 
    \symmetriesText{\emph{Symmetry}}: Next, we formalise the premise via the transformations $\mathfrak{g}: z \mapsto \pi(z)$, which permute the input features (e.g., from $[2, 3]$ to $[3, 2]$). Any function respecting this premise must satisfy the invariance condition $f(z) = f(\mathfrak{g}.z)$ for all inputs $z$ and all permutations $\mathfrak{g}$, for instance $f([2, 3]) = f([3, 2])$.
    \item 
    \symmetriesText{\emph{Constraint}}: We translate the symmetry into the constraint $\|f(z) - f(\mathfrak{g}.z)\| = 0$ for all permutations $\mathfrak{g}$.
    \item 
    \learningText{\emph{Permutation invariance in the loss}}: We can encourage~$f$ to satisfy the symmetry by penalising the constraint violation in the optimisation problem:
    $$\min_\theta \mathcal{L}(f(z;\theta),y) + \lambda \sum_{\mathfrak{g}} \| f(z;\theta) - f(\mathfrak{g}.z;\theta)\|$$
    \item 
    \architecturesText{\emph{Permutation invariance in the architecture}}: Alternatively, we can compile the constraint into the model structure, for instance using a Deep Set model $f(z) = f_2 \big(\sum_i f_1(z_i)\big)$~\citep{zaheer2017deep}, which guarantees that the symmetry always holds, since a sum does not depend on the order of its terms.
\end{enumerate}

\section{Details on Symmetries}
\label{app:symmetries}

\subsection{Symmetry I}
Note that Symmetry~\ref{invariance:symm_1} extends classical formal notions of concepts \citep{goguen2005concept,barbiero2025foundations} to continuous/fuzzy semantics \citep{hajek2001metamathematics,marra2019lyrics}. In our setting, a concept score does not simply indicate whether an object satisfies a concept. It may also express degrees of membership to a concept class, such as one object being more \texttt{red} or more \texttt{old} than another. In this context, classical Boolean concepts arise as a special case, where the concept map has the form $c_w^{[h]}: Z \to \{0,1\}$. There, $Z$ is divided into two classes: objects satisfying $w$ and objects that do not. The set $\mathfrak{G}^{[h]}_w$ then contains the transformations that move an object from the second class to the first, and preserving the human semantics means that the model scores every object satisfying $w$ above every object that does not. Thus, in the Boolean case, an interpretable model must reproduce the human's partition up to a threshold, and must reproduce the human's concept map exactly if the model's concept map is Boolean as well.

\subsection{Symmetry II}
The second premise requires $\text{span}(\nabla_z f)$ to lie in $\text{span}(\nabla_z c)$ at every point~$z$. Since any change in~$f$ can then be expressed as a linear combination of changes in the concept maps, this symmetry enables us to \emph{control}~$f$ by acting on~$c$. This property, often known as \emph{intervenability} \citep{stochastic_cbms} or \emph{steerability} \citep{turner2023steering}, allows us to identify changes or interventions in the concept activations that produce desired changes in~$f$'s output. Under this view, when Symmetry~\ref{invariance:symm_2} holds at a single point $z_0$, intervenability is \emph{local}: one can steer~$f$ via~$c$ only in a neighbourhood of $z_0$ (e.g., using local surrogate models \citep{lime}). When the symmetry holds everywhere, intervenability becomes \emph{global}, and the model is steerable through its concepts for all inputs, as in Concept Bottleneck Models \citep{cbm}. In practice, this distinction implies that models enforcing Symmetry~\ref{invariance:symm_2} through optimisation rather than by architectural design may satisfy the symmetry only approximately, or only for a subset of inputs. Hence, a model trained with the constraint in the loss may appear steerable on the training set, while responding to interventions in unexpected ways at test time.

\section{Constraint~\ref{constraint:1}}
\label{app:constraintI}

\begin{lemma}
Symmetry~\ref{invariance:symm_1} is equivalent to Constraint~\ref{constraint:1}, i.e., for any object $z$ and transformation $\mathfrak{g}_w$,
$$c_w^{[h]}(\mathfrak{g}_w.z) > c_w^{[h]}(z) \implies c_w(\mathfrak{g}_w.z) > c_w(z) \quad \text{iff} \quad \mathbb{I}_{\Delta c_w^{[h]} > 0} \cdot \gamma(\Delta c_w) = 0$$
where
\begin{itemize}
\item[(a)] $c_w, c_w^{[h]}\colon Z \to \mathbb{R}$ are functions;
\item[(b)] $\Delta c_w^{[h]} = c_w^{[h]}(\mathfrak{g}_w.z) - c_w^{[h]}(z)$ and $\Delta c_w = c_w(\mathfrak{g}_w.z) - c_w(z)$;
\item[(c)] $\mathbb{I}_X(x)=1$ if $x\in X$ and $\mathbb{I}_X(x)=0$ if $x\notin X$;
\item[(d)] $\gamma\colon \mathbb{R} \to \mathbb{R}_{\geq 0}$ is a function such that $\gamma(s)=0$ if and only if $s>0$.
\end{itemize}
\end{lemma}

\begin{proof}
Recall that ($\dagger$) $P \Rightarrow Q \iff \neg(P \land \neg Q)$. The following holds.
\begin{align*}
c_w^{[h]}(\mathfrak{g}_w.z) > c_w^{[h]}(z) \implies c_w(\mathfrak{g}_w.z) > c_w(z) & \text{ iff } \Delta c_w^{[h]} > 0 \implies \Delta c_w > 0 \tag{b}\\
&\text{ iff } \neg\big(\Delta c_w^{[h]} > 0 \land \Delta c_w \leq 0\big) \tag{$\dagger$}\\
&\text{ iff } \mathbb{I}_{\Delta c_w^{[h]} > 0} \cdot \mathbb{I}_{\Delta c_w \leq 0} = 0 \tag{c}\\
&\text{ iff } \mathbb{I}_{\Delta c_w^{[h]} > 0} \cdot \gamma(\Delta c_w) = 0 \tag{d}
\end{align*}
\end{proof}

\section{Constraint~\ref{constraint:2}}
\label{app:constraintII}

\begin{lemma}
Let $r = \mathrm{rank}(\nabla_z c)$ and $k = \mathrm{rank}(\nabla_z f) > 0$. Then
$$\mathrm{span}\bigl\{\nabla_z f_1,\,\ldots,\,\nabla_z f_v\bigr\} \;\subseteq\; \mathrm{span}\bigl\{\nabla_z c_1,\,\ldots,\,\nabla_z c_l\bigr\} \iff 1 - \frac{\|Q_c^\top Q_f\|_F^2}{k} = 0$$
where
\begin{itemize}
    \item[(a)] $f \colon \mathbb{R}^n \to \mathbb{R}^v$ and $c \colon \mathbb{R}^n \to \mathbb{R}^l$;
    \item[(b)] $\nabla_z f \in \mathbb{R}^{v \times n}$ and $\nabla_z c \in \mathbb{R}^{l \times n}$ are Jacobian matrices;
    \item[(c)] $Q_c \in \mathbb{R}^{n \times r}$ is the matrix whose columns form an orthonormal basis for $\mathrm{span}\bigl\{\nabla_z c_1,\,\ldots,\,\nabla_z c_l\bigr\}$, so that $P_c = Q_c Q_c^\top$ is the orthogonal projector onto this span;\footnote{In practice, $Q_c$ and $Q_f$ are obtained via singular value decomposition \citep{eckart1936approximation}, keeping singular vectors corresponding to singular values greater than a small threshold $\varepsilon > 0$.}
    \item[(d)] $Q_f \in \mathbb{R}^{n \times k}$ is the matrix whose columns form an orthonormal basis for $\mathrm{span}\bigl\{\nabla_z f_1,\,\ldots,\,\nabla_z f_v\bigr\}$, so that $Q_f^\top Q_f = I_k$.
\end{itemize}
\end{lemma}
\begin{proof}
Recall that ($\dagger$) a vector $x$ lies in a subspace $S$ if and only if $(I - P_S)\,x = 0$, where $P_S$ is the orthogonal projector onto $S$; ($\ddagger$) for any matrix $A$, $\|A\|_F^2 = \mathrm{tr}(A^\top A)$; and ($\S$) an orthogonal projector $P$ is symmetric and idempotent, so that $(I - P)^\top (I - P) = I - P$. The following holds.
\begin{align*}
\mathrm{span}\bigl\{\nabla_z f_1,\,\ldots,\,\nabla_z f_v\bigr\} \subseteq \mathrm{span}\bigl\{\nabla_z c_1,\,\ldots,\,\nabla_z c_l\bigr\}
&\iff \text{every column of } Q_f \text{ lies in } \mathrm{span}(Q_c) \tag{c, d} \\
&\iff (I - P_c)\, Q_f = 0 \tag{$\dagger$} \\
&\iff \|(I - P_c)\, Q_f\|_F^2 = 0 \\
&\iff \mathrm{tr}\bigl(Q_f^\top (I - P_c)\, Q_f\bigr) = 0 \tag{$\ddagger$, $\S$} \\
&\iff \mathrm{tr}(Q_f^\top Q_f) - \mathrm{tr}(Q_f^\top Q_c Q_c^\top Q_f) = 0 \tag{c} \\
&\iff k - \|Q_c^\top Q_f\|_F^2 = 0 \tag{d, $\ddagger$} \\
&\iff 1 - \frac{\|Q_c^\top Q_f\|_F^2}{k} = 0
\end{align*}
\end{proof}

\section{Architecture~\ref{architecture:1}}
\label{app:architectureI}

\begin{lemma}
Architecture~\ref{architecture:1} satisfies Constraint~\ref{constraint:1} on all training pairs, i.e., for every pair $k = 1, \dots, p$,
$$c_w^{[h]}(z_{j_k}) < c_w^{[h]}(z_{i_k}) \implies c_w(z_{j_k}) < c_w(z_{i_k})$$
where
\begin{itemize}
    \item[(a)] $(z_{i_1}, z_{j_1}), \dots, (z_{i_p}, z_{j_p})$ are the training pairs with $c_w^{[h]}(z_{j_k}) < c_w^{[h]}(z_{i_k})$, and $D_w \in \mathbb{R}^{p \times n}$ is the matrix whose $k$-th row is $z_{i_k} - z_{j_k}$, so that for any $\theta$
    $$D_w \theta = \begin{pmatrix} \theta^\top (z_{i_1} - z_{j_1}) \\ \vdots \\ \theta^\top (z_{i_p} - z_{j_p}) \end{pmatrix};$$
    \item[(b)] the rows of $D_w$ are linearly independent, so that $D_wD_w^\top$ is invertible and $D_wD_w^\top\left(D_wD_w^\top\right)^{-1} = I_p$, the $p \times p$ identity matrix;
    \item[(c)] $\theta_{c_w} = D_w^\top\left(D_wD_w^\top\right)^{-1} \mathbf{1}_p$, where $\mathbf{1}_p \in \mathbb{R}^p$ is the all-ones vector, is the minimum-norm solution of $D_w \theta = \mathbf{1}_p$, i.e., the shortest direction along which every observed difference $z_{i_k} - z_{j_k}$ increases the concept score by exactly $1$. This is the closed-form ridge solution in its dual form~\citep{hoerl1970ridge,saunders1998ridge}, as $\varepsilon \to 0$;
    \item[(d)] $c_w(z) = \theta_{c_w}^\top z$.
\end{itemize}
\end{lemma}
\begin{proof}
The following holds.
\begin{align*}
D_w \theta_{c_w} &= D_w D_w^\top\left(D_wD_w^\top\right)^{-1} \mathbf{1}_p \tag{c}\\
D_w \theta_{c_w} &= I_p\, \mathbf{1}_p = \mathbf{1}_p \tag{b}\\
\theta_{c_w}^\top (z_{i_k} - z_{j_k}) &= 1 \quad \text{for } k = 1, \dots, p \tag{a}\\
c_w(z_{i_k}) - c_w(z_{j_k}) = \theta_{c_w}^\top (z_{i_k} - z_{j_k}) &= 1 > 0 \quad \text{for } k = 1, \dots, p \tag{d}
\end{align*}
\end{proof}

\section{Architecture~\ref{architecture:2}}
\label{app:architectureII}

\begin{lemma}
Architecture~\ref{architecture:2} satisfies Symmetry~\ref{invariance:symm_2}, i.e.,
$$f = \phi \circ c \implies \mathrm{rowspan}(\nabla_z f) \subseteq \mathrm{rowspan}(\nabla_z c) \quad \text{for all } z$$
where
\begin{itemize}
    \item[(a)] $f \colon \mathbb{R}^n \to \mathbb{R}^v$ and $c \colon \mathbb{R}^n \to \mathbb{R}^l$ are smooth functions, with Jacobian matrices $\nabla_z f \in \mathbb{R}^{v \times n}$ and $\nabla_z c \in \mathbb{R}^{l \times n}$;
    \item[(b)] $\phi \colon \mathbb{R}^l \to \mathbb{R}^v$ is a smooth function, with Jacobian matrix $\frac{\partial \phi}{\partial c} \in \mathbb{R}^{v \times l}$;
    \item[(c)] by the chain rule, $\nabla_z (\phi \circ c) = \frac{\partial \phi}{\partial c}\, \nabla_z c$;
    \item[(d)] for any matrices $A$ and $B$, every row of the product $AB$ is a linear combination of the rows of $B$, so that $\mathrm{rowspan}(AB) \subseteq \mathrm{rowspan}(B)$.
\end{itemize}
\end{lemma}
\begin{proof}
The following holds for all $z$.
\begin{align*}
\nabla_z f &= \nabla_z (\phi \circ c) \tag{Architecture~\ref{architecture:2}}\\
\nabla_z f &= \tfrac{\partial \phi}{\partial c}\, \nabla_z c \tag{c}\\
\mathrm{rowspan}(\nabla_z f) &\subseteq \mathrm{rowspan}(\nabla_z c) \tag{d}
\end{align*}
\end{proof}

\section{Empirical validation of the Standard Interpretable Model}
\label{app:validation}

To validate the proposed Standard Interpretable Model, we analyse the following research questions: Are standard performance metrics (e.g., mean error) suitable for measuring interpretability? Does compiling interpretability constraints into the architecture yield different behaviour compared to optimising for them? Do we benefit from enforcing an interpretability symmetry in both the architecture's structure and the optimisation process?

To answer these questions, we compared the behaviour of three models.
The base model is a generic DNN without any interpretability inductive biases (\texttt{DNN}). The second model uses the same architecture as the base model but introduces interpretability constraints during training (\texttt{DNN+L}). The third model instead, adds layers to the DNN that enforce interpretability constraints (\texttt{DNN+A}).

The code to reproduce these experiments is available at \url{https://github.com/pyc-team/pytorch_concepts/tree/master/examples/standard-model}.


\paragraph{Validating Symmetry~I (Figure~\ref{fig:val-symI}).}
To validate Symmetry~\ref{invariance:symm_1}, we sample points from $Z \times C \subseteq \mathbb{R}^2$, randomly assigning a ground-truth concept label $c_w^{[h]}(z)$ to each $z \in Z$. Next, we train concept maps $c_w$, in the three configurations \texttt{DNN}, \texttt{DNN+L}, and \texttt{DNN+A}, to predict the ground-truth concept scores. The \texttt{DNN} model is trained via a standard $L_1$ loss, while we train the \texttt{DNN+L} model by adding  Constraint~\ref{constraint:1} to the training objective and construct the \texttt{DNN+A} model by compiling Constraint~\ref{constraint:1} into the model (hence avoiding any training). 
We assess the learned concept maps in two complementary ways: (1)~by measuring mean absolute error against the human-assigned scores, which evaluates numerical fit, and (2)~by measuring violations of Constraint~\ref{constraint:1}, which evaluates whether the model preserves the preorder induced by the human concept semantics.

Our results in Figure~\ref{fig:val-symI} show that fitting human-assigned scores is not the same as preserving human concept semantics. In the first row, \texttt{DNN} closely matches the absolute ground-truth scores and therefore achieves the lowest mean absolute error, as expected from its $L_1$ training objective. However, under Symmetry~\ref{invariance:symm_1}, semantics are defined by the preorder induced by these scores, not by their precise numerical values.

The second row makes this distinction visible by sorting samples according to the human ground-truth ordering. A semantics-preserving concept map should be monotone along this ordering. Although \texttt{DNN} has low MAE, its sorted predictions contain several decreases, showing that it reverses some pairwise human orderings. \texttt{DNN+L} reduces these violations, while \texttt{DNN+A} preserves the ordering most faithfully because Constraint~\ref{constraint:1} is built into the architecture.

This is confirmed by the Constraint~\ref{constraint:1} violation rate, which measures the proportion of observed pairs that violate the human-induced preorder. While \texttt{DNN} performs best under MAE, it performs worst under this semantic metric; conversely, \texttt{DNN+A} achieves the lowest violation rate. These results show that \textbf{standard performance metrics}, such as mean absolute error, \textbf{are inappropriate proxies for interpretability}: a model can fit human-assigned scores while failing to preserve the semantic ordering those scores induce. Interpretability metrics should therefore be derived from the relevant semantic symmetry.
\begin{figure}[t]
    \centering
    \includegraphics[width=0.22\linewidth]{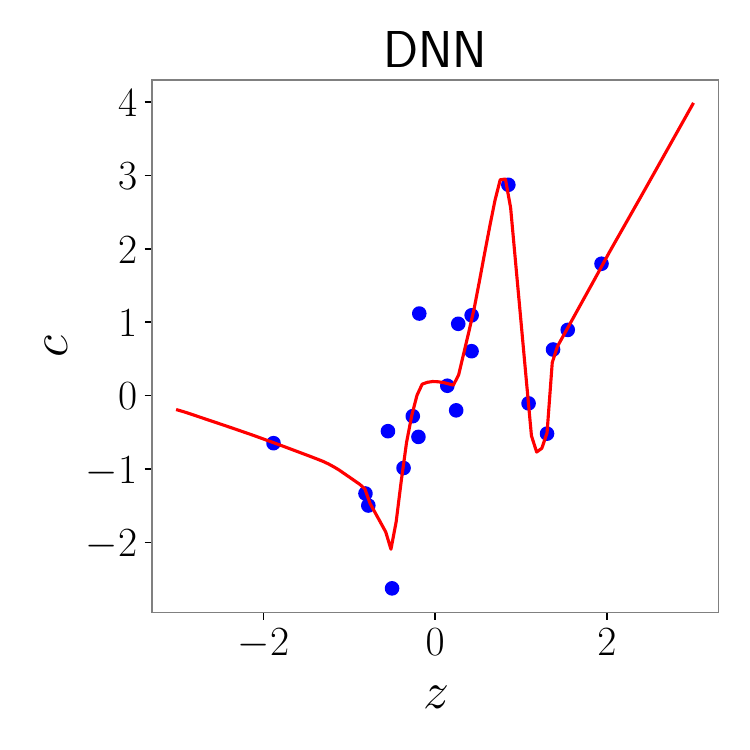}
    \includegraphics[width=0.22\linewidth]{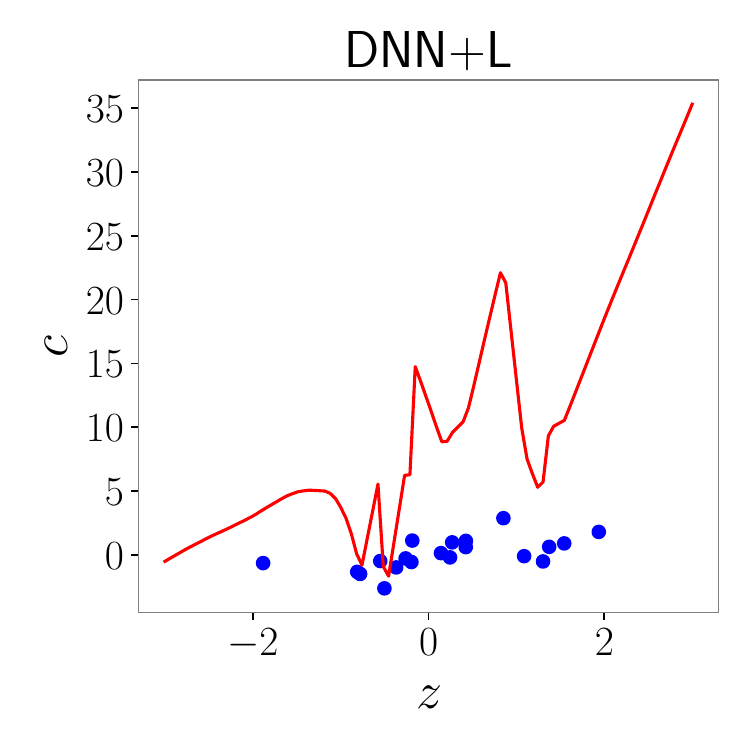}
    \includegraphics[width=0.22\linewidth]{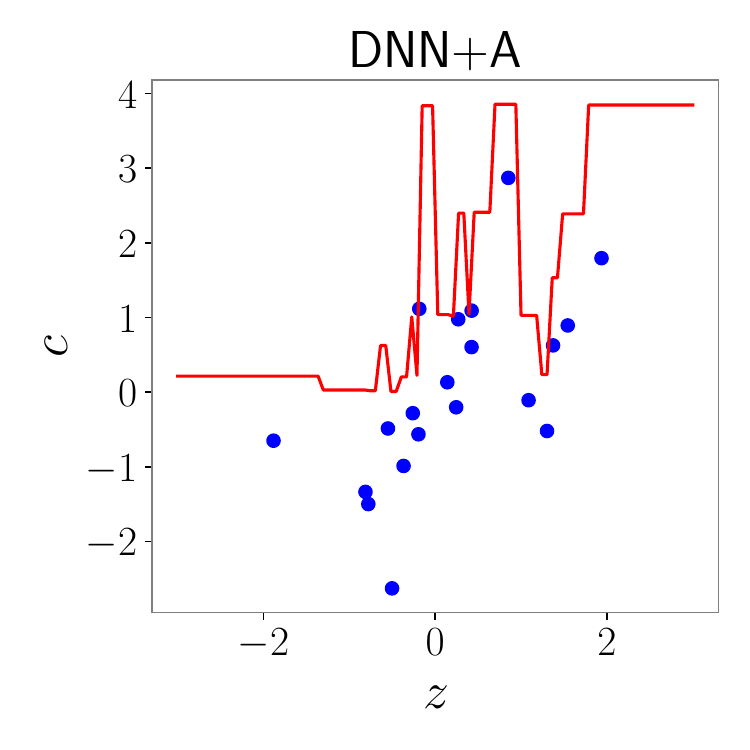}
    \includegraphics[width=.22\linewidth]{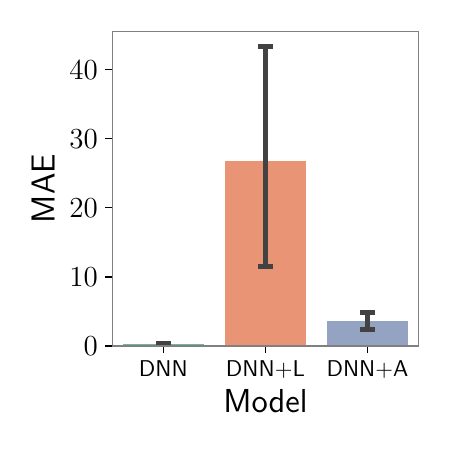}\\
    \includegraphics[width=0.22\linewidth]{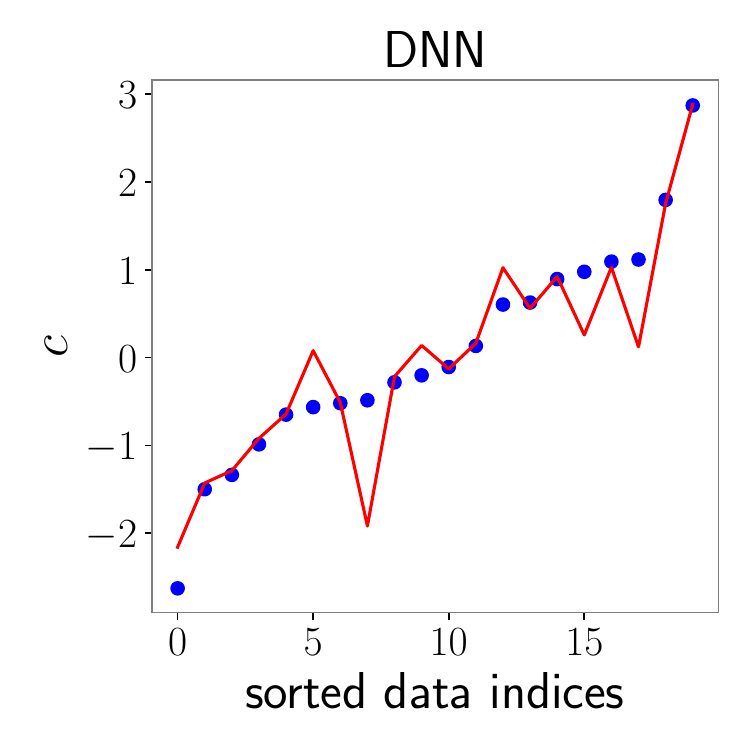}
    \includegraphics[width=0.22\linewidth]{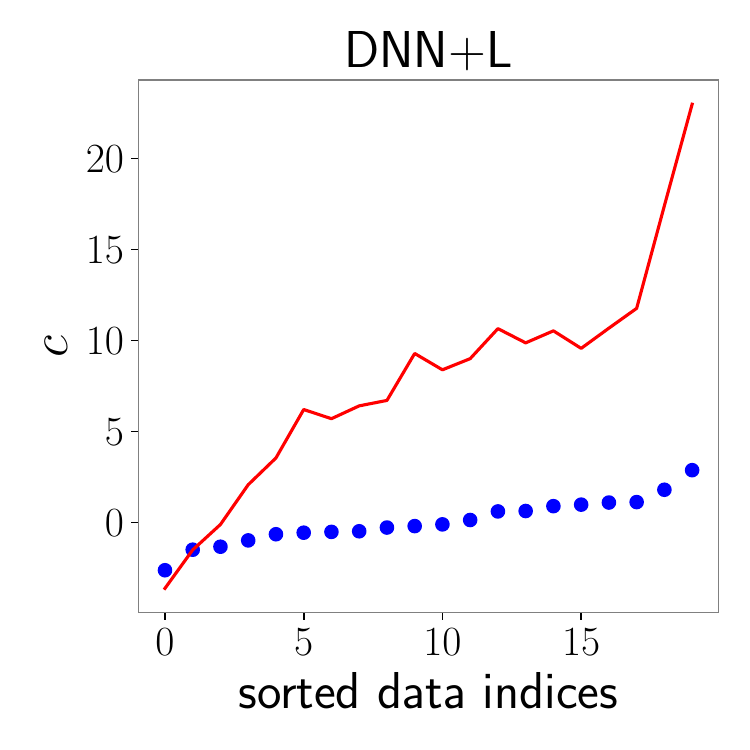}
    \includegraphics[width=0.22\linewidth]{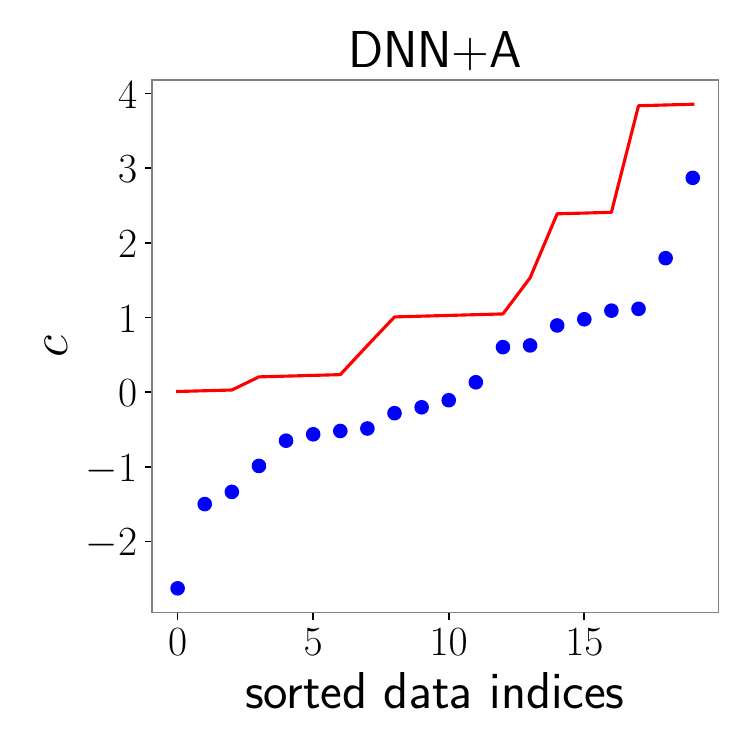}
    \includegraphics[width=.22\linewidth]{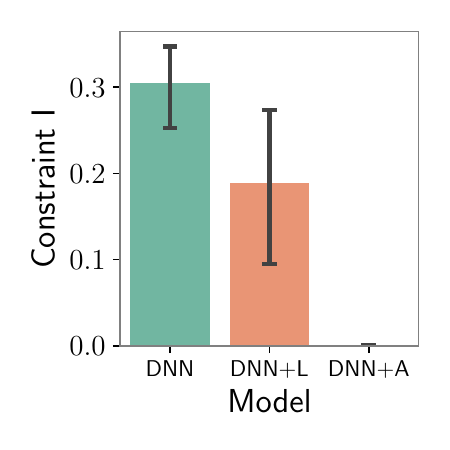}
\caption{\textbf{Validation of Symmetry~I.} Top row: learned concept maps fitted to human-assigned scores. Bottom row: the same predictions sorted by the human-induced ordering, where semantic preservation requires monotonicity. Right: MAE favours score fitting, while Constraint~\ref{constraint:1} reveals violations of the human preorder.  Low prediction error alone does not imply preserved concept semantics.\looseness-1}
    \label{fig:val-symI}
\end{figure}

\paragraph{Validating Symmetry~\ref{invariance:symm_2} (Figures~\ref{fig:val-symII-bis},~\ref{fig:val-symII}).}
To validate Symmetry~\ref{invariance:symm_2}, we sample $4$ points from $Z^2 \times C \times Y$, using tuples $(z, c) \in Z^2 \times C$ to train concept maps $c: Z^2 \to C$ and tuples $(z, y) \in Z^2 \times Y$ to train task predictors $f: Z^2 \to Y$. Both $c$ and $f$ are trained \textit{jointly} for each configuration (\texttt{DNN}, \texttt{DNN+L}, \texttt{DNN+A}). Here, \texttt{DNN+L} enforces Constraint~\ref{constraint:2} only locally, at $z=(0.7, 0.8)$, encouraging the gradients $\nabla f$ and $\nabla c$ to align at that point. By contrast, \texttt{DNN+A} compiles Constraint~\ref{constraint:2} directly into the architecture, so that $f$ is constrained to depend on $z$ only through $c$.

\begin{figure}[!t]
    \centering
    \includegraphics[width=0.5\linewidth]{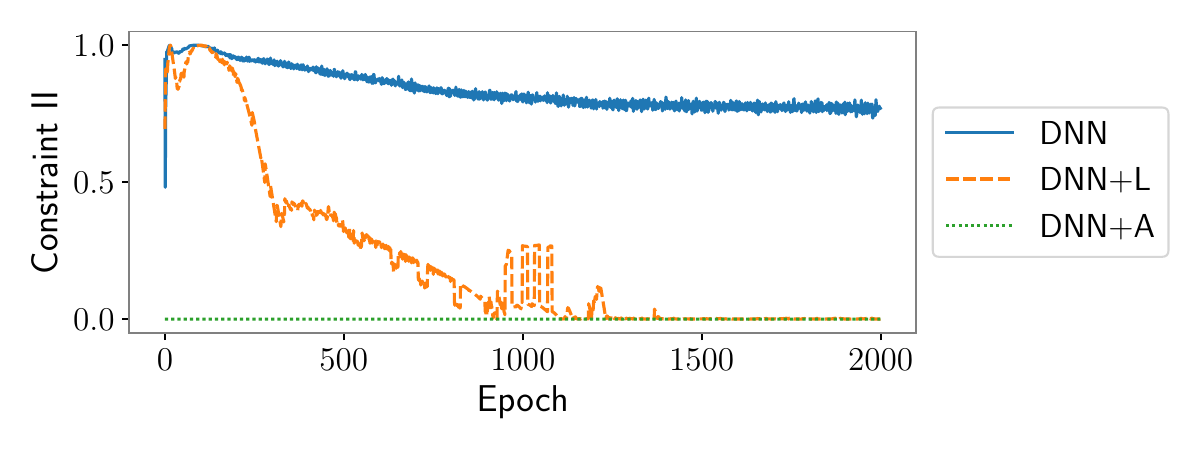}
\caption{\textbf{Validation of Symmetry~\ref{invariance:symm_2}.} Constraint~\ref{constraint:2} violation during training, measuring misalignment between the prediction gradients $\nabla f$ and concept gradients $\nabla c$. Optimising the constraint can reduce local gradient misalignment, while architectural compilation satisfies the dependency by construction.}
    \label{fig:val-symII-bis}
\end{figure}
Figure~\ref{fig:val-symII-bis} shows the violation of Constraint~\ref{constraint:2} during training. Initially, \texttt{DNN+L} behaves similarly to \texttt{DNN}, indicating that the predictor $f$ is not yet aligned with the concept map~$c$. After approximately $1000$ epochs, the constraint violation decreases and \texttt{DNN+L} reaches behaviour comparable to \texttt{DNN+A} at the constrained point. This shows that optimisation can enforce Symmetry~\ref{invariance:symm_2} locally when the corresponding constraint is included in the loss.

\begin{figure}[!t]
    \centering
    \includegraphics[width=.32\linewidth]{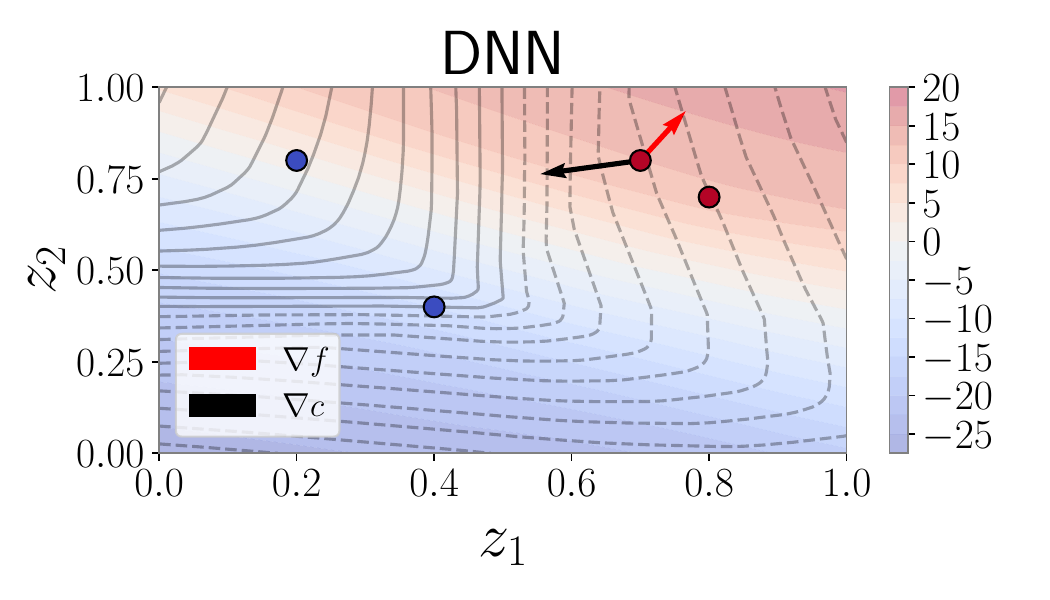}
    \includegraphics[width=.32\linewidth]{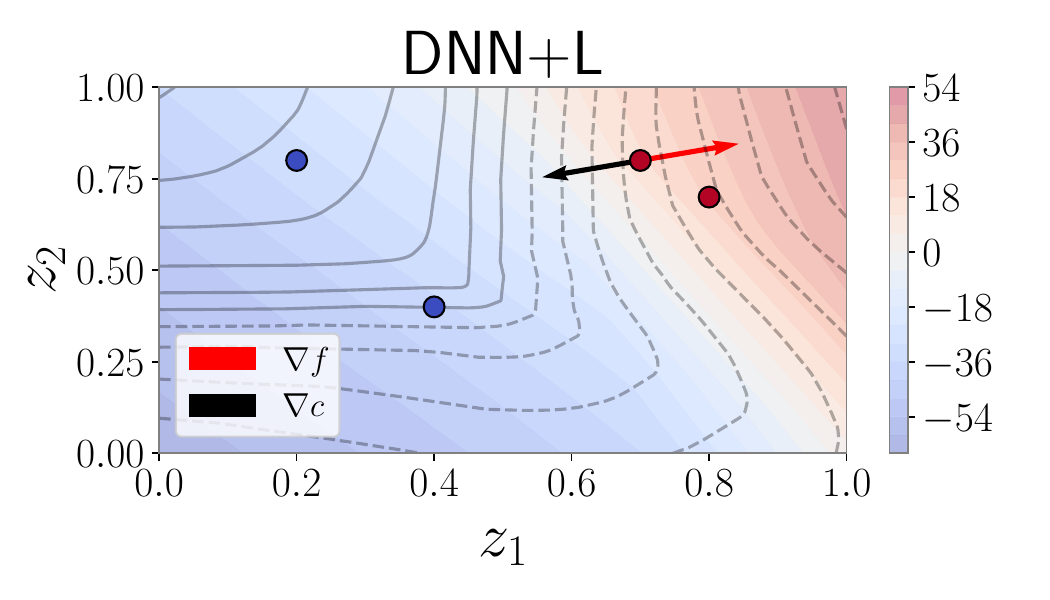}
    \includegraphics[width=.32\linewidth]{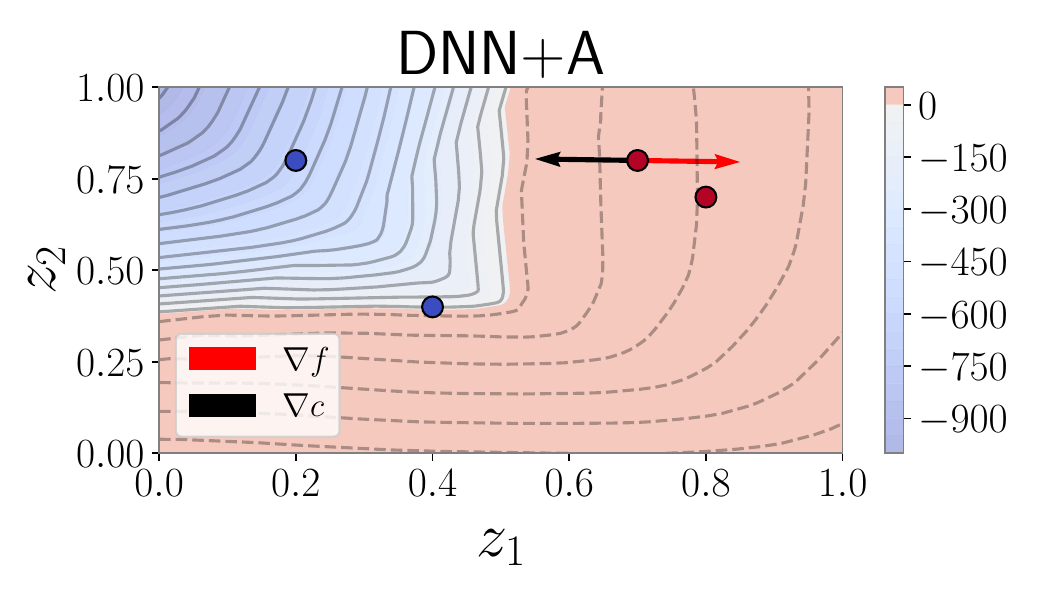}\\
    \includegraphics[width=.32\linewidth]{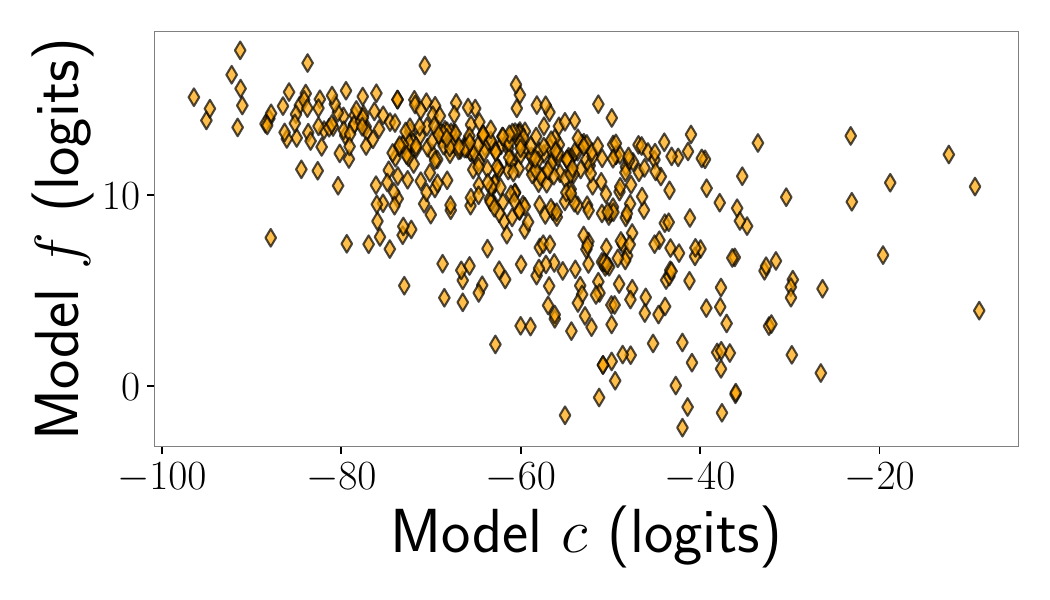}
    \includegraphics[width=.32\linewidth]{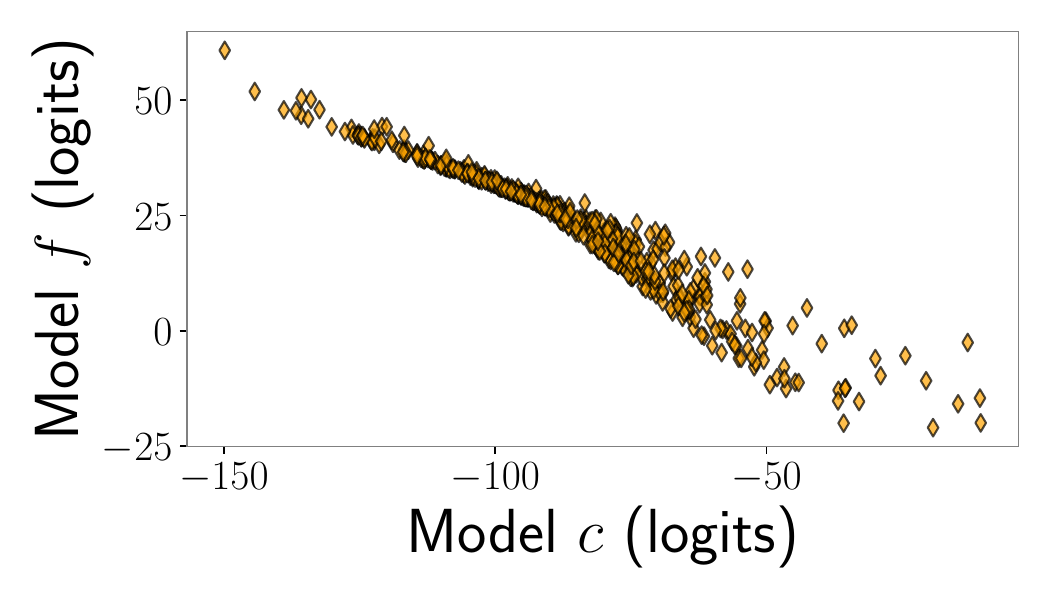}
    \includegraphics[width=.32\linewidth]{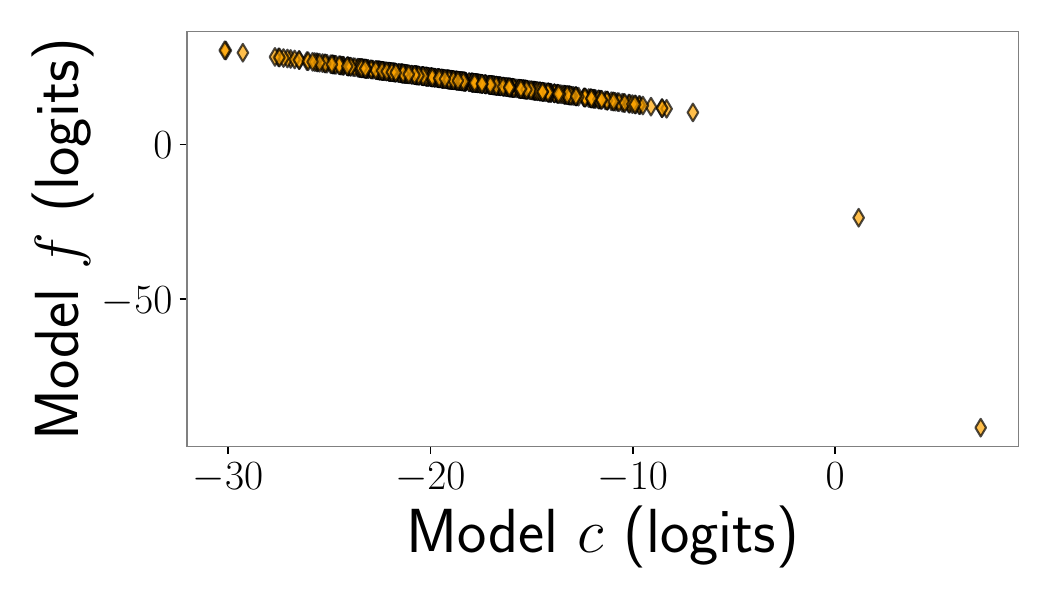}
    \caption{\textbf{Validation of Symmetry~\ref{invariance:symm_2}.} Top row: alignment between the predictor $f$ and the concept map~$c$. Coloured dots represent training samples. Dark curves show level sets of~$c$, while colour gradients show level sets of~$f$. If~$f$ depends only on~$c$, these level sets should be aligned. Bottom row: relationship between predicted task values~$f(z)$ and concept value~$c(z)$. Local constraint optimisation can align $f$ and $c$ near selected points, but architectural compilation is needed to guarantee concept-mediated predictions globally.}
    \label{fig:val-symII}
\end{figure}
Figure~\ref{fig:val-symII} shows why this local enforcement is not enough to guarantee global interpretability. In the top row, the level sets of the concept map $c$ are shown by dark curves, while the level sets of the predictor $f$ are represented by colour gradients. If $f$ depends only on $c$, then changes in $f$ should occur only along directions in which $c$ changes; equivalently, the gradients of $f$ and $c$ should be aligned. For \texttt{DNN}, the level sets of $f$ and $c$ are misaligned almost everywhere, showing that the predictor uses directions in $z$ that are not captured by the concept map. For \texttt{DNN+L}, the alignment improves near the point where Constraint~\ref{constraint:2} is enforced, but this local alignment does not extend across the whole input space. For \texttt{DNN+A}, the alignment holds globally by construction.

The bottom row confirms the same conclusion from the relationship between the predicted task value $f(z)$ and the concept value $c(z)$. In \texttt{DNN}, there is no clear functional dependency between $f$ and $c$, meaning that knowing the concept value is insufficient to explain the model prediction. In \texttt{DNN+L}, a local dependency emerges around the constrained region, but it is not guaranteed everywhere. In \texttt{DNN+A}, the relationship is exact because the architecture forces $f$ to be a function of $c$. Together, these results show that \textbf{without Constraint~\ref{constraint:2}, task predictions may not actually depend on concepts}. Moreover, enforcing the constraint only through optimisation can make this dependency local or approximate. To guarantee that predictions are interpretable and controllable through concepts for all inputs, Constraint~\ref{constraint:2} must be compiled into the architecture.

\paragraph{Validating Symmetry~\ref{invariance:symm_3} (Figure~\ref{fig:val-symIII}).}
To validate Symmetry~\ref{invariance:symm_3}, we sample data points from $C \times Y \subseteq \mathbb{R}^2$, where the target $y$ is generated as a noisy polynomial of~$c$. We then train \texttt{DNN+L} with Constraint~\ref{constraint:3} in its loss objective, sweeping $\lambda_3 = [10^{-5}, 10]$.
Figure~\ref{fig:val-symIII} illustrates how Constraint~\ref{constraint:3} controls the complexity of the concept formula $\phi$. When $\lambda_3$ is small, the model behaves more like an unconstrained \texttt{DNN}, allowing highly curved formulae that fit the noisy target flexibly. As $\lambda_3$ increases, violations of the operator constraint are penalised more strongly, and the learned formula becomes progressively closer to the hypothesis space specified by $K^{[h]}$. In the architectural case \texttt{DNN+A}, this restriction is enforced by construction, yielding formulae that remain within the compiled function class.


These results show that $\lambda_3$ acts as a knob for controlling the inductive bias of the concept-formula hypothesis space. Practically, this allows Symmetry~\ref{invariance:symm_3} to be used flexibly: one may compile a stricter operator into the architecture to guarantee bounded reasoning, while also using a softer operator penalty in the objective to encourage simpler formulae when supported by the data.
\begin{figure}[h]
    \centering
    \includegraphics[width=0.35\linewidth]{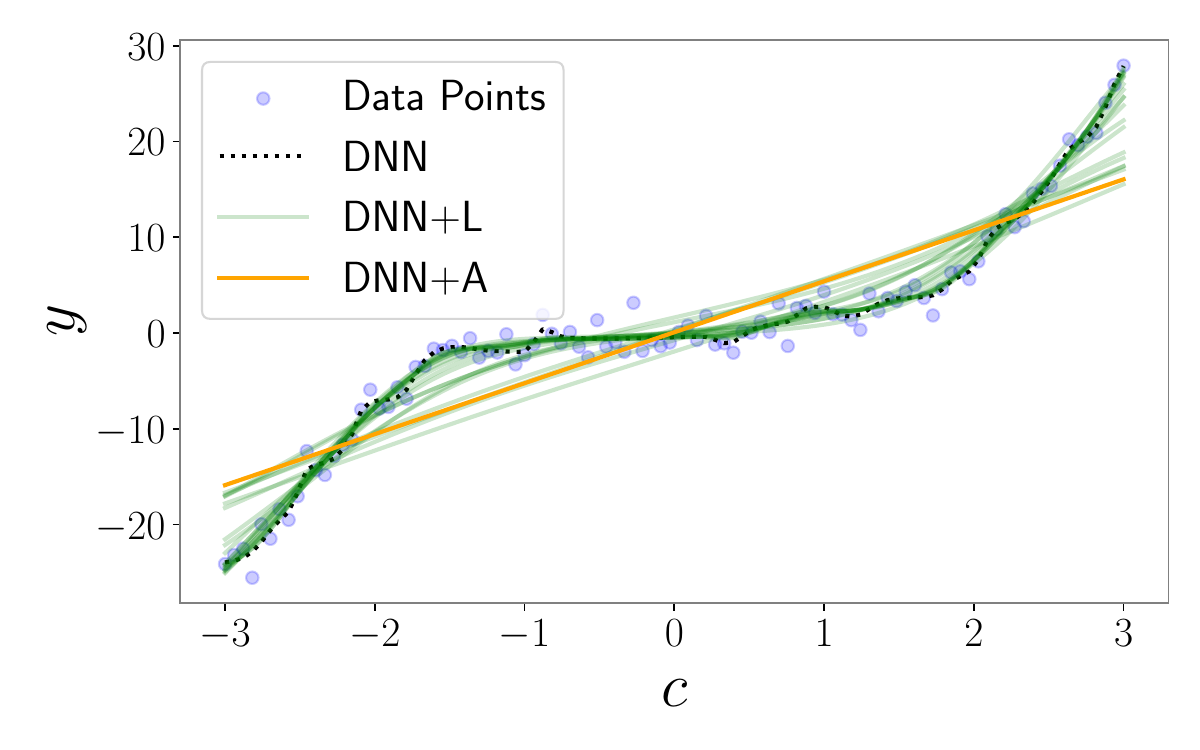}
    \caption{\textbf{Validation of Symmetry~\ref{invariance:symm_3}.} Increasing the weight of Constraint~\ref{constraint:3} restricts the curvature of the learned concept formula, interpolating between flexible unconstrained fitting and the architecturally compiled hypothesis space. The SIM makes reasoning complexity both tunable through optimisation and enforceable by design.}
    \label{fig:val-symIII}
\end{figure}

\section{Experimental details}
\label{app:hyperparams}

All interpretable language models share the same GPT-2-style backbone~\citep{radford2019language} and the same optimisation setup, and differ only in the concept-based predictor and in how the residual is treated. Table~\ref{tab:hyperparams} summarises the main hyperparameters, and Table~\ref{tab:params} reports the number of parameters of each architecture.

\begin{table}[h]
\centering
\renewcommand{\arraystretch}{1.15}
\begin{tabular}{ll}
\toprule
\textbf{Hyperparameter} & \textbf{Value} \\
\midrule
\multicolumn{2}{l}{\textit{Architecture (shared backbone)}} \\
Embedding dimension & $128$ \\
Attention heads & $4$ \\
Layers & $4$ \\
Dropout & $0.2$ \\
Context length & $256$ \\
Vocabulary size & $2{,}000$ \\
Concepts & $623$ \\
\midrule
\multicolumn{2}{l}{\textit{Optimisation}} \\
Optimiser & AdamW \\
Learning rate (with cosine decay) & $3\times10^{-4} \rightarrow 3\times10^{-5}$ \\
Warmup & $200$ steps \\
Weight decay & $0.1$ \\
Batch size & $64$ \\
Training steps & $5{,}000$ \\
Train / validation split & $0.9$ / $0.1$ \\
\bottomrule
\end{tabular}
\caption{Main hyperparameters shared by all language models.}
\label{tab:hyperparams}
\end{table}

\begin{table}[h]
\centering
\renewcommand{\arraystretch}{1.15}
\begin{tabular}{lc}
\toprule
\textbf{Model} & \textbf{Parameters} \\
\midrule
LM (baseline) & $4.21$M \\
CBM & $4.92$M \\
CRM & $9.78$M \\
ICRM & $6.02$M \\
Steerling & $6.07$M \\
\bottomrule
\end{tabular}
\caption{Number of parameters per architecture, including non-trainable buffers.}
\label{tab:params}
\end{table}

\section{Newtonian Lagrangian of interpretable models}
\label{app:lagrangian}
The Lagrangian relaxation of the interpretability-constrained problem in Equation~\eqref{lagrangian:1}~\citep{hestenes1969multiplier} is a static objective: it assigns a scalar value to any configuration of parameters, but says nothing about how the parameters evolve during training. Optimisation, however, is a dynamic process, in which the parameters follow a trajectory through parameter space that progressively decreases the objective. Making this dynamics explicit requires a time derivative $\partial_t \theta$ describing how each parameter moves.

This suggests a mechanical analogy. If we treat the relaxed objective as a potential energy $V$, the parameters behave like particles moving through an energy landscape, attracted towards its minima, and each interpretability constraint contributes a term to the landscape whose height grows with the constraint violation. To complete the mechanical picture, we follow~\citet{guo2025physics} and introduce a kinetic term $T$ capturing the cost of motion through parameter space. The difference $T - V$ is then the Lagrangian of classical mechanics, and training an interpretable model corresponds to the motion of a particle in a landscape shaped by both the task loss and the interpretability constraints.

\subsection{Equations of motion}
\label{app:equations-of-motion}
Given the Lagrangian, one could naively minimise $V$ at each step independently. However, in classical mechanics, a particle does not simply fall towards the nearest potential minimum. Due to the inertial term, the particle's trajectory is consistent with its past and future motion. A ball thrown upwards, for instance, does not instantly drop to the ground but follows a smooth parabolic arc determined by both its velocity and the gravitational potential. The action $S = \int L \, dt$ encodes this by accumulating $T - V$ over the entire trajectory. Minimising $S$ selects the trajectory that optimally trades off paths that move too fast (high $T$) with paths that linger in high-potential regions (high $V$). The Euler-Lagrange equations provide the conditions under which a trajectory $\theta_i(t)$ is a stationary point of $S$, that is, a path along which no small perturbation can reduce it. For each parameter $\theta_i$, the Euler-Lagrange equations read:
$$\frac{d}{dt} \left( \frac{\partial L}{\partial (\partial_t \theta_i)} \right) - \frac{\partial L}{\partial \theta_i} = 0$$
The first term captures how the momentum of $\theta_i$ changes over time, while the second captures the force acting on $\theta_i$ due to the potential $V$. Setting their difference to zero yields Newton's second law in parameter space: the acceleration of each parameter $\partial_t^2 \theta_i$ is determined by the gradient of the potential acting on it. Applying this to each parameter group yields the equations of motion describing the dynamics of the full interpretable model.

\begin{lemma}
The Euler-Lagrange equations applied to $L$ with respect to the parameter groups $\theta_f$ and $\theta_c$ yield:
\begin{align*}
m\, \partial_t^2 \theta_f &= -\nabla_{\theta_f} \mathcal{L}(f(z;\theta_f,\theta_c), y) - \lambda_2 \nabla_{\theta_f} C_2(\theta_f, \theta_c) - \lambda_3 \nabla_{\theta_f} C_3(\theta_f, \theta_c) \\
m\, \partial_t^2 \theta_c &= -\nabla_{\theta_c} \mathcal{L}(f(z;\theta_f,\theta_c), y) - \lambda_1 \nabla_{\theta_c} C_1(\theta_c) - \lambda_2 \nabla_{\theta_c} C_2(\theta_f, \theta_c) - \lambda_3 \nabla_{\theta_c} C_3(\theta_f, \theta_c)
\end{align*}
where
\begin{itemize}
    \item[(a)] the model $f(z;\theta_f,\theta_c)$ depends on $\theta_c$ through the concept maps $c = c(z;\theta_c)$, and the constraint terms are
    \begin{align*}
    C_1(\theta_c) &= \sum_{w=1}^m \sum_{\mathfrak{g}_w} \mathbb{I}_{\Delta c_w^{[h]}(z, \mathfrak{g}_w.z) > 0} \cdot \gamma\big(\Delta c_w(z, \mathfrak{g}_w.z;\theta_c)\big) && \text{(Constraint~\ref{constraint:1})} \\
    C_2(\theta_f, \theta_c) &= 1 - \frac{\|Q_c^\top Q_f\|_F^2}{\mathrm{rank}(\nabla_z f)} && \text{(Constraint~\ref{constraint:2})} \\
    C_3(\theta_f, \theta_c) &= \big\| K^{[h]}_f\big(c, f_u\big) \big\| && \text{(Constraint~\ref{constraint:3})}
    \end{align*}
    with $f_u$ the model $f(\cdot\,;\theta_f)$ as a function of the concepts with $u$ held fixed;
    \item[(b)] the Lagrangian is $L = T - V$, with potential
    $$V(\theta_f, \theta_c) = \mathcal{L}(f(z;\theta_f,\theta_c), y) + \lambda_1 C_1(\theta_c) + \lambda_2 C_2(\theta_f, \theta_c) + \lambda_3 C_3(\theta_f, \theta_c);$$
    \item[(c)] $T = \sum_{i \in \{f,c\}} \frac{1}{2} m\, (\partial_t \theta_i)^\top (\partial_t \theta_i)$ is the kinetic term, which is the only part of $L$ depending on $\partial_t \theta_i$;
    \item[(d)] the Euler-Lagrange equations for each parameter group $\theta_i$ read $\dfrac{d}{dt}\left(\dfrac{\partial L}{\partial (\partial_t \theta_i)}\right) - \dfrac{\partial L}{\partial \theta_i} = 0$.
\end{itemize}
\end{lemma}
\begin{proof}
Since $T$ is the only part of $L$ depending on $\partial_t \theta_i$ (c), the first term of the Euler-Lagrange equations (d) is
$$\frac{d}{dt}\left(\frac{\partial L}{\partial (\partial_t \theta_i)}\right) = \frac{d}{dt}\left(m\, \partial_t \theta_i\right) = m\, \partial_t^2 \theta_i.$$
Since $T$ does not depend on $\theta_i$, the second term is
$$\frac{\partial L}{\partial \theta_i} = -\nabla_{\theta_i} V(\theta_f, \theta_c).$$
The Euler-Lagrange equations therefore reduce to
$$m\, \partial_t^2 \theta_i = -\nabla_{\theta_i} V(\theta_f, \theta_c).$$

\paragraph{Dynamics of $\theta_f$.}
The task loss, $C_2$, and $C_3$ depend on $\theta_f$, while $C_1$ does not (a, b), so that
$$m\, \partial_t^2 \theta_f = -\nabla_{\theta_f} \mathcal{L}(f(z;\theta_f,\theta_c), y) - \lambda_2 \nabla_{\theta_f} C_2(\theta_f, \theta_c) - \lambda_3 \nabla_{\theta_f} C_3(\theta_f, \theta_c).$$

\paragraph{Dynamics of $\theta_c$.}
All four terms of $V$ depend on $\theta_c$, the task loss through the concept maps (a, b), so that
$$m\, \partial_t^2 \theta_c = -\nabla_{\theta_c} \mathcal{L}(f(z;\theta_f,\theta_c), y) - \lambda_1 \nabla_{\theta_c} C_1(\theta_c) - \lambda_2 \nabla_{\theta_c} C_2(\theta_f, \theta_c) - \lambda_3 \nabla_{\theta_c} C_3(\theta_f, \theta_c).$$
\end{proof}

\subsection{Gradient descent}
\label{app:gradient-descent}

\begin{lemma}
The central difference approximation of the equations of motion $m\, \partial^2_t \theta = F(\theta)$ yields the update rule
$$\theta_{t+1} = \theta_{t} + (\theta_{t} - \theta_{t-1}) + \frac{(\Delta t)^2}{m} F(\theta_{t})$$
where
\begin{itemize}
    \item[(a)] $\theta = (\theta_f, \theta_c)$ collects all parameters, $\theta_t$ denotes their value at step $t$, and $F(\theta) = -\nabla_\theta V(\theta_f, \theta_c)$ is the right-hand side of the equations of motion in Appendix~\ref{app:equations-of-motion};
    \item[(b)] the central difference approximation of $\partial^2_t \theta$ with step size $\Delta t$ reads
    $$\partial^2_t \theta \approx \dfrac{\theta_{t+1} - 2\theta_{t} + \theta_{t-1}}{(\Delta t)^2}.$$
\end{itemize}
\end{lemma}
\begin{proof}
Replacing $\partial^2_t \theta$ with its central difference approximation (b) in the equations of motion (a) gives
\begin{align*}
m\, \frac{\theta_{t+1} - 2\theta_{t} + \theta_{t-1}}{(\Delta t)^2} = F(\theta_t)
&\iff \theta_{t+1} - 2\theta_{t} + \theta_{t-1} = \frac{(\Delta t)^2}{m} F(\theta_t) \\
&\iff \theta_{t+1} = 2\theta_{t} - \theta_{t-1} + \frac{(\Delta t)^2}{m} F(\theta_t) \\
&\iff \theta_{t+1} = \theta_{t} + (\theta_{t} - \theta_{t-1}) + \frac{(\Delta t)^2}{m} F(\theta_t).
\end{align*}
\end{proof}

\noindent
Since $F(\theta_t) = -\nabla_\theta V(\theta_t)$, the update rule corresponds to gradient descent with momentum coefficient $1$ and learning rate $(\Delta t)^2/m$.

\end{document}